\documentclass{article} %
\usepackage{iclr2027_conference,times}

\usepackage{amsmath,amsfonts,bm}

\def\1{\bm{1}}

\DeclareMathAlphabet{\mathsfit}{\encodingdefault}{\sfdefault}{m}{sl}
\SetMathAlphabet{\mathsfit}{bold}{\encodingdefault}{\sfdefault}{bx}{n}

\usepackage{microtype}
\usepackage{amsmath,amssymb,amsfonts,amsthm,mathtools}
\usepackage{booktabs,tabularx,array,multirow}
\usepackage{float}
\usepackage{enumitem}
\usepackage{graphicx}
\usepackage{xcolor}
\usepackage{tikz}
\usepackage{placeins}
\usepackage{subcaption}
\usetikzlibrary{arrows.meta,calc,positioning,decorations.pathreplacing,
                decorations.markings,shapes.geometric,fit,backgrounds}

\newtheorem{definition}{Definition}
\newtheorem{theorem}{Theorem}
\newtheorem{proposition}{Proposition}
\newtheorem{corollary}{Corollary}

\newcommand{\cO}{\mathcal{O}}

\newcommand{\cX}{\mathcal{X}}
\newcommand{\cE}{\mathcal{E}}

\newcommand{\cT}{\mathcal T}
\newcommand{\cG}{\mathcal G}
\newcommand{\cW}{\mathcal W}

\newcommand{\inner}[2]{\left\langle #1,#2\right\rangle}

\newcommand{\cI}{\mathcal I}

\usepackage{xcolor}
\usetikzlibrary{calc,positioning,decorations.pathreplacing}
\definecolor{degc}{RGB}{106,61,154}
\definecolor{rankc}{RGB}{23,138,88}
\definecolor{ink}{RGB}{60,60,60}
\providecommand{\ind}{\mathbf{1}}
\providecommand{\rankcp}{\operatorname{rank}_{\mathrm{CP}}}

\usepackage{hyperref}
\usepackage{url}

\definecolor{tnblue}{RGB}{45,110,168}
\definecolor{tnorange}{RGB}{206,113,30}
\definecolor{tngreen}{RGB}{40,138,90}
\definecolor{tngray}{RGB}{122,129,138}
\definecolor{tnink}{RGB}{38,42,50}
 
\tikzset{
  >={Latex[length=1.7mm,width=1.5mm]},
  tn/.style       = {draw=tnblue!85!black,fill=tnblue!12,rounded corners=1.8pt,
                     minimum width=8mm,minimum height=7mm,inner sep=1.5pt,
                     font=\scriptsize},
  core/.style     = {draw=tnblue!85!black,fill=tnblue!16,circle,
                     minimum size=4.4mm,inner sep=0pt,font=\tiny},
  bigcore/.style  = {core,minimum size=6.2mm,font=\scriptsize},
  qv/.style       = {draw=tnorange!85!black,fill=tnorange!13,circle,
                     minimum size=4.7mm,inner sep=0pt,font=\tiny},
  qb/.style       = {draw=tnorange!85!black,fill=tnorange!13,rounded corners=1pt,
                     minimum width=4.7mm,minimum height=4.7mm,inner sep=1pt,
                     font=\tiny},
  qcap/.style     = {draw=tnorange!85!black,fill=tnorange!13,
                     isosceles triangle,shape border rotate=90,
                     minimum width=4.4mm,minimum height=3.2mm,inner sep=0pt,
                     font=\tiny},
  loc/.style      = {draw=tngreen!80!black,fill=tngreen!13,rounded corners=.8pt,
                     minimum width=4.6mm,minimum height=4.6mm,inner sep=1pt,
                     font=\tiny},
  copydot/.style  = {circle,fill=tnink,minimum size=1.9mm,inner sep=0pt},
  bond/.style     = {draw=tnink,line width=.55pt},
  qbond/.style    = {draw=tnorange!80!black,line width=.55pt},
  oleg/.style     = {draw=tngreen!70!black,line width=.6pt},
  flow/.style     = {draw=tnink!55,line width=.5pt,->},
  tndim/.style      = {font=\tiny,text=tngray,inner sep=1.3pt},
  tnttl/.style      = {font=\scriptsize\bfseries,text=tnink,align=center},
  tncap/.style      = {font=\tiny,text=tnink,align=center},
  tnhint/.style     = {font=\tiny,text=tnink!70,align=center},
  maskon/.style   = {draw=tnblue!70!black,fill=tnblue!55,minimum size=2.4mm,inner sep=0pt},
  maskoff/.style  = {draw=tnblue!55!black,fill=white,minimum size=2.4mm,inner sep=0pt},
  panelbox/.style = {draw=tnink!15,line width=.4pt,rounded corners=2.5pt},
}
 
\newcommand{\ttchainV}[4]{%
  \foreach \tnk in {1,...,5}{%
    \node[core] (#3\tnk) at ({#1-2*#4+#4*(\tnk-1)},#2) {};}%
  \foreach \tnk in {1,...,4}{%
    \draw[bond] (#3\tnk)--(#3\the\numexpr\tnk+1\relax);}%
}

\title{Tensor Network Datamodels for Training Data Attribution}

\author{
Farzaneh Heidari$^{1,2}$,
Md Rifat Arefin$^{1,2}$,
Louis Jaburi$^{3}$,
Guillaume Rabusseau$^{1,2,4}$\\
$^{1}$Mila -- Quebec AI Institute,
$^{2}$DIRO, Université de Montréal\\
$^{3}$ EleutherAI,
$^{4}$Canada CIFAR AI Chair\\
\texttt{farzaneh.heidari@umontreal.ca}
}

\iclrfinalcopy %

\begin{document}

\maketitle
\lhead{}

\begin{abstract}

Datamodels predict how a model's behavior changes when different parts of its training data are included or removed. However, linear datamodels assign fixed contributions to training
units and cannot represent non-additive joint effects.
We introduce Tensor Network Datamodels (TNDMs), which use low-rank tensor networks to model dependencies among training units while keeping the representation compact.
TNDMs learn a map from subsets of the training data to the outcome of training on those subsets, allowing them to represent interactions among several training units. We prove that complementary and redundant training groups
can induce high-degree subset functions with compact
low-rank representations. 
A single fitted TNDM can then be used for several training data questions, including predicting the effect of removing data, measuring the importance of individual training units, identifying interactions between training units, and selecting useful subsets of the training data.
We evaluate TNDMs on controlled tasks designed to contain strong interactions, as well as standard fine-tuning settings for classification and
reasoning. On a controlled reasoning task, TNDMs outperform every baseline while observing only half as many training subsets,
and TNDM-TT reduces prediction NMSE by $89\%$ relative to the linear
datamodel after observing approximately $0.0061\%$ of all possible
training subsets. In factual-association fine-tuning, a single TNDM-CP
raises held-out $\mathrm{R}^2$ from $0.579$ to $0.668$ over the linear
datamodel. These results demonstrate that compact low-rank models
improve prediction when training data effects depend on combinations
of training groups.
\end{abstract}

\section{Introduction}
At its simplest, training data attribution asks how individual training units contribute to a model's behavior.
However, training data rarely affects a model one example at a time. The usefulness of one training example, document, or data source can depend on what other data was used during training. Two sources may provide similar information, making one less useful when the other is present. In other cases, several sources may only become useful when they appear together, amplify or harm each other~\citep{ghorbani2019data, wang2025data, heo2026interactionaware}.

Understanding these dependencies matters for questions such as which training data contributed to a behavior, what would happen if some data were removed, and which subset of data should be kept. Existing methods approach these questions in different ways. Influence functions approximate how a model prediction changes when the weight of a training example is perturbed \citep{koh2017influence}. TracIn uses gradients recorded along the training trajectory \citep{pruthi2020tracin}, while TRAK provides a scalable gradient-based approach for attributing model predictions to training examples \citep{park2023trak}. MAGIC differentiates through  training  to obtain estimates of the effect of changing training data weights \citep{ilyas2025magic}. All the methods above, when used to predict subset effects by summing fixed
per-example scores, induce additive predictors and fail to capture dependencies.
A training unit's contribution can depend on which other units are
present, particularly when sources provide complementary or redundant
information. We discuss the limitations of predicting joint effects from fixed
per-example scores in Sec.~\ref{sec:examples} and
Appendix~\ref{app:if-trak}.

Instead of directly assigning one score to each training example, datamodels learn to predict how model behavior changes across different subsets of the training data \citep{ilyas2022datamodels}. 
Linear datamodels~\citep{ilyas2022datamodels} assign each training unit a fixed contribution, independent of which other units are present. This formulation supports scalable estimation but cannot capture redundancy or complementarity among training units. Polynomial and kernel datamodels~\citep{zhang2026efficient}
capture nonlinear training data interactions, but explicit
polynomial expansions grow combinatorially with interaction
degree, while standard kernel formulations do not explicitly
exploit low-rank structure among interactions.

To capture high-order dependencies efficiently, we introduce \emph{tensor network datamodels} (TNDMs), which leverage low-rank tensor factorizations to compactly parameterize training data interactions. 
For \(d\) training units, the outcomes over all \(2^d\) training subsets form an order-\(d\) \emph{value tensor}, with one binary mode per training unit. This tensor view naturally suggests a low-rank parameterization.
TNDMs approximate the value tensor through a low-rank factorization learned from observed coalition--outcome pairs. Rather than restricting interaction degree, the factorization constrains the rank of dependence across partitions of the training units, allowing high-order interactions when this dependence admits a compact representation. We consider two parameterizations: TNDM-CP, based on CP decomposition, and TNDM-TT, based on the tensor train
(TT) decomposition.

To make this distinction precise, we characterize datamodels as set functions through two distinct notions of complexity: interaction degree and tensor rank (Sec.~\ref{sec:tndm}). 
We then prove that this regime arises naturally in training data attribution: complementary and redundant training groups can induce high interaction degree while retaining low tensor rank (Sec.~\ref{sec:examples}).

\begin{figure}[t]
\centering
\includegraphics[width=\textwidth]{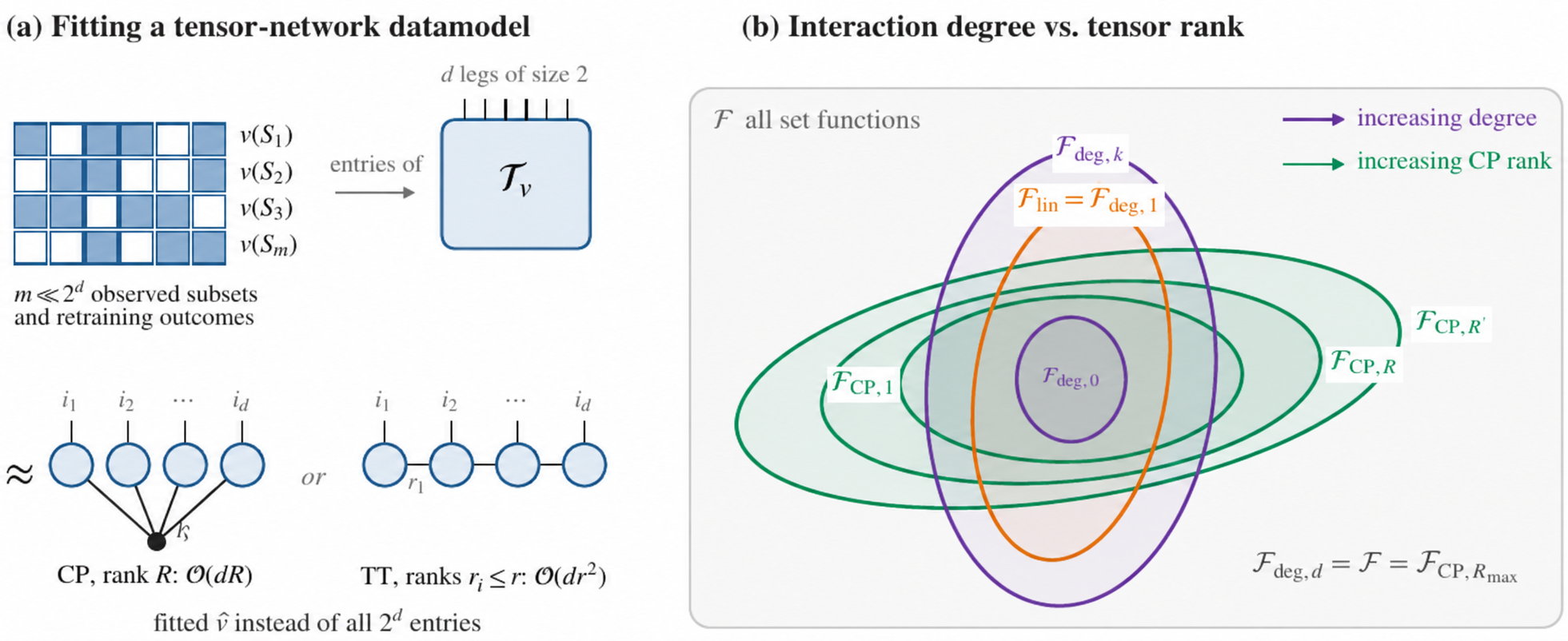}
\caption{\textbf{Tensor-network datamodels.}
(a)~Each observed subset $S_n$ of training groups (a row; shaded groups are included) gives a
retraining outcome $v(S_n)$, one entry of the value tensor $\mathcal T_v$. From $m\ll 2^d$ such
entries, a TNDM fits a low-rank approximation $\widehat v$ of $\mathcal T_v$: a rank-$R$ CP
decomposition (factors joined by a copy tensor, black dot) or a tensor train with ranks $r_j$.
(b)~Interaction degree (purple) and CP rank (green) define two nested families of set functions;
linear datamodels have degree at most one. 
The diagrams illustrate two distinct inductive biases: low interaction degree and low tensor rank, with the latter able to represent high-order interactions compactly.}
\label{fig:overview}

\end{figure}

A single fitted TNDM supports a range of training data attribution
tasks through direct computations on the low-rank representation.
Besides predicting outcomes for unseen training subsets, tensor
contractions yield estimates of individual importance, including
Shapley and Banzhaf values, higher-order interactions, and joint-removal
effects. The attribution computations reuse the learned factors
without additional retraining of the underlying model, while predicted
subset outcomes support data selection.
Figure~\ref{fig:overview} summarizes TNDM;
Appendix~\ref{app:queries} details the attribution computations.

Our experiments demonstrate the practical benefits of capturing
training data interactions through low-rank structure. On controlled
reasoning tasks, TNDMs predict unseen subset outcomes more accurately
than linear, polynomial, and kernel baselines, even with substantially
fewer observed training subsets. TNDM-TT also improves recovery of
complementary training groups and achieves the lowest error in
predicting joint-removal effects. The benefits extend beyond the
controlled reasoning setting: in factual-association fine-tuning,
TNDM-CP outperforms tuned linear, matched-pair, and quadratic
datamodels. Further experiments evaluate prediction and attribution
in LoRA fine-tuning, classifier retraining, and in-context learning
with frozen language models (Sec.~\ref{sec:experiments}).

\paragraph{Our contributions are:}
\begin{enumerate}[leftmargin=*,nosep,topsep=2pt]
    \item
    \textbf{Tensor Network Datamodels.}
    We introduce TNDMs, a family of datamodels representing
training-subset outcomes through low-rank tensors.
TNDMs capture complementarity and redundancy among training units,
including interactions involving many units. Fitted TNDMs also support Shapley, Banzhaf, and interaction
queries without additional retraining.

    \item
\textbf{High interaction degree with low tensor rank.}
For reasoning tasks with complementary or redundant training groups,
we prove that interaction degree can grow while tensor rank remains
bounded. The results establish a theoretical motivation for capturing
high-order training data interactions through low-rank representations
(Secs.~\ref{sec:tndm} and~\ref{sec:examples}).
\item
\textbf{Accurate prediction and attribution from limited observations.}
On a controlled reasoning task, TNDM-CP and TNDM-TT
outperform all evaluated non-tensor baselines while using half as many
observed subsets.
 A single fitted TNDM-CP achieves the lowest mean error
among the compared methods for Shapley values and
pairwise and third-order Shapley interactions
at a shared observation budget.
\end{enumerate}

\section{Datamodels and multilinear representations}
\label{sec:tndm}
In this section, we introduce datamodels and show that every datamodel on subsets of a fixed training collection has a unique affine multilinear representation. This representation has $2^d$ possible terms, but its coefficients may have structure that allows a more compact description. We consider two forms of structure: low interaction degree and low tensor rank.

\subsection{Datamodels as set functions}

Let $N=\{1,\ldots,d\}$ index training datasets
$\mathcal D_1,\ldots,\mathcal D_d$, each containing one or more
examples. For a subset $S\subseteq N$, let
$f_S=\operatorname{Train}(\bigcup_{j\in S}\mathcal D_j)$ denote
the model trained on the selected datasets. A scalar outcome of
$f_S$, such as its accuracy or prediction on an evaluation query,
defines a set function $v:2^N\to\mathbb R$. We denote the set of all such set functions by \(\mathcal{F}\). A datamodel $\widehat v$ learns from observed
subset--outcome pairs to predict $v(S)$ for unobserved
subsets~\citep{ilyas2022datamodels}.
We identify each subset $S$ with its inclusion vector
$\mathbf z\in\{0,1\}^d$, where $z_j=\mathbf 1\{j\in S\}$,
and use $v(S)$ and $v(\mathbf z)$ interchangeably.

An (affine) linear datamodel takes the form
$\widehat v(\mathbf z)=\theta_{\emptyset}
+\sum_{j=1}^d\theta_jz_j$,
where $\theta_{\emptyset}$ is a baseline and $\theta_j$
is the contribution assigned to training group $j$.
Adding group $j$ therefore changes the prediction by the same
amount $\theta_j$, regardless of which other groups are present.

To capture dependencies between training groups, we extend
the linear model with interaction terms. Every set function
has a unique multilinear representation
\begin{equation}
    v(\mathbf z)
    = c_{\emptyset}
    + \sum_{i=1}^d c_{\{i\}}z_i
    + \sum_{i<j} c_{\{i,j\}}z_i z_j
    + \sum_{i<j<k} c_{\{i,j,k\}}z_i z_j z_k
    + \cdots,
    \label{eq:multilinear-expansion}
\end{equation}
where the expansion includes terms up to order $d$.
Pairwise terms allow a group's contribution to depend on
whether another group is present, while higher-order terms
describe interactions involving larger combinations of groups.
The existence and uniqueness of the multilinear representation are
established in Appendix~\ref{app:set-function-foundations}.

Equivalently, Eq.~\eqref{eq:multilinear-expansion} can be written
as \(v(\mathbf{z})=\sum_{A\subseteq N}c_A\prod_{j\in A}z_j\), where
\(A\) indexes the training groups appearing in each monomial. For
\(A=\emptyset\), we adopt the convention
\(\prod_{j\in\emptyset}z_j=1\), so the corresponding term is the
constant \(c_{\emptyset}\). Since every variable appears with exponent
at most one, the representation is affine in each variable when the
others are fixed. We collect the coefficients in the
\emph{coefficient tensor}
\(\mathcal{C}_v\in\mathbb{R}^{2\times\cdots\times2}\), defined by
\((\mathcal{C}_v)_{i_1\cdots i_d}=c_{\{j:i_j=1\}}\), where \(i_j=2\)
selects the constant component in mode \(j\). The multilinear
representation can then be written compactly as\footnote{We use the tensor notation and mode-product
conventions of \citet{kolda2009tensor}, except that contraction with a
vector is written without transposing the vector.}
\[
    v(\mathbf z)
    =
    \mathcal{C}_v
    \times_1\begin{bmatrix}z_1\\1\end{bmatrix}
    \cdots
    \times_d\begin{bmatrix}z_d\\1\end{bmatrix}.
\]
We also define the \emph{value tensor}
\(\mathcal{T}_v\in\mathbb{R}^{2\times\cdots\times2}\) by
\((\mathcal{T}_v)_{i_1\cdots i_d}
=v(\{j:i_j=1\})\). Its entries are therefore the values of \(v\) on
all subsets of \(N\), as illustrated in
Fig.~\ref{fig:overview}(a). The tensors \(\mathcal{C}_v\) and
\(\mathcal{T}_v\) are related by an invertible change of basis applied
independently to each mode and thus have the same CP and TT ranks; see
Appendix~\ref{app:set-function-foundations}.

\subsection{Interaction degree and Tensor rank}
Interaction degree and tensor rank describe two different
forms of structure in the coefficient tensor $\mathcal C_v$.
Interaction degree limits the number of training groups
appearing together in a nonzero term. Tensor rank measures
the complexity of a tensor factorization. We introduce
both notions below; Appendix~\ref{app:set-function-foundations}
provides the full mathematical development and proofs.

\begin{definition}[Interaction degree]
Let $v\in\mathcal F$ have the unique expansion
$v(\mathbf z)=\sum_{A\subseteq N}c_A\prod_{j\in A}z_j$.
For nonzero $v$, the \emph{interaction degree} of $v$ is
$\operatorname{deg}(v)=\max\{|A|:c_A\neq0\}$.
We set $\operatorname{deg}(0)=0$ and denote the class of
set functions of interaction degree at most $k$ by
$\mathcal F_{\deg,k}
=\{v\in\mathcal F:\operatorname{deg}(v)\leq k\}$.
\label{def:degree}
\end{definition}

A datamodel of interaction degree at most $k$ includes
interactions among at most $k$ training groups.
Linear datamodels are exactly the functions of interaction
degree at most one, while
quadratic datamodels only allow interactions involving at
most two training groups.

Alternatively, low tensor rank provides another way to represent
$\mathcal C_v$ compactly. A CP representation expresses the datamodel as a sum of
products of simple affine functions.

\begin{definition}[CP rank]
The \emph{CP rank} of a nonzero set function $v\in\mathcal F$
is the smallest integer $R$ for which
\[
    v(\mathbf z)
    =\sum_{r=1}^R\prod_{j=1}^d
    (\alpha_{rj}+\beta_{rj}z_j),
    \qquad \alpha_{rj},\beta_{rj}\in\mathbb R.
\]
We set $\operatorname{rank}_{\mathrm{CP}}(0)=0$ and write
$\mathcal F_{\mathrm{CP},R}
=\{v\in\mathcal F:\operatorname{rank}_{\mathrm{CP}}(v)\leq R\}$.
\label{def:cp_rank}
\end{definition}

Each rank-one product uses \(2d\) raw parameters. A representation with
\(R\) products therefore uses \(2dR\) raw parameters and can include
interactions of any degree up to \(d\). The product representation in
Definition~\ref{def:cp_rank} is equivalent to a rank-\(R\) CP
decomposition of \(\mathcal{C}_v\), and, through the change of basis
above, of \(\mathcal{T}_v\). Thus, low interaction degree and low CP
rank provide two distinct ways of imposing structure on a datamodel; replacing CP rank with Tucker
(multilinear), TT, or hierarchical Tucker rank would yield alternative
forms of structured low-rank bias.

Interaction degree and CP rank measure different properties.
For example, $v(\mathbf z)=\prod_{j=1}^d z_j$ has interaction
degree $d$ but CP rank one. Figure~\ref{fig:overview}(b) illustrates the
degree-bounded and rank-bounded families.

\subsection{Tensor Network Datamodels}
\label{sec:tndm-model}

We now define the proposed model. Let
\(\operatorname{TN}_{\mathcal G}(\mathbf r)\) denote the family of
tensors represented by a tensor network with topology \(\mathcal G\)
and internal dimensions \(\mathbf r\); that is, by contracting a
collection of smaller parameter tensors arranged according to
\(\mathcal G\).

\begin{definition}[Tensor network datamodel]
    \label{def:tndm}
A \emph{tensor network datamodel} (TNDM) with topology \(\mathcal G\)
and internal dimensions \(\mathbf r\) is a set function of the form
\[
    \widehat v_{\Theta}(\mathbf z)
    =
    \mathcal C_{\Theta}
    \times_1\begin{bmatrix}z_1\\1\end{bmatrix}
    \cdots
    \times_d\begin{bmatrix}z_d\\1\end{bmatrix},
    \qquad
    \mathcal C_{\Theta}
    \in \operatorname{TN}_{\mathcal G}(\mathbf r),
\]
where \(\Theta\) denotes the collection of tensor network factors.
\end{definition}

Given observed coalition--outcome pairs
\(\mathcal O=\{(\mathbf z^{(s)},y_s)\}_{s=1}^n\), where
\(y_s=v(\mathbf z^{(s)})\), we learn the factors by solving
\[
    \widehat\Theta
    \in
    \arg\min_{\Theta}
    \frac{1}{n}\sum_{s=1}^n
    \bigl(
        \widehat v_{\Theta}(\mathbf z^{(s)})-y_s
    \bigr)^2
    +\lambda\Omega(\Theta),
\]
where \(\Omega\) is a regularizer. Ranks and regularization parameters
are selected using held-out coalition observations.

We consider two tensor network parameterizations. In
\emph{TNDM-CP}, \(\mathcal C_{\Theta}\) has CP rank at most \(R\), so
Definition~\ref{def:tndm} reduces to the sum of products of affine
functions in Definition~\ref{def:cp_rank}. In \emph{TNDM-TT},
\(\mathcal C_{\Theta}\) is represented by a chain of tensor cores with
TT ranks \(r_1,\ldots,r_{d-1}\). For uniform rank \(r\), these
parameterizations use respectively \(\mathcal O(dR)\) and
\(\mathcal O(dr^2)\) parameters, rather than the \(2^d\) entries of an
unstructured tensor. Their explicit tensor network constructions are
given in Appendix~\ref{app:tndm-definitions}.

Each observed outcome \(y_s\) is one entry of the value tensor
\(\mathcal T_v\). Fitting a TNDM from \(n\ll 2^d\) such observations
can therefore be viewed as a structured tensor-completion problem.
When the underlying datamodel is approximately low rank, the shared
tensor network factors provide an inductive bias for predicting
unobserved coalitions, in addition to making the representation
tractable.

Crucially, low tensor rank does not restrict the datamodel to
low-order interactions. A TNDM can represent functions with high
interaction degree whenever their interaction coefficients possess a
compact tensor network structure. Section~\ref{sec:examples} shows
that such high-degree, low-rank datamodels arise naturally in data
attribution for \(k\)-hop reasoning tasks.

\section{Low-rank structure in reasoning-based data attribution}
\label{sec:examples}

We now show that high-interaction-degree, low-rank set functions arise
naturally in reasoning-based data attribution.

Let $q_1,\ldots,q_m$ be evaluation queries.
For each query $q_a$, let $P_a\subseteq N$ index the
training groups required to answer $q_a$.
We consider the stylized setting in which
\(f_S\) answers \(q_a\) correctly if and only if \(P_a\subseteq S\).
The resulting average-accuracy datamodel is
$v(S)=\frac{1}{m}\sum_{a=1}^m\mathbf 1\{P_a\subseteq S\}$.

The following theorem allows arbitrary positive query
weights and possibly overlapping support sets.

\begin{theorem}
    \label{thm:support-rank-degree}
Let \(P_1,\ldots,P_m\) be nonempty subsets of \(N\), and let
\(c_1,\ldots,c_m>0\). Consider the set function
$
    v(S)
    =
    \sum_{a=1}^m c_a\,\mathbf{1}\{P_a\subseteq S\},
    \qquad S\subseteq N.
$
Then
$
    \operatorname{deg}(v)=\max_{1\leq a\leq m}|P_a|
    \qquad\text{and}\qquad
    \operatorname{rank}_{\mathrm{CP}}(v)\leq m.
$
If the support sets \(P_1,\ldots,P_m\) are pairwise disjoint, then the
CP-rank bound is tight:
$
    \operatorname{rank}_{\mathrm{CP}}(v)=m.
$
These conclusions remain valid when some elements of \(N\) do not
belong to any support set.
\end{theorem}

The proof is given in Appendix~\ref{app:set-function-foundations}.

\usetikzlibrary{calc,positioning,decorations.pathreplacing,fit}
\definecolor{degc}{RGB}{106,61,154}
\definecolor{rankc}{RGB}{23,138,88}
\definecolor{ink}{RGB}{60,60,60}
\providecommand{\ind}{\mathbf{1}}
\providecommand{\rankcp}{\operatorname{rank}_{\mathrm{CP}}}

\begin{figure}[t]
\centering
\begin{tikzpicture}[
  font=\small,
  grp/.style={draw=black!80, line width=0.5pt, fill=black!4,
              minimum width=0.62cm, minimum height=0.5cm, inner sep=0pt},
  qry/.style={draw=ink, line width=0.6pt, rounded corners=1.5pt, fill=white,
              minimum width=0.66cm, minimum height=0.44cm, inner sep=1pt},
  sup/.style={ink, line width=0.55pt},
  lab/.style={font=\footnotesize, text=black!65},
  stat/.style={line width=0.6pt, rounded corners=1.5pt, font=\footnotesize,
               inner xsep=4pt, inner ysep=3pt, align=center, minimum width=2.9cm},
  brc/.style={decorate, decoration={brace, mirror, amplitude=3.5pt}, rankc, line width=0.6pt},
  ttl/.style={font=\small\bfseries, anchor=west},
  enc/.style={draw=ink, dashed, line width=0.5pt, rounded corners=2pt, inner sep=3pt},
]
\foreach \i in {1,...,8}{
  \pgfmathsetmacro{\x}{(\i-1)*0.75 + (\i>4)*0.3}
  \node[grp] (D\i) at (\x,0) {$D_{\i}$};
}
\coordinate (g1) at ($(D2)!0.5!(D3)$);
\coordinate (g2) at ($(D6)!0.5!(D7)$);
\coordinate (mida) at ($(D4)!0.5!(D5)$);
\node[ttl] at ($(D1.west)+(-0.6,2.45)$) {(a) Required support sets (Theorem~\ref{thm:support-rank-degree})};

\node[qry] (q1) at ($(g1)+(0,1.3)$) {$q_1$};
\node[qry] (q2) at ($(g2)+(0,1.3)$) {$q_2$};
\foreach \i in {1,...,4} \draw[sup] (q1.south) -- (D\i.north);
\foreach \i in {5,...,8} \draw[sup] (q2.south) -- (D\i.north);
\node[lab, above=0pt of q1] {$P_1=\{1,2,3,4\}$};
\node[lab, above=0pt of q2] {$P_2=\{5,6,7,8\}$};

\def\ey{-1.1}
\node (t1) at ($(g1)+(0,\ey)$) {$c_1\,\textcolor{degc}{z_1z_2z_3z_4}$};
\node (t2) at ($(g2)+(0,\ey)$) {$c_2\,\textcolor{degc}{z_5z_6z_7z_8}$};
\node at ($(mida)+(0,\ey)$) {$+$};
\node[anchor=east] at ($(t1.west)+(0.05,0)$) {$v(S)=$};
\draw[brc] ($(t1.south west)+(0.05,0)$) -- ($(t1.south east)+(-0.05,0)$)
  node[midway, below=3pt, font=\footnotesize\itshape, text=rankc] {one product term};
\draw[brc] ($(t2.south west)+(0.05,0)$) -- ($(t2.south east)+(-0.05,0)$)
  node[midway, below=3pt, font=\footnotesize\itshape, text=rankc] {one product term};

\def\sy{-2.8}
\node[stat, draw=degc, fill=degc!5] at ($(g1)+(0,\sy)$)
  {\textcolor{degc}{interaction degree}\\[1pt]
   $\deg(v)=\max_a|P_a|=\textcolor{degc}{\mathbf{4}}$\\[1pt]
   {\color{black!60} support-set size}};
\node[stat, draw=rankc, fill=rankc!5] at ($(g2)+(0,\sy)$)
  {\textcolor{rankc}{CP rank}\\[1pt]
   $\rankcp(v)=m=\textcolor{rankc}{\mathbf{2}}$\\[1pt]
   {\color{black!60} number of support sets}};

\draw[black!22, line width=0.4pt] (6.45,2.6) -- (6.45,-3.55);

\def\xb{8.05}
\foreach \i in {1,...,4}{
  \pgfmathsetmacro{\x}{\xb + (\i-1)*1.0 + (\i>2)*0.6}
  \node[grp] (E\i) at (\x,0) {$D_{\i}$};
}
\coordinate (midb) at ($(E2)!0.5!(E3)$);
\node[ttl] at (6.75,{2.45}) {(b) Redundant support (Proposition~\ref{prop:coverage-tt})};

\node[enc, fit=(E1)(E2)] (G1) {};
\node[enc, fit=(E3)(E4)] (G2) {};
\node[qry] (q) at ($(midb)+(0,1.3)$) {$q$};
\draw[sup] (q.south) -- (G1.north);
\draw[sup] (q.south) -- (G2.north);
\node[lab, below=1pt of G1] {$G_1=\{1,2\}$};
\node[lab, below=1pt of G2] {$G_2=\{3,4\}$};
\node[lab, align=left, anchor=west] at ($(q.east)+(0.15,0)$) {one group\\from each $G_h$};

\node (eqb) at ($(midb)+(0,\ey-0.15)$)
  {$v(S)=(z_1+z_2-z_1z_2)\,(z_3+z_4-z_3z_4)$};
\node[lab] at ($(midb)+(0,\ey-0.68)$)
  {9 nonzero monomials; highest: $\textcolor{degc}{z_1z_2z_3z_4}$};

\node[stat, draw=degc, fill=degc!5] at ($(midb)+(-1.58,\sy)$)
  {\textcolor{degc}{interaction degree}\\[1pt]
   $\deg(v)=\sum_h|G_h|=\textcolor{degc}{\mathbf{4}}$\\[1pt]
   {\color{black!60} total support size}};
\node[stat, draw=rankc, fill=rankc!5] at ($(midb)+(1.58,\sy)$)
  {\textcolor{rankc}{TT ranks}\\[1pt]
   $r_j\le\textcolor{rankc}{\mathbf{2}}$\\[1pt]
   {\color{black!60} each $G_h$ contiguous}};

\node[font=\small\bfseries] at ($(D1.west)!0.5!(E4.east)+(0,-4.0)$)
  {High interaction degree does not imply high tensor rank.};
\end{tikzpicture}

\caption{\textbf{Interaction degree and tensor rank measure different properties.}
High-order interactions can have low tensor rank.
(a) Each query requires four training groups together,
giving interaction degree four but CP rank two.
(b) A query requires one group from each pair of alternatives,
giving interaction degree four but TT ranks at most two
when each pair is consecutive in the ordering.}
\label{fig:degree-rank}
\end{figure}

\begin{corollary}
    \label{cor:k-hop-rank-degree-main}
The average-accuracy datamodel satisfies
$
    \operatorname{deg}(v)=\max_{1\leq a\leq m}|P_a|
$ and 
$
    \operatorname{rank}_{\mathrm{CP}}(v)\leq m.
$
If the query support sets are pairwise disjoint, then
\(\operatorname{rank}_{\mathrm{CP}}(v)=m\).

In particular, if query \(q_a\) requires \(k_a\) reasoning facts stored
in \(k_a\) distinct datasets, then
\(\operatorname{deg}(v)=\max_a k_a\). If every query is a \(k\)-hop
reasoning task of this form, then \(\operatorname{deg}(v)=k\),
regardless of whether the query support sets overlap.

Datasets with indices outside \(\bigcup_{a=1}^mP_a\) affect
neither the interaction degree nor the CP rank of \(v\).
\end{corollary}
The corollary separates the two complexity measures: interaction degree
is controlled by the largest query-support size, whereas the CP-rank
upper bound is controlled by the number of queries—or, after identical
supports are combined, by the number of distinct support patterns.
Thus, with the number of queries fixed, increasing the number of
reasoning steps can increase interaction degree without increasing the
CP-rank bound.

Similarly, redundant training groups can also produce high-order
interactions with low tensor rank.
Figure~\ref{fig:degree-rank}(b) illustrates a query requiring
one group from each set of alternatives. Proposition~\ref{prop:coverage-tt} 
in Appendix~\ref{app:set-function-foundations} formalizes this case:
let $G_1,\ldots,G_k$ be nonempty, pairwise-disjoint sets
of interchangeable training groups, and suppose that the
query is answered correctly exactly when $S$ contains
at least one group from every $G_h$. Then
$v(S)=\prod_{h=1}^k[1-\prod_{j\in G_h}(1-z_j)]$. The interaction degree is $\sum_{h=1}^k|G_h|$, while
the TT ranks are at most two when the groups in each
$G_h$ are consecutive in the tensor ordering.

Theorem~\ref{thm:support-rank-degree} and the redundant-support example
expose the principal modeling advantage of TNDMs. With the number of
query-support patterns fixed, enlarging the supports can increase
interaction degree without increasing the CP-rank bound. Similarly,
reasoning over interchangeable data sources can produce high-degree
functions with small TT ranks. A degree-bounded datamodel may exclude
such functions, whereas a TNDM can retain the relevant high-order
interactions while controlling complexity through tensor rank.
Additional results on datamodel structure and fixed attribution scores
are provided in Appendices~\ref{app:structural-theory}
and~\ref{app:if-trak}.

\section{Experiments}
\label{sec:experiments}
We experimentally investigate four questions:
(1) Can TNDMs predict the outcome of training on unseen data subsets? (Sec.~\ref{sec:khop-prediction}; Table~\ref{tab:d24-main-readable})
(2) Can TNDMs recover interactions between training groups? (Sec.~\ref{sec:khop-interactions}; Fig.~\ref{fig:khop-evidence})
(3) Do TNDMs remain useful outside our controlled reasoning task? (Sec.~\ref{sec:finetuning-results}; Tables~\ref{tab:sportmusic-prediction} and~\ref{tab:agnews-lora-attribution})
(4) Can one fitted TNDM support several attribution queries? (Sec.~\ref{sec:query-experiments}; Fig.~\ref{fig:main-har-attribution})

All methods receive the same observed subset--outcome pairs, and hyperparameters
are selected using validation subsets only.
Appendices~\ref{app:surrogates} and~\ref{app:protocol} give the full models,
training, and evaluation details.
\subsection{Prediction and joint-removal effects}
\label{sec:khop-experiment}
\label{sec:khop-prediction}

We first use a controlled reasoning task, inspired by the compositional setting of~\citet{wang2024grokked}, where a query specifies a sequence of three relations to compose, e.g.,
\(2\xrightarrow{+4}6\xrightarrow{+7}13\xrightarrow{+3}16\).
Different training groups provide the same relation, so
some groups are redundant, while groups providing different relations are
complementary. Each training subset is used to train a separate transformer from scratch. The full task construction is described in Appendix~\ref{app:protocol}. 

We fit datamodels on observed coalition–outcome pairs and test how well they predict unseen subsets.
On the 24-group task, TNDM-CP and TNDM-TT achieve
lower prediction error, outperforming linear, polynomial, sparse polynomial, and RBF
baselines. Notably, with 512 observed subsets, both TNDMs outperform the
baselines even when they use 1,024 observations
(Table~\ref{tab:d24-main-readable}). 

We also test a joint-removal intervention, removing all groups that teach one required relation and asking
each datamodel to predict the resulting change in
performance. TNDM-TT gives the lowest error among the compared methods
(Table~\ref{tab:d24-main-readable}).

\begin{table}[t]
\centering
\small
\setlength{\tabcolsep}{4pt}
\begin{tabular}{lccc}
\toprule
Datamodel
& NMSE, $m=512$
& NMSE, $m=1{,}024$
& Joint-removal RMSE \\
\midrule
Linear
& $0.687\pm0.001$
& $0.669\pm0.001$
& $2.741\pm0.016$ \\

Degree-2 ridge
& $0.467\pm0.001$
& $0.348\pm0.001$
& $2.125\pm0.007$ \\

Degree-3 ridge
& $0.510\pm0.002$
& $0.342\pm0.002$
& $3.869\pm0.019$ \\

Degree-3 OMP
& $0.635\pm0.016$
& $0.271\pm0.021$
& $6.292\pm0.057$ \\

RBF
& $0.386\pm0.001$
& $0.257\pm0.001$
& $2.100\pm0.035$ \\

TNDM-CP
& $\mathbf{0.148\pm0.026}$
& $0.087\pm0.002$
& $1.939\pm0.258$ \\

TNDM-TT
& $0.188\pm0.004$
& $\mathbf{0.071\pm0.006}$
& $\mathbf{1.395\pm0.043}$ \\
\bottomrule
\end{tabular}
\caption{Prediction and joint-removal effects on the
24-group reasoning task. Entries show means $\pm$ sample
SD over three learner seeds; lower is better.
The last column uses $m=1{,}024$ and measures error in the
utility change when all sources for one required relation
are removed. Bold marks the lowest mean in each column. At 512 observations, both TNDMs outperform every
non-tensor baseline evaluated at 1,024 observations;
TNDM-TT also achieves the lowest joint-removal error.}
\label{tab:d24-main-readable}
\end{table}
In this setting, with three interchangeable providers for each of two required
relations, the ideal coverage function has degree six and CP rank
four, while admitting TT ranks at most two under a provider-contiguous
ordering (Prop.~\ref{prop:coverage-tt};
Thm.~\ref{thm:general-support-structure}).
The prediction and joint-removal gains are consistent with the
practical value of this high-degree, low-rank structure.

\subsection{Recovering interactions and selecting data}
\label{sec:khop-interactions}

We next ask whether the fitted datamodel can identify groups that are useful
together.
In the 12-group task, a useful triple contains one group for each required
relation. We rank all triples using the interaction predicted by each datamodel.
As more subset outcomes are observed, TNDM-TT becomes substantially better at
recovering these complementary triples than the degree-bounded and RBF
baselines (Fig.~\ref{fig:khop-evidence}(a)).

We also test whether individual Shapley values from the fitted datamodel can
select useful groups. At the largest budget, Linear + TNDM-TT, a hybrid model that augments a fitted linear predictor with a TT residual correction (Appendix~\ref{app:surrogate-fitting}), gives the highest
mean completeness among the compared methods
(Fig.~\ref{fig:khop-evidence}(b)). However, the best method changes across
budgets and attribution rules.
Improved recovery of complementary triples supports the practical
value of representing higher-order interactions through low tensor
rank rather than a fixed degree cutoff.
Strong selection performance from Degree-2 does not imply that the
underlying utility function is low degree: a model can select useful subsets without 
accurately reconstructing the full subset function.

\begin{figure}[t]
\centering
\includegraphics[width=\linewidth]{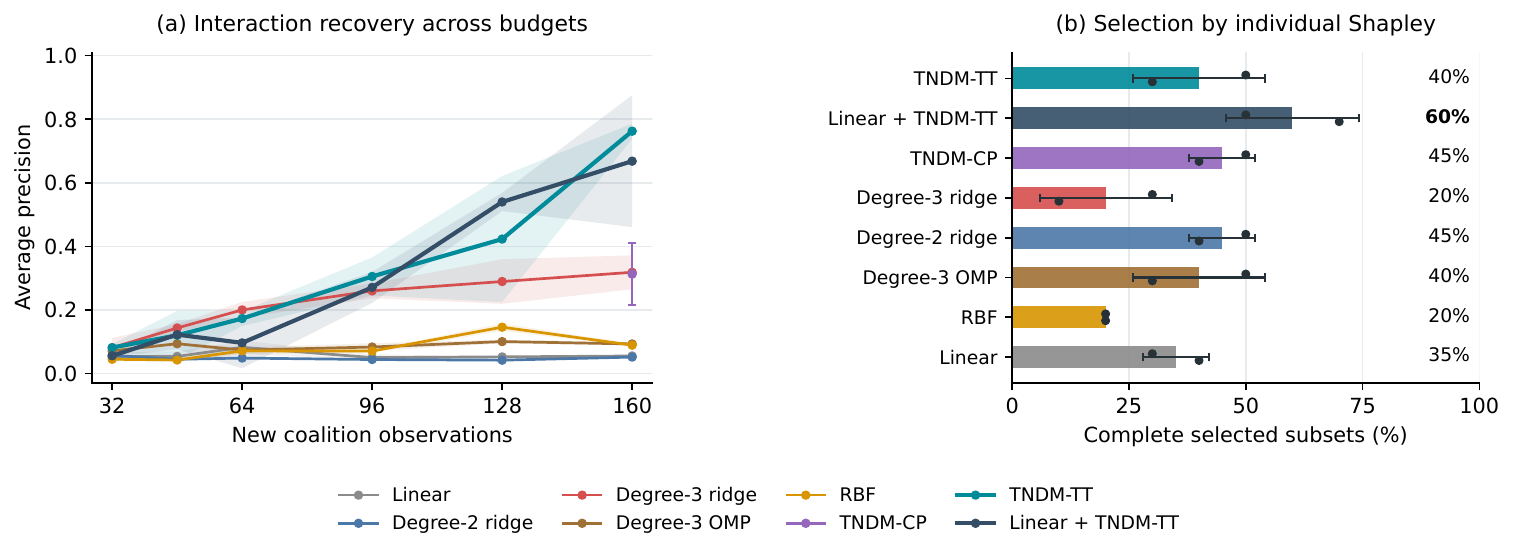}
\caption{Interaction recovery and Shapley-based selection on the 12-group reasoning task.
(a) Average precision for ranking complete triples by predicted third-order differences, across observation budgets.
(b) Completeness of subsets formed by selecting the
three groups with the largest individual Shapley values
at 160 new observations.
Curves and bars show means across two learner seeds, each averaged over ten relation sequences;
bands and whiskers show sample SD, and dots in (b) show both seed means.
 TNDM-TT improves interaction recovery at larger budgets,
while Linear + TNDM-TT achieves the highest mean
selection completeness at 160 new observations.}
\label{fig:khop-evidence}
\label{fig:shapley-selection-main}
\end{figure}

Finally, we rank candidate subsets directly using their predicted utility.
Here several nonlinear methods perform similarly: TNDM-TT, degree-3 ridge, and RBF
all identify complete top-ranked subsets in the primary comparison.
Thus, better interaction recovery does not always imply better subset selection.
The full comparison is given in Appendices~\ref{app:interaction-results} and~\ref{app:shapley-selection}.

\subsection{Pretrained-model fine-tuning}
\label{sec:finetuning-results}

The previous experiments were designed to contain strong interactions.
We therefore test whether TNDMs also help in more standard fine-tuning settings.
\paragraph{Sport--Music.}
We adapt the Sport--Music task of \citet{huang2025generalization}.
Each subset chooses which sport and music examples are used to fine-tune
Qwen2-7B. TNDM-CP predicts held-out subset outcomes better than the linear,
matched-pair, and degree-2 baselines, and an ensemble of TNDMs improves further
(Table~\ref{tab:sportmusic-prediction}). The improvement from a single rank-five CP model supports the
usefulness of the low-rank modeling assumption beyond the
controlled support constructions.

\paragraph{AG News.}
We fine-tune RoBERTa with LoRA on subsets of AG News~\citep{zhang2015character}. Generic text-classification fine-tuning need not exhibit
the low-rank, high-degree interaction structure motivating
the theoretical results. Nevertheless,
TNDMs provide competitive attribution overall, with
TNDM-TT achieving the lowest median Shapley and Banzhaf
errors among the compared fitted datamodels at 128 observations (Table~\ref{tab:agnews-lora-attribution}).
Across-budget results favor degree-truncated and RBF baselines;
comparisons with dedicated attribution estimators appear in
Appendix~\ref{app:finetuning-evidence}.

\begin{table}[t]
\centering
\begin{minipage}[t]{.49\linewidth}
\vspace{0pt}\centering\footnotesize
\setlength{\tabcolsep}{3pt}
\begin{tabular}{@{}lc@{}}
\toprule
Datamodel & Held-out $R^2$ $\uparrow$ \\
\midrule
Linear & $0.579\pm0.048$ \\
Matched-pair & $0.556\pm0.054$ \\
Degree-2 & $0.565\pm0.109$ \\
TNDM-CP, single & $0.668\pm0.084$ \\
TNDM-CP, ensemble & $\mathbf{0.734\pm0.061}$ \\
\bottomrule
\end{tabular}
\caption{Sport--Music prediction: mean $R^2$ $\pm$ sample SD over
15 shared coalition splits. Baselines are tuned. Single-fit TNDM-CP
uses rank five; the ranks 3--6 ensemble uses additional fitting compute.
Bold marks the highest mean. TNDM-CP improves prediction over the linear, matched-pair, and degree-2 baselines on the Sport–Music fine-tuning task.}
\label{tab:sportmusic-prediction}
\end{minipage}\hfill
\begin{minipage}[t]{.49\linewidth}
\vspace{0pt}\centering\footnotesize
\setlength{\tabcolsep}{3pt}
\begin{tabular}{@{}lcc@{}}
\toprule
Datamodel & Banzhaf $\downarrow$ & Shapley $\downarrow$ \\
\midrule
Linear & 0.1701 & 0.3555 \\
Degree-2 & 0.2013 & 0.3721 \\
RBF & 0.1724 & 0.3368 \\
TNDM-CP & 0.1736 & 0.3358 \\
TNDM-TT & \textbf{0.1634} & \textbf{0.3285} \\
\bottomrule
\end{tabular}
\caption{AG News RoBERTa--LoRA attribution at 128 observations:
median relative $\ell_2$ error over six adapter ranks, three learner
seeds, and 20 observation-sampling repeats per combination.
Bold marks the lowest median among these datamodels. TNDM-TT has the lowest median Banzhaf and Shapley error among the fitted datamodels.}
\label{tab:agnews-lora-attribution}
\end{minipage}
\end{table}

\begin{figure}[ht]
\centering
\includegraphics[width=\linewidth]{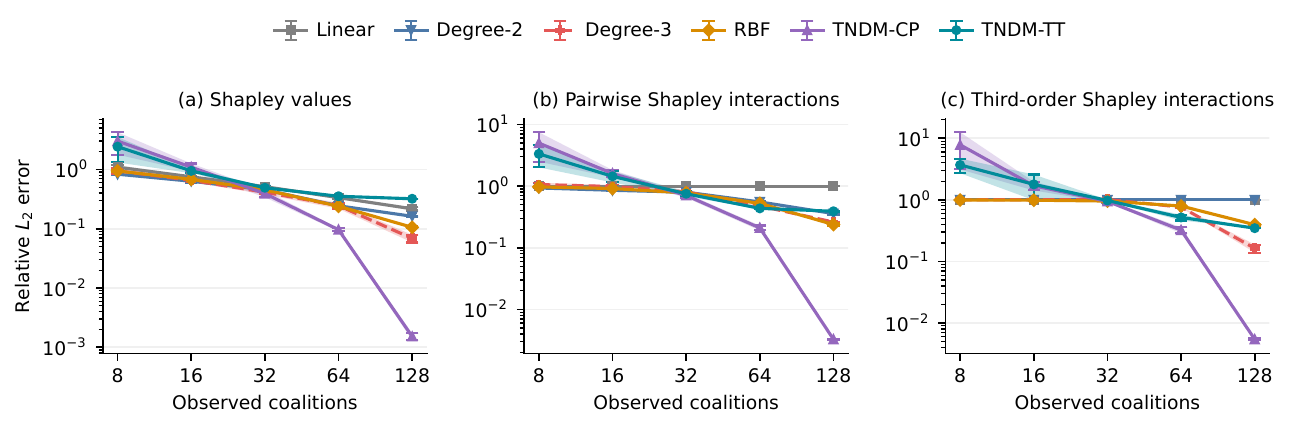}
\caption{ \textbf{Training data attribution on a modular
retraining task.} Relative $\ell_2$ error for Shapley values and pairwise
and third-order Shapley interaction indices over eight
training groups. Exact references use all 256 retrained
coalition outcomes. All queries reuse the same fitted
surrogate, and budgets include validation observations.
Errors are averaged over 20 coalition-sampling repetitions
within each of three learner seeds. Curves show the mean
of the three seed averages; bands and bars show $\pm1$
sample SD across those averages. Both axes are logarithmic;
TNDM-TT uses additive-informed initialization.
At 64 observations, TNDM-CP achieves the lowest mean
error among the plotted methods for all three quantities.}
\label{fig:main-har-attribution}
\end{figure}

\subsection{Several attribution quantities from one fit}
\label{sec:query-experiments}

A fitted TNDM models the subset-to-outcome function, allowing the
same fit to answer several attribution questions. We evaluate this
on a controlled modular retraining task with eight training groups.
For each coalition, we retrain a conditional-logit learner; all
256 coalition outcomes provide exact attribution references.
From each fitted surrogate, we compute Shapley values and pairwise
and third-order Shapley interaction indices
(Fig.~\ref{fig:main-har-attribution}).
TNDM-CP is competitive with the regression baselines at
32 observations. At 64 observations, it achieves the lowest mean error among
the plotted methods for all three quantities.
One fitted CP surrogate thus recovers both individual
contributions and higher-order interactions from the
same limited set of retraining outcomes.
Appendix~\ref{app:modular-attribution} provides the task and evaluation details.
 \paragraph{Additional experiments.}
Appendices~\ref{app:broader-prediction}
and~\ref{app:large-icl} report further prediction comparisons,
with gains that vary across settings.
Appendix~\ref{app:completion-results} reports the enumerated
eight-group prediction experiment, and
Appendix~\ref{app:arithmetic-dynamics} explores attribution
and tensor compressibility during learning.

\section{Related Work}

\paragraph{Datamodels and training data attribution.}
Datamodels predict outcomes from training-subset membership
\citep{ilyas2022datamodels}, supporting model comparison
\citep{shah2023modeldiff} and data selection
\citep{engstrom2024dsdm}.
Data Shapley and Data Banzhaf summarize marginal contributions
\citep{ghorbani2019data,wang2023databanzhaf}, while influence
functions, TracIn, TRAK, and SOURCE estimate training data effects
from fitted models or training trajectories
\citep{koh2017influence,pruthi2020tracin,park2023trak,bae2024source}.
TNDMs use low-rank tensors to capture  joint
effects among complementary and redundant training groups. These
dependencies can remain compact even at high interaction
degree, supporting prediction, attribution, and selection
from one fitted subset function.

\paragraph{Nonlinear datamodels and tensor regression.}
Kernel surrogates model nonlinear training data dependencies
\citep{zhang2026efficient}; TNDMs represent these dependencies
through low-rank interaction tensors.
CP, TT, and hierarchical trees compactly represent high-order
tensors \citep{kolda2009tensor,oseledets2011tt,Grasedyck2010}.
Exponential Machines already use TT coefficients to represent
interactions of all orders \citep{novikov2016exponential}.
TNDMs build on this representation, connecting tensor rank
and interaction degree to training data support structure
and enabling reusable attribution queries from a surrogate
fitted to retraining outcomes.

\paragraph{Tensor networks for attribution.}
TN-SHAP and TN-SHAP-G use tensor network surrogates for
masked-input attribution
\citep{heidari2026tnshap,heidari2026tnshapg}.
In TNDMs, tensor modes index training groups and entries record
outcomes after training on their subsets. The resulting
attributions describe training data interventions rather than
input-feature masking. 
\section{Limitations}
TNDMs offer compact models of training data interactions.
However, empirical gains depend on task structure,
rank selection, and optimization, and retraining remains costly.
Topology also matters: TT requires a suitable group ordering,
while tensor trees require a suitable hierarchy.
Structural knowledge, especially in scientific applications,
can guide these choices.

\section{Conclusion}
Tensor network datamodels capture higher-order training data
interactions through compact low-rank representations.
The theory separates tensor rank from interaction degree
for complementary and redundant training groups.
Experiments demonstrate gains in subset prediction and
interaction recovery, while one fitted surrogate supports
Shapley, Banzhaf, and higher-order attribution.
Low tensor rank offers a practical structural prior for
learning training data effects from limited retraining outcomes.
\section*{Reproducibility Statement}
Proofs and experimental details are provided in the appendix.
Code is available at
\url{https://github.com/farzana0/Tensor_datamodels}.
\section*{Acknowledgment}
Farzaneh Heidari thanks Fazl Barez and Gonçalo Paulo for their helpful discussions, thoughtful feedback, and suggestions throughout this work. 
\section*{AI Use Statement}
Generative AI tools assisted with manuscript editing, verification of the theoretical and experimental results, figure preparation and  coding. The authors take responsibility for the final manuscript,
including its claims, results, and supporting artifacts.
\bibliography{iclr2027_conference}
\bibliographystyle{iclr2027_conference}
\newpage

\appendix

\section{TNDM: Definitions, Properties and Theorems}

\label{app:set-function-foundations}

\subsection{Set functions and tensors}

Let \(N \coloneqq \{1,\ldots,d\}\), where \(d\geq 1\). We are interested in
set functions
\[
    v\colon 2^N \to \mathbb{R}.
\]
We denote by \(\mathcal{F}\) the set of all such functions.

\begin{proposition}
The set \(\mathcal{F}\), endowed with pointwise addition and scalar
multiplication, is a vector space over \(\mathbb{R}\).
\end{proposition}

\begin{proof}
This follows immediately from the vector space structure of
\(\mathbb{R}\), since the vector space operations are defined pointwise.
\end{proof}

Any set function \(v\in\mathcal{F}\) can be encoded by its
$d$-th order \emph{value tensor}
$ \mathcal{T}_v
    \in
    \mathbb{R}^{2\times\cdots\times 2}
$,
whose entries are defined by
\begin{equation}
    \label{eq:value-tensor}
    (\mathcal{T}_v)_{i_1\cdots i_d}
    \coloneqq
    v\bigl(S(i_1,\ldots,i_d)\bigr),
    \qquad
    S(i_1,\ldots,i_d)
    \coloneqq
    \{j\in N:i_j=1\},
\end{equation}
for \(i_1,\ldots,i_d\in\{1,2\}\).

\begin{theorem}
    \label{thm:set-functions-tensors}
The vector space \(\mathcal{F}\) is isomorphic to the space of $d$-th order tensors
$ \mathcal{T}_v
    \in
    \mathbb{R}^{2\times\cdots\times 2}
$
through the map
\[
    \Phi\colon v\longmapsto \mathcal{T}_v
\]
defined in \eqref{eq:value-tensor}.
\end{theorem}

\begin{proof}
The map \(\Phi\) is linear because
\[
    \mathcal{T}_{\alpha v+\beta w}
    =
    \alpha\mathcal{T}_v+\beta\mathcal{T}_w
\]
for every \(v,w\in\mathcal{F}\) and \(\alpha,\beta\in\mathbb{R}\).

It is injective because \(\mathcal{T}_v\) contains the value of \(v(S)\)
for every \(S\subseteq N\). It is surjective because, given any tensor
\(\mathcal{T}\in\mathbb{R}^{2\times\cdots\times2}\), one may define
\[
    v(S)
    \coloneqq
    \mathcal{T}_{i_1(S)\cdots i_d(S)},
    \qquad
    i_j(S)
    \coloneqq
    \begin{cases}
        1, & j\in S,\\
        2, & j\notin S.
    \end{cases}
\]
The resulting set function satisfies \(\mathcal{T}_v=\mathcal{T}\).
Hence \(\Phi\) is a vector space isomorphism.
\end{proof}

\subsection{Set functions, datamodels, linearity, and multilinearity}

A \emph{datamodel}, as introduced by \citet{ilyas2022datamodels}, predicts a
quantity associated with a model trained on a subset of a fixed
collection of datasets. Let \(\mathcal{D}_1,\ldots,\mathcal{D}_d\) be a collection of datasets,
indexed by \(N=\{1,\ldots,d\}\). For every \(S\subseteq N\), let
\[
    f_S
    \coloneqq
    \operatorname{Train}\!\left(\bigcup_{i\in S}\mathcal{D}_i\right)
\]
denote the model trained on the datasets indexed by \(S\). A scalar
quantity associated with \(f_S\), such as its expected output on a
validation distribution,
\[
    v(S)
    \coloneqq
    \mathbb{E}_{x\sim\mathcal{D}_{\mathrm{val}}}
    \bigl[f_S(x)\bigr],
\]
defines a set function \(v\colon 2^N\to\mathbb{R}\). In this sense,
datamodels can be viewed as set functions describing how a model's
behavior depends on the data used to train it. Linear datamodels, as introduced in  \citet{ilyas2022datamodels},
motivate the following definition.

\begin{definition}
A set function \(v\colon 2^N\to\mathbb{R}\) is \emph{linear} if there
exist a parameter vector \(\boldsymbol{\theta}\in\mathbb{R}^d\) and a
bias \(\theta_{\emptyset}\in\mathbb{R}\) such that
\[
    v(S)
    =
    \theta_{\emptyset}
    +
    \sum_{j\in S}\theta_j
\]
for every \(S\subseteq N\). We denote the set of linear set functions
by \(\mathcal{F}_{\mathrm{lin}}\).
\end{definition}

Identifying \(S\subseteq N\) with its indicator vector
\(\mathbf{z}\in\{0,1\}^d\), where
$
    z_j=\mathbf{1}\{j\in S\},
$
a linear set function can equivalently be written as
$
    v(\mathbf{z})
    =
    \theta_{\emptyset}
    +
    \langle\boldsymbol{\theta},\mathbf{z}\rangle.
$
In what follows, we use \(v(S)\) and \(v(\mathbf{z})\) interchangeably.

The class \(\mathcal{F}_{\mathrm{lin}}\) is a vector subspace of
\(\mathcal{F}\). Moreover,
$
    \mathcal{F}_{\mathrm{lin}}\subsetneq\mathcal{F}
$
whenever \(d\geq2\). For example, the set function
$
    v(\mathbf{z})=z_1z_2
$
is not linear.

In the same spirit, we introduce multilinear set functions.

\begin{definition}
A set function \(v\colon 2^N\to\mathbb{R}\) is \emph{multilinear} if
there exists a $d$-th order tensor
\[
    \mathcal{C}
    \in
    \mathbb{R}^{2\times\cdots\times2}
\]
such that
\begin{equation}
    \label{eq:multilinear-representation}
    v(\mathbf{z})
    =
    \mathcal{C}
    \times_1
    \begin{bmatrix}z_1\\1\end{bmatrix}
    \times_2
    \begin{bmatrix}z_2\\1\end{bmatrix}
    \cdots
    \times_d
    \begin{bmatrix}z_d\\1\end{bmatrix}
\end{equation}
for every \(\mathbf{z}\in\{0,1\}^d\).
We denote the set of multilinear\footnote{More precisely, the extension of
\eqref{eq:multilinear-representation} from \(\{0,1\}^d\) to
\(\mathbb{R}^d\) is multiaffine: it is affine in each variable when the
other variables are fixed. By abuse of language, we use the term
``multilinear set function.''} set functions by
\(\mathcal{F}_{\mathrm{ml}}\).
\end{definition}

The tensor \(\mathcal{C}\) in
\eqref{eq:multilinear-representation} is the
\emph{coefficient tensor} of \(v\). It should be distinguished from
the value tensor \(\mathcal{T}_v\) in \eqref{eq:value-tensor}. The two
tensors are nevertheless related by invertible changes of basis.

Indeed, define
\[
    A
    \coloneqq
    \begin{bmatrix}
        1&0\\
        -1&1
    \end{bmatrix}.
\]
For every \(z_j\in\{0,1\}\),
\[
    \begin{bmatrix}
        z_j\\
        1-z_j
    \end{bmatrix}
    =
    A
    \begin{bmatrix}
        z_j\\
        1
    \end{bmatrix}.
\]
The definition of the value tensor gives
\begin{equation}
    \label{eq:value-tensor-evaluation}
    v(\mathbf{z})
    =
    \mathcal{T}_v
    \times_1
    \begin{bmatrix}z_1\\1-z_1\end{bmatrix}
    \cdots
    \times_d
    \begin{bmatrix}z_d\\1-z_d\end{bmatrix}.
\end{equation}
Consequently, the coefficient tensor is
\begin{equation}
    \label{eq:value-coefficient-change-basis}
    \mathcal{C}_v
    =
    \mathcal{T}_v
    \times_1 A^\top
    \times_2 A^\top
    \cdots
    \times_d A^\top.
\end{equation}

\begin{theorem}
    \label{thm:all-set-functions-multilinear}
Every set function is multilinear: 
$
    \mathcal{F}=\mathcal{F}_{\mathrm{ml}}.
$
\end{theorem}

\begin{proof}
Let \(v\in\mathcal{F}\), and let \(\mathcal{T}_v\) be its value tensor.
Equation~\eqref{eq:value-tensor-evaluation} expresses \(v\) using
\(\mathcal{T}_v\). Applying the change of basis
\eqref{eq:value-coefficient-change-basis} yields
\eqref{eq:multilinear-representation} with
\(\mathcal{C}=\mathcal{C}_v\). Thus
\(\mathcal{F}\subseteq\mathcal{F}_{\mathrm{ml}}\). The reverse
inclusion follows directly from the definition of
\(\mathcal{F}_{\mathrm{ml}}\).
\end{proof}

\subsection{Rank and interaction degree of set functions}

Through the isomorphism of
Thm.~\ref{thm:set-functions-tensors}, any notion of tensor rank can
be extended to set functions. We focus here on CP rank.

\begin{definition}
The \emph{CP rank} of a set function \(v\colon2^N\to\mathbb{R}\) is
the CP rank of its value tensor:
\[
    \operatorname{rank}_{\mathrm{CP}}(v)
    \coloneqq
    \operatorname{rank}_{\mathrm{CP}}(\mathcal{T}_v).
\]
For \(R\in\mathbb{N}\), we define
\[
    \mathcal{F}_{\mathrm{CP},R}
    \coloneqq
    \left\{
        v\in\mathcal{F}:
        \operatorname{rank}_{\mathrm{CP}}(v)\leq R
    \right\}.
\]
\end{definition}

This definition is unchanged if the coefficient tensor
\(\mathcal{C}_v\) is used in place of the value tensor. Indeed,
\(A\) in \eqref{eq:value-coefficient-change-basis} is invertible, and
CP rank is invariant under invertible transformations applied
separately to its modes. Therefore,
\begin{equation}
    \label{eq:rank-basis-invariance}
    \operatorname{rank}_{\mathrm{CP}}(v) = \operatorname{rank}_{\mathrm{CP}}(\mathcal{T}_v)
    =
    \operatorname{rank}_{\mathrm{CP}}(\mathcal{C}_v).
\end{equation}

For every fixed \(d\), all tensors in
\(\mathbb{R}^{2\times\cdots\times2}\) have finite CP rank. For example,
the elementary decomposition gives
$
    \operatorname{rank}_{\mathrm{CP}}(\mathcal{T})
    \leq 2^{d-1}.
$
It follows that
$
    \mathcal{F}
    =
    \mathcal{F}_{\mathrm{CP},\,2^{d-1}}.
$

A general set function requires \(2^d\) values to be specified. By
contrast, if \(v\) has CP rank at most \(R\), its value tensor admits a
decomposition
\[
    \mathcal{T}_v
    =
    \sum_{r=1}^R
    \mathbf{a}^{(1)}_r\otimes\cdots\otimes
    \mathbf{a}^{(d)}_r,
    \qquad
    \mathbf{a}^{(j)}_r\in\mathbb{R}^2.
\]
This representation uses \(2dR\) raw parameters, before accounting
for the usual scaling indeterminacies of a CP decomposition. CP rank
therefore induces a notion of model complexity: increasing \(R\)
produces the nested family
\[
    \mathcal{F}_{\mathrm{CP},1}
    \subseteq
    \mathcal{F}_{\mathrm{CP},2}
    \subseteq
    \cdots
    \subseteq
    \mathcal{F}.
\]

A different and equally important notion of complexity is the
interaction degree. For every \(A\subseteq N\), define
\[
    m_A(\mathbf{z})
    \coloneqq
    \prod_{j\in A}z_j,
\]
with the convention \(m_{\emptyset}(\mathbf{z})=1\).

\begin{proposition}
    \label{prop:monomial-basis}
The family
$
    \{m_A:A\subseteq N\}
$
is a basis of \(\mathcal{F}\).
\end{proposition}

\begin{proof}
Every set function admits the expansion
\begin{equation}
    \label{eq:mobius-expansion}
    v(\mathbf{z})
    =
    \sum_{A\subseteq N}c_A m_A(\mathbf{z}),
\end{equation}
where the coefficients are given by Möbius inversion:
\[
    c_A
    =
    \sum_{B\subseteq A}
    (-1)^{|A|-|B|}v(B).
\]
Hence the family spans \(\mathcal{F}\), and the expansion is unique.
Equivalently, since it contains \(2^d=\dim(\mathcal{F})\) elements, it
is a basis.
\end{proof}

The entries of the coefficient tensor \(\mathcal{C}_v\) are precisely
the coefficients in \eqref{eq:mobius-expansion}. More explicitly,
\[
    (\mathcal{C}_v)_{i_1\cdots i_d}
    =
    c_A,
    \qquad
    A=\{j\in N:i_j=1\},
\]
where \(i_j=2\) selects the constant component in the \(j\)-th mode.

\begin{definition}
Let \(v\in\mathcal{F}\) be nonzero, with the unique expansion
$
    v(\mathbf{z})
    =
    \sum_{A\subseteq N}c_A m_A(\mathbf{z}).
$
The \emph{interaction degree} of \(v\) is
\[
    \operatorname{deg}(v)
    \coloneqq
    \max\bigl\{|A|:c_A\neq0\bigr\}.
\]
For the zero function, we adopt the convention
$
    \operatorname{deg}(0)\coloneqq 0.
$

We denote the set of set functions of interaction degree at most \(k\) by
\[
    \mathcal{F}_{\deg,k}
    \coloneqq
    \left\{
        v\in\mathcal{F}:
        \operatorname{deg}(v)\leq k
    \right\}.
\]
\end{definition}

\begin{proposition}
    \label{prop:linear-order-one}
The following statements hold. 
\begin{itemize}
    \item A set function \(v\) is linear if and only if
$
    \operatorname{deg}(v)\leq 1.
$
In particular, every linear set function has interaction degree at
most one.
\item The interaction degree of an arbitrary set function $v:2^{[d]} \to \mathbb{R}$ is upper bounded by $d$:  $\operatorname{deg}(v)\leq d$.
\end{itemize}

\end{proposition}

\begin{proof}
If \(v\in\mathcal{F}_{\mathrm{lin}}\), then
\[
    v(\mathbf{z})
    =
    \theta_{\emptyset}m_{\emptyset}(\mathbf{z})
    +
    \sum_{j=1}^d\theta_jm_{\{j\}}(\mathbf{z}).
\]
Only monomials indexed by sets of cardinality zero or one occur, so
\( \operatorname{deg}(v)\leq1\).

Conversely, if \( \operatorname{deg}(v)\leq1\), its expansion in the
basis of Prop.~\ref{prop:monomial-basis} has the form
\[
    v(\mathbf{z})
    =
    c_{\emptyset}
    +
    \sum_{j=1}^d c_{\{j\}}z_j.
\]
Thus \(v\in\mathcal{F}_{\mathrm{lin}}\).

The second statement follows directly from
\eqref{eq:mobius-expansion}: every monomial is indexed by a subset
\(A\subseteq N\), and therefore satisfies \(|A|\leq d\).
\end{proof}

Similarly, increasing the degree bound produces the nested family
\[
\mathcal F_{\deg,0}
\subseteq
\mathcal F_{\deg,1}
=
\mathcal F_{\mathrm{lin}}
\subseteq
\mathcal F_{\deg,2}
\subseteq\cdots\subseteq
\mathcal F_{\deg,d}
=
\mathcal F.
\]
Increasing the CP-rank bound likewise gives
\[
\mathcal F_{\mathrm{CP},1}
\subseteq
\mathcal F_{\mathrm{CP},2}
\subseteq\cdots\subseteq
\mathcal F.
\]
These two filtrations describe different notions of complexity and,
in general, neither refines the other, as illustrated in
Fig.~\ref{fig:overview}(b).

\subsection{Interaction degree and CP rank in k-hop reasoning data attribution}

We first establish a general result for set functions defined by possibly
overlapping support sets.

\begin{theorem}
    \label{thm:support-rank-degree_app}
Let \(P_1,\ldots,P_m\) be nonempty subsets of \(N\), and let
\(c_1,\ldots,c_m>0\). Consider the set function
\[
    v(S)
    =
    \sum_{a=1}^m c_a\,\mathbf{1}\{P_a\subseteq S\},
    \qquad S\subseteq N.
\]
Then
\[
    \operatorname{deg}(v)=\max_{1\leq a\leq m}|P_a|
    \qquad\text{and}\qquad
    \operatorname{rank}_{\mathrm{CP}}(v)\leq m.
\]
If the support sets \(P_1,\ldots,P_m\) are pairwise disjoint, then the
CP-rank bound is tight:
\[
    \operatorname{rank}_{\mathrm{CP}}(v)=m.
\]
These conclusions remain valid when some elements of \(N\) do not
belong to any support set.
\end{theorem}

\begin{proof}
Identifying \(S\subseteq N\) with its indicator vector
\(\mathbf{z}\in\{0,1\}^d\), we have
\[
    v(\mathbf{z})
    =
    \sum_{a=1}^m c_a\prod_{j\in P_a}z_j.
\]
After combining terms with identical support sets, all coefficients
remain positive. Hence no highest-degree monomial can cancel, and
\(\operatorname{deg}(v)=\max_a|P_a|\).

To bound the CP rank, let
\(\mathbf{e}=[1,0]^\top\) and \(\mathbf{u}=[1,1]^\top\). For each
\(a\in\{1,\ldots,m\}\), define
\[
    \mathcal{X}_a
    \coloneqq
    \bigotimes_{j=1}^d\mathbf{x}_{a,j},
    \qquad
    \mathbf{x}_{a,j}
    \coloneqq
    \begin{cases}
        \mathbf{e}, & j\in P_a,\\
        \mathbf{u}, & j\notin P_a.
    \end{cases}
\]
The value tensor of \(v\) is
\(\mathcal{T}_v=\sum_{a=1}^m c_a\mathcal{X}_a\). Since each
\(\mathcal{X}_a\) has CP rank one,
\(\operatorname{rank}_{\mathrm{CP}}(v)\leq m\).

Suppose now that the support sets are pairwise disjoint. Apply an
invertible change of basis in every mode sending
\(\mathbf{u}\) to \(\mathbf{e}_0\) and
\(\mathbf{e}\) to \(\mathbf{e}_1\). Choose one representative
\(p_a\in P_a\) for each support set and contract every other mode with
the linear functional \(\varphi\) satisfying
\(\varphi(\mathbf{e}_0)=\varphi(\mathbf{e}_1)=1\). The resulting tensor is
\[
    \sum_{a=1}^m
    c_a\,
    \mathbf{e}_0^{\otimes(a-1)}
    \otimes\mathbf{e}_1
    \otimes\mathbf{e}_0^{\otimes(m-a)}.
\]
Because every \(c_a\) is nonzero, invertible rescaling of the modes
reduces this tensor to the standard \(W_m\)-tensor, defined by
\[
    \mathcal{W}_m
    \coloneqq
    \sum_{a=1}^m
    \mathbf{e}_0^{\otimes(a-1)}
    \otimes \mathbf{e}_1
    \otimes \mathbf{e}_0^{\otimes(m-a)},
\]
whose CP rank is
\(m\)~\cite[Theorem~3(a)]{chen2010tensor}. Since contraction with linear functionals cannot increase CP rank,
\(\operatorname{rank}_{\mathrm{CP}}(v)\geq m\). Combined with the upper
bound, this gives \(\operatorname{rank}_{\mathrm{CP}}(v)=m\).
\end{proof}

We apply this result to a simplified model of data attribution for
multi-hop reasoning. Let \(\mathcal{D}_1,\ldots,\mathcal{D}_d\) be a
collection of datasets and let \(q_1,\ldots,q_m\) be evaluation
queries. For each query \(q_a\), let \(P_a\subseteq N\) index the
datasets containing the facts required to answer \(q_a\).

For \(S\subseteq N\), let \(f_S\) denote the model trained on
\(\bigcup_{j\in S}\mathcal{D}_j\). We make the simplifying assumption
that \(f_S\) answers \(q_a\) correctly if and only if \(P_a\subseteq S\).
The average accuracy of \(f_S\) is therefore the set function
\[
    v(S)
    =
    \frac{1}{m}
    \sum_{a=1}^m\mathbf{1}\{P_a\subseteq S\}.
\]

\begin{corollary}
    \label{cor:k-hop-rank-degree}
The average-accuracy datamodel satisfies
\[
    \operatorname{deg}(v)=\max_{1\leq a\leq m}|P_a|
    \qquad\text{and}\qquad
    \operatorname{rank}_{\mathrm{CP}}(v)\leq m.
\]
If the query support sets are pairwise disjoint, then
\(\operatorname{rank}_{\mathrm{CP}}(v)=m\).

In particular, if query \(q_a\) requires \(k_a\) reasoning facts stored
in \(k_a\) distinct datasets, then
\(\operatorname{deg}(v)=\max_a k_a\). If every query is a \(k\)-hop
reasoning task of this form, then \(\operatorname{deg}(v)=k\),
regardless of whether the query support sets overlap.

Datasets whose indices lie outside \(\bigcup_{a=1}^mP_a\) affect
neither the interaction degree nor the CP rank of \(v\).
\end{corollary}

\begin{proof}
Apply Thm.~\ref{thm:support-rank-degree} with \(c_a=1/m\) for every
\(a\).
\end{proof}

This result provides a direct motivation for Tensor Network Datamodels
(TNDMs). TNDMs fit multilinear set functions while imposing a low-rank
structure on their associated tensors, with CP rank being one possible
choice. Besides making the representation computationally tractable,
this structure provides an inductive bias for generalization: when the
underlying datamodel is approximately low rank, observations from a
limited number of training subsets can inform predictions on unseen
subsets, for essentially the same reason that low-rank structure enables
tensor completion. The strength of TNDMs therefore lies in their ability
to represent set functions that are low rank but may nevertheless have
high interaction degree, as can naturally occur in realistic
data-attribution problems involving \(k\)-hop reasoning over facts
distributed across multiple data sources.

Note that the pairwise-disjointness assumption can be relaxed. The same
CP-rank conclusion holds whenever each support \(P_a\) contains a
private dataset, that is, an element
\[
    p_a\in P_a\setminus\bigcup_{b\neq a}P_b.
\]
Indeed, retaining one private mode per support in the contraction
argument again produces a weighted \(W_m\)-tensor.

Without such an assumption, the CP-rank bound need not be tight. For
example, the set function
\(v(\mathbf{z})=z_1+z_1z_2=z_1(1+z_2)\) has two overlapping support
sets but CP rank one.

\subsection{Coverage with redundant supporting groups}

\begin{proposition}[Tensor-train ranks of a coverage function]
\label{prop:coverage-tt}
Let $G_1,\ldots,G_k$ be nonempty, pairwise-disjoint subsets
of $N$, and define
\[
    c(S)
    =
    \prod_{h=1}^{k}
    \left[1-\prod_{i\in G_h}(1-z_i)\right],
    \qquad z_i=\mathbf 1\{i\in S\}.
\]
Then
\[
    \deg(c)=\sum_{h=1}^{k}|G_h|.
\]
In any ordering in which each $G_h$ forms a contiguous block,
the value tensor of $c$ has TT ranks at most two.
Groups outside $\bigcup_h G_h$ affect neither conclusion.
\end{proposition}

\begin{proof}
Each factor has degree $|G_h|$ and a nonzero coefficient
on the product of all variables in $G_h$. Since the sets
are disjoint, their product has a nonzero coefficient on
$\prod_{i\in\cup_h G_h}z_i$, establishing the degree claim.

Consider a prefix cut in the stated ordering, with selected
groups $S_L$ and $S_R$ on its two sides. A cut between
support blocks separates the function into a product of
a left function and a right function, so its unfolding
has rank one.

A cut within a block splits at most one support set.
Let $u(S_L)$ and $w(S_R)$ indicate that no group has been
selected from the left and right portions of that block,
respectively. The remaining support blocks contribute
factors $a(S_L)$ and $b(S_R)$, giving
\[
    c(S_L\cup S_R)
    =
    a(S_L)b(S_R)
    -
    a(S_L)u(S_L)b(S_R)w(S_R).
\]
This is a sum of two separable terms, so the unfolding
has rank at most two. The minimal TT ranks equal the
ranks of the prefix unfoldings, proving the claim.
Groups outside the support sets contribute constant
modes and do not change the argument.
\end{proof}

\section{TNDM: fitting and Parameterizations}
\label{app:tndm-definitions}

This appendix gives the explicit CP and tensor-train
parameterizations used by the TNDMs of Sec.~\ref{sec:tndm}.
Throughout, we use the coefficient-tensor representation
\[
    \widehat v_\Theta(\mathbf z)
    =
    \mathcal C_\Theta
    \times_1
    \begin{bmatrix}z_1\\1\end{bmatrix}
    \cdots
    \times_d
    \begin{bmatrix}z_d\\1\end{bmatrix}.
\]
By the invertible mode-wise change of basis described in
Appendix~\ref{app:set-function-foundations}, the same tensor network
ranks are obtained if the corresponding value tensor is used instead.

\subsection{CP Parameterization}

A rank-$R$ CP tensor has the form
\[
    \mathcal C_\Theta
    =
    \sum_{r=1}^{R}
    \mathbf a_r^{(1)}
    \otimes\cdots\otimes
    \mathbf a_r^{(d)},
    \qquad
    \mathbf a_r^{(j)}\in\mathbb R^2.
\]
Consequently,
\[
    \widehat v_\Theta(\mathbf z)
    =
    \sum_{r=1}^{R}
    \prod_{j=1}^{d}
    \left(
        a_{r,1}^{(j)} z_j
        +
        a_{r,2}^{(j)}
    \right).
\]
Thus TNDM-CP is exactly the sum-of-products representation of
Definition~\ref{def:cp_rank}. It uses $2dR$ parameters and evaluating
one coalition requires $O(dR)$ operations.
\subsection{Tensor-Train Parameterization}

For an ordering $\pi$ of the training units, a tensor-train represents
the coefficient tensor using cores
\[
    G^{(j)}
    \in
    \mathbb R^{r_{j-1}\times 2\times r_j},
    \qquad
    r_0=r_d=1.
\]
Writing $G^{(j)}[i]$ for the matrix obtained by fixing the physical
index to $i\in\{1,2\}$, the corresponding datamodel is
\[
    \widehat v_\Theta(\mathbf z)
    =
    \bigl(z_{\pi(1)}G^{(1)}[1]+G^{(1)}[2]\bigr)
    \bigl(z_{\pi(2)}G^{(2)}[1]+G^{(2)}[2]\bigr)
    \cdots
    \bigl(z_{\pi(d)}G^{(d)}[1]+G^{(d)}[2]\bigr),
\]
where each factor contracts a coefficient-tensor core with
$[z_{\pi(j)},1]^\top$.

Equivalently, after transforming the cores to the value-tensor basis,
the $j$-th factor is $\widetilde G^{(j)}[2-z_{\pi(j)}]$,
where index $1$ denotes inclusion and index $2$ denotes exclusion.
The two representations are related by the invertible local change
of basis of Appendix~\ref{app:set-function-foundations}.

If all internal ranks satisfy $r_j\le r$, the model uses
$O(dr^2)$ parameters and evaluates a coalition in $O(dr^2)$ time.

Unlike CP rank, TT rank depends on the ordering of the tensor modes.
For a fixed ordering, each TT bond corresponds to a prefix
bipartition of the training units. The minimum exact bond dimension
at that cut equals the rank of the corresponding matrix unfolding;
we state this explicitly in Appendix~\ref{app:proofs-representation}.
\subsection{Fitting from Coalition Observations}

Both parameterizations are fitted using the same observed
coalition--outcome pairs as the other datamodels. Given
\[
    \mathcal O
    =
    \{(\mathbf z^{(s)},y_s)\}_{s=1}^{n},
\]
we optimize the regularized least-squares objective of
Sec.~\ref{sec:tndm},
\[
    \widehat\Theta
    \in
    \arg\min_\Theta
    \frac1n\sum_{s=1}^{n}
    \left(
        \widehat v_\Theta(\mathbf z^{(s)})-y_s
    \right)^2
    +
    \lambda\Omega(\Theta).
\]
Ranks and regularization parameters are selected using held-out
coalition observations.

Figure~\ref{fig:completion} illustrates this tensor-completion
view for a tensor-train: observed coalition outcomes determine a
subset of tensor entries, the TT factors are fitted on those entries,
and the fitted network predicts the remaining coalition values.

\begin{figure}[t]
\centering
\begin{tikzpicture}[
  font=\small,
  chip/.style={draw=black!45,line width=0.3pt,minimum size=2.4mm,inner sep=0pt},
  mbox/.style={rounded corners=0.6mm,draw=black!30,fill=black!4,
               minimum width=8.4mm,minimum height=4.8mm,inner sep=0pt,font=\scriptsize},
  msel/.style={mbox,draw=blue!60,fill=blue!14,line width=0.7pt},
  flow/.style={-{Stealth[length=2.4mm,width=2mm]},line width=0.9pt,draw=black!70},
  sel/.style={draw=blue!65,line width=1.1pt},
  pick/.style={draw=black!40,line width=0.4pt,dotted},
  stage/.style={font=\footnotesize\bfseries},
  sub/.style={font=\scriptsize,text=black!65},
  hd/.style={font=\scriptsize\bfseries,text=black!75},
]
\def\cs{0.26}  %

\def\mkrow#1#2#3{%
  \foreach \b [count=\c from 0] in {#2}{%
    \ifnum\b=1 \fill[blue!45] (\c*\cs,#1) rectangle ++(\cs,-\cs);\fi
    \draw[black!35,line width=0.3pt] (\c*\cs,#1) rectangle ++(\cs,-\cs);}%
  \node[sub,anchor=west,text=black!80] at (5*\cs+0.16,{#1-0.5*\cs}) {#3};}

\node[stage,anchor=south west] at (0,1.30) {(a) what we measure};
\node[hd,anchor=south west] at (0,0.66) {subset $S$};
\node[hd,anchor=south west] at (5*\cs+0.16,0.66) {$v(S)$};
\draw[black!40,line width=0.3pt] (0,0.62) -- (1.95,0.62);
\mkrow{0.46}{1,0,1,1,0}{$0.71$}
\mkrow{0.20}{0,1,1,0,1}{$0.55$}
\mkrow{-0.06}{1,1,0,1,0}{$0.83$}
\mkrow{-0.32}{0,0,1,1,1}{\textcolor{black!35}{?}}
\mkrow{-0.58}{1,1,1,0,1}{\textcolor{black!35}{?}}
\node[sub] at (0.65,-1.00) {$\vdots$};
\node[sub,anchor=north,align=center] at (0.95,-1.22)
  {$2^{d}$ rows in all;\\one retraining each};

\draw[flow] (2.02,-0.06) -- ++(0.38,0);

\begin{scope}[shift={(3.15,0)}]
  \node[stage,anchor=south] at (2.90,1.30) {(b) how a row is predicted};

  \node[sub,anchor=east]  at (0.42,0.80)  {mask};
  \node[hd,anchor=east]   at (0.42,0.10)  {absent};
  \node[hd,anchor=east]   at (0.42,-0.62) {present};
  \node[sub,anchor=east]  at (0.42,-1.06) {unit};

  \foreach \i/\z in {0/1,1/0,2/1,3/1,4/0}{
    \pgfmathtruncatemacro{\u}{\i+1}
    \pgfmathsetmacro{\x}{0.98+\i*1.04}
    \ifnum\z=1 \node[chip,fill=blue!45] (c\i) at (\x,0.80) {};
    \else       \node[chip,fill=white]  (c\i) at (\x,0.80) {};\fi
    \ifnum\z=0
      \node[msel] (s\i) at (\x,0.10)  {$G_{\u}[2]$};
      \node[mbox]        at (\x,-0.62) {$G_{\u}[1]$};
    \else
      \node[mbox]        at (\x,0.10)  {$G_{\u}[2]$};
      \node[msel] (s\i) at (\x,-0.62) {$G_{\u}[1]$};
    \fi
    \draw[pick] (c\i) -- (s\i);
    \node[sub] at (\x,-1.06) {\u};
  }
  \foreach \i in {0,1,2,3}{\pgfmathtruncatemacro{\j}{\i+1}\draw[sel] (s\i) -- (s\j);}

  \node[sub,anchor=north,align=center] at (2.90,-1.50)
    {the mask selects one matrix per unit;\\$\widehat{v}(S)$ is the product along the blue path};
\end{scope}

\draw[flow] (8.85,-0.06) -- ++(0.45,0);

\begin{scope}[shift={(9.45,0)}]
  \node[stage,anchor=south east] at (1.95,1.30) {(c) what we get};
  \node[hd,anchor=south west] at (0,0.66) {subset $S$};
  \node[hd,anchor=south west] at (5*\cs+0.16,0.66) {$\widehat{v}(S)$};
  \draw[black!40,line width=0.3pt] (0,0.62) -- (1.95,0.62);
  \mkrow{0.46}{1,0,1,1,0}{$0.70$}
  \mkrow{0.20}{0,1,1,0,1}{$0.56$}
  \mkrow{-0.06}{1,1,0,1,0}{$0.82$}
  \mkrow{-0.32}{0,0,1,1,1}{\textcolor{blue!60!black}{$0.62$}}
  \mkrow{-0.58}{1,1,1,0,1}{\textcolor{blue!60!black}{$0.94$}}
  \node[sub] at (0.65,-1.00) {$\vdots$};
  \node[sub,anchor=north,align=center] at (0.95,-1.22)
    {every row predicted;\\blue: never retrained};
\end{scope}

\draw[black!25,line width=0.3pt] (0,-2.16) -- (11.40,-2.16);
\node[stage,anchor=north west] at (0,-2.28) {how it is fitted};
\node[anchor=north,align=center,font=\small] at (6.30,-2.26)
  {$\displaystyle
     \min_{\{G_i\}}\ \frac{1}{m}\sum_{n=1}^{m}\bigl(\widehat{v}(\mathbf{z}^{(n)})-y_n\bigr)^{2}
     \;+\;\lambda\,\Omega(\{G_i\})$};
\node[sub,anchor=north,align=center] at (5.90,-3.22)
  {least squares over the matrix entries, on the measured rows only ($\approx 2dr^{2}$ unknowns, not $2^{d}$);\\
   the rank $r$, the network topology, and $\lambda$ are chosen on held-out validation coalitions};

\end{tikzpicture}
\caption{
\textbf{Fitting a TNDM as tensor completion.}
(a) Each observed coalition provides one entry of the coalition tensor.
(b) The fitting objective in Eq.~\ref{eq:tndm-objective} estimates the parameters of a low-rank
tensor network to the observed entries.
(c) The fitted network defines predictions for unobserved coalitions.
The example uses a tensor train; CP and tree networks use different
factorizations of the same coalition tensor. A TNDM learns from a small set of measured training subsets and uses the shared low-rank structure to predict outcomes for subsets that were never retrained.}
\label{fig:completion}
\end{figure}
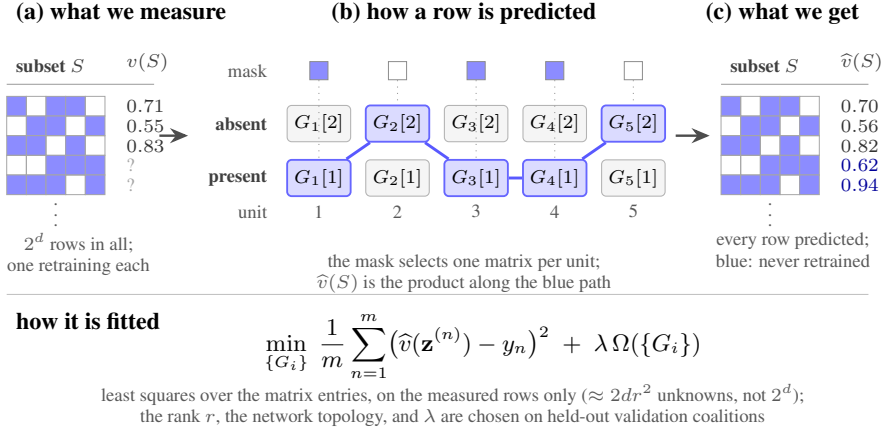

\section{Tensor Network Notation and Constructions}
\label{app:diagrams}

We use the conventions collected by \citet{rakhshan2026tensor}: a node is a
tensor and each leg is one mode, with its dimension in gray at the free end;
joining two legs of equal dimension sums that index; a filled black dot is a
copy tensor; attaching $\mathbf e_0=(0,1)^\top$ or $\mathbf e_1=(1,0)^\top$ to
a physical leg fixes absence or presence of a training unit; and a network with
no free legs is a scalar. Colour is a reading aid only: blue is the learned
coalition tensor or its cores, orange represents attribution weights, green is a local
change of basis. Figure~\ref{fig:primer} summarizes these conventions.

\begin{figure}[ht!]
\centering
\begin{tikzpicture}[x=1cm,y=1cm,font=\small]
\node[tn] (A) at (1.25,0) {$A$};
\draw[bond] (A.west)--++(-0.48,0);  \node[tndim,anchor=south] at (0.55,0.03) {$d_1$};
\draw[bond] (A.north)--++(0,0.46);  \node[tndim,anchor=south] at (1.25,0.44) {$d_2$};
\draw[bond] (A.east)--++(0.48,0);   \node[tndim,anchor=south] at (1.95,0.03) {$d_3$};
\node[tncap,anchor=north,text width=2.4cm,align=center] at (1.25,-0.52)
     {$3$ free legs\\ $=$ order-$3$ tensor $A_{ijk}$};
 
\node[tn] (B) at (3.65,0) {$A$};  \node[tn] (C) at (4.85,0) {$B$};
\draw[bond] (B.east)--(C.west);   \node[tndim,anchor=south] at (4.25,0.06) {$a$};
\draw[bond] (B.west)--++(-0.38,0); \draw[bond] (C.east)--++(0.38,0);
\node[tncap,anchor=north,text width=2.4cm,align=center] at (4.25,-0.52)
     {joined leg is summed\\ $\textstyle\sum_a A_{ia}B_{aj}$};
 
\node[copydot] (cp) at (6.95,0.06) {};
\draw[bond] (cp)--++(-0.44,0.34); \draw[bond] (cp)--++(0.44,0.34);
\draw[bond] (cp)--++(0,-0.44);
\node[tncap,anchor=north,text width=2.4cm,align=center] at (6.95,-0.52)
     {copy tensor\\ $\Delta_{ijk}=\mathbf1\{i{=}j{=}k\}$};
 
\node[tn] (D) at (9.35,0) {$\widehat{\cT}$};
\draw[bond] (D.north)--++(0,0.30);
\node[qv] (e) at (9.35,0.73) {$\mathbf e_1$};
\draw[bond] (D.west)--++(-0.36,0);
\draw[bond] (D.east)--++(0.36,0);
\node[tncap,anchor=north,text width=2.5cm,align=center] at (9.35,-0.52)
     {basis vector on a leg\\ fixes $z_i=1$ (unit retained)};
 
\node[tn,minimum width=12mm] (E) at (11.95,0) {$\widehat{\cT}$};
\foreach \x in {-0.42,-0.14,0.14,0.42}{
   \draw[bond] ({11.95+\x},0.36)--({11.95+\x},0.52);
   \node[qv,minimum size=3.4mm] at ({11.95+\x},0.69) {};}
\node[tncap,anchor=north,text width=2.6cm,align=center] at (11.95,-0.52)
     {no free legs left\\ $\Rightarrow$ the network is a scalar};
 
\foreach \x in {2.55,5.90,8.20,10.70}{
   \draw[draw=tnink!14,line width=.4pt] (\x,-1.32)--(\x,1.00);}
\end{tikzpicture}
\caption{Diagram notation. A node is a tensor and each leg is one mode, with
its dimension in gray at the free end; joining two legs of equal dimension sums
that index; a filled dot is a copy tensor; attaching $e_0=(0,1)^\top$ or
$e_1=(1,0)^\top$ to a physical leg fixes absence or presence of a training unit;
and a network with no free legs evaluates to a scalar. Only connectivity carries algebraic meaning. Tensor-network diagrams provide a compact graphical notation for the tensor operations used throughout the paper.}
\label{fig:primer}
\end{figure}
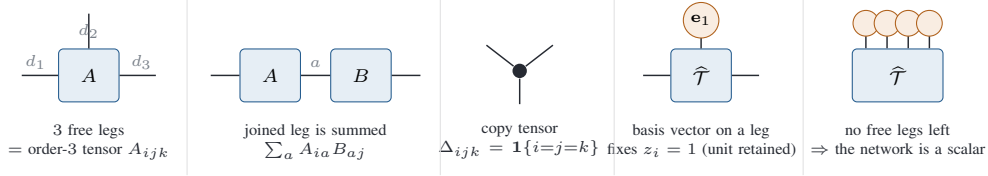

\section{Tensor-Train Rank and Coalition Unfoldings}
\label{app:proofs-representation}

For an ordering $\pi$, let
\[
    L_k^\pi=\{\pi(1),\ldots,\pi(k)\},
    \qquad
    R_k^\pi=N\setminus L_k^\pi,
\]
and define the coalition unfolding
\[
    \mathbf M_k^\pi[A,C]
    =
    v(A\cup C),
    \qquad
    A\subseteq L_k^\pi,\;
    C\subseteq R_k^\pi .
\]

\begin{proposition}[TT rank and coalition unfoldings]
\label{prop:state-count}
For a fixed ordering $\pi$, the minimum exact TT bond dimension
at cut $k$ equals
\[
    \operatorname{rank}(\mathbf M_k^\pi).
\]
Equivalently, $\operatorname{rank}(\mathbf M_k^\pi)$ is the smallest
integer $r$ for which maps
\[
    \sigma:2^{L_k^\pi}\to\mathbb R^r,
    \qquad
    \rho:2^{R_k^\pi}\to\mathbb R^r
\]
satisfy
\[
    v(A\cup C)
    =
    \langle \sigma(A),\rho(C)\rangle
\]
for every $A\subseteq L_k^\pi$ and $C\subseteq R_k^\pi$.
\end{proposition}

\begin{proof}
A representation
\[
    v(A\cup C)=\langle\sigma(A),\rho(C)\rangle
\]
corresponds exactly to a matrix factorization
\[
    \mathbf M_k^\pi=\Sigma P^\top
\]
with inner dimension $r$, where the rows of $\Sigma$ are
$\sigma(A)^\top$ and the rows of $P$ are $\rho(C)^\top$.
The smallest admissible inner dimension therefore equals
$\operatorname{rank}(\mathbf M_k^\pi)$.
\end{proof}
\section{General Complementary and Redundant Support Structure}
\label{app:structural-theory}

The support-set model of Sec.~\ref{sec:examples} assumes one
training group for each required fact. We now allow several training
groups to provide the same required fact.

For query $a\in\{1,\ldots,m\}$, let
\[
    G_{a,1},\ldots,G_{a,k_a}\subseteq N
\]
be nonempty, pairwise-disjoint sets. Query $a$ is answered correctly
when the selected coalition contains at least one training group from
every $G_{a,h}$. Define
\begin{equation}
    u_a(S)
    =
    \prod_{h=1}^{k_a}
    \left[
        1-\prod_{j\in G_{a,h}}(1-z_j)
    \right],
    \qquad
    z_j=\mathbf 1\{j\in S\}.
    \label{eq:general-query-coverage}
\end{equation}
Sets belonging to different queries may overlap.

Let
\[
    D_a
    =
    \sum_{h=1}^{k_a}|G_{a,h}|,
    \qquad
    \nu_a
    =
    \bigl|\{h:|G_{a,h}|>1\}\bigr|.
\]
Thus, $D_a$ is the number of training groups involved in query $a$,
and $\nu_a$ is the number of required facts with redundant supporting
groups.

For an ordering $\pi$ and prefix cut
\[
    L_\ell^\pi=\{\pi(1),\ldots,\pi(\ell)\},
    \qquad
    R_\ell^\pi=N\setminus L_\ell^\pi,
\]
define
\[
    s_{a,\ell}^{\pi}
    =
    \bigl|
    \{
        h:
        G_{a,h}\cap L_\ell^\pi\neq\emptyset,
        \;
        G_{a,h}\cap R_\ell^\pi\neq\emptyset
    \}
    \bigr|.
\]
The quantity $s_{a,\ell}^{\pi}$ counts the sets $G_{a,h}$ split by
the cut.

\begin{theorem}[Complementary and redundant support structure]
\label{thm:general-support-structure}
For every query $a$,
\[
    \deg(u_a)=D_a,
    \qquad
    \operatorname{rank}_{\mathrm{CP}}(u_a)=2^{\nu_a}.
\]
For every ordering $\pi$ and prefix cut $\ell$,
\[
    \operatorname{rank}
    \left(
        \mathbf M_\ell^\pi(u_a)
    \right)
    =
    2^{s_{a,\ell}^{\pi}}.
\]

For positive weights $c_1,\ldots,c_m$, define
\[
    v(S)
    =
    \sum_{a=1}^{m}c_a u_a(S).
\]
Then
\[
    \deg(v)\le\max_a D_a,
\]
\[
    \operatorname{rank}_{\mathrm{CP}}(v)
    \le
    \sum_{a=1}^{m}2^{\nu_a},
\]
and
\[
    \operatorname{rank}
    \left(
        \mathbf M_\ell^\pi(v)
    \right)
    \le
    \sum_{a=1}^{m}2^{s_{a,\ell}^{\pi}}.
\]
The degree bound is an equality whenever the degree-$\max_aD_a$
terms do not cancel.
\end{theorem}

\begin{proof}
For a fixed query $a$, each factor
\[
    1-\prod_{j\in G_{a,h}}(1-z_j)
\]
has degree $|G_{a,h}|$ and contains the nonzero highest-degree term
\[
    (-1)^{|G_{a,h}|-1}
    \prod_{j\in G_{a,h}}z_j.
\]
The sets $G_{a,h}$ are disjoint, so their product contains the
nonzero term
\[
    (-1)^{D_a-k_a}
    \prod_{j\in \cup_h G_{a,h}}z_j.
\]
Hence $\deg(u_a)=D_a$.

A singleton set $G_{a,h}$ contributes the rank-one factor $z_j$.
For $|G_{a,h}|>1$,
\[
    1-\prod_{j\in G_{a,h}}(1-z_j)
\]
is a sum of two rank-one terms. Expanding the $\nu_a$ nonsingleton
factors therefore gives
\[
    \operatorname{rank}_{\mathrm{CP}}(u_a)\le 2^{\nu_a}.
\]

For the reverse inequality, choose a matrix unfolding that places at
least one variable from each nonsingleton $G_{a,h}$ on each side.
The unfolding of every split factor has matrix rank two, while
singleton factors and unused variables have rank one. After permuting
rows and columns, the full unfolding is a Kronecker product of these
factor unfoldings, so its matrix rank is
\[
    2^{\nu_a}.
\]
Matrix-unfolding rank cannot exceed CP rank, giving
\[
    \operatorname{rank}_{\mathrm{CP}}(u_a)=2^{\nu_a}.
\]

For a fixed TT cut, every set $G_{a,h}$ lying entirely on one side
contributes unfolding rank one. Every set $G_{a,h}$ split by the cut
contributes unfolding rank two because
\[
    1-
    \left[
        \prod_{j\in G_{a,h}\cap L_\ell^\pi}(1-z_j)
    \right]
    \left[
        \prod_{j\in G_{a,h}\cap R_\ell^\pi}(1-z_j)
    \right]
\]
is the sum of two separable terms whose left and right factors are
linearly independent. The full unfolding is therefore a Kronecker
product with rank
\[
    2^{s_{a,\ell}^{\pi}}.
\]

Finally, degree, CP rank, and matrix rank are respectively bounded
under addition by the maximum degree, the sum of CP ranks, and the
sum of matrix ranks. Applying these bounds to
$v=\sum_a c_a u_a$ gives the three aggregate bounds.
\end{proof}

\begin{corollary}[Pure complementarity]
\label{cor:general-pure-complementarity}
If every $G_{a,h}$ contains one training group, then
\[
    \operatorname{rank}_{\mathrm{CP}}(u_a)=1
\]
for every query $a$, and
\[
    \operatorname{rank}_{\mathrm{CP}}(v)\le m.
\]
Writing
\[
    P_a=\bigcup_{h=1}^{k_a}G_{a,h},
\]
gives
\[
    u_a(S)=\mathbf 1\{P_a\subseteq S\}.
\]
Thus the support-set model of Sec.~\ref{sec:examples} is recovered.
\end{corollary}

\begin{corollary}[Single-query redundant coverage]
\label{cor:general-redundant-coverage}
For one query with nonempty, pairwise-disjoint sets
$G_1,\ldots,G_k$,
\[
    u(S)
    =
    \prod_{h=1}^{k}
    \left[
        1-\prod_{j\in G_h}(1-z_j)
    \right]
\]
has
\[
    \deg(u)=\sum_{h=1}^{k}|G_h|.
\]
If every $G_h$ is contiguous in the tensor ordering, every prefix cut
splits at most one $G_h$, so every TT rank is at most two.
\end{corollary}
\section{Attribution Functionals and Contraction Proofs}
\label{app:queries}

Let $\cT=\mathcal T_v$ denote the value tensor, using physical index
$2-z_j$ for inclusion bit $z_j$. Define
\[
    \mathbf e_1=(1,0)^\top,
    \qquad
    \mathbf e_0=(0,1)^\top,
    \qquad
    \mathbf b(t)=(t,1-t)^\top,
\]
and
\[
    \boldsymbol\delta
    =
    \mathbf e_1-\mathbf e_0
    =
    (1,-1)^\top.
\]
For a weight tensor $\cW$, define the linear functional
\[
    L_{\cW}(\widehat{\cT})
    =
    \langle \widehat{\cT},\cW\rangle_F.
\]

\subsection{Linear Functionals of a Fitted TNDM}

\begin{theorem}[Linear-functional contraction]
\label{thm:query}
Every linear functional on
$\mathbb R^{2\times\cdots\times2}$
has a unique weight tensor $\cW$.
If $\widehat{\cT}$ and $\cW$ have TT ranks bounded by
$r$ and $\rho$, respectively, in the same ordering, then
\[
    \langle\widehat{\cT},\cW\rangle_F
\]
can be evaluated in
$\cO(dr^2\rho^2)$ operations.
For a product weight tensor, the cost reduces to
$\cO(dr^2)$.
Moreover,
\[
    \left|
    \langle\cT-\widehat{\cT},\cW\rangle_F
    \right|
    \le
    \|\cT-\widehat{\cT}\|_F
    \|\cW\|_F .
\]
\end{theorem}

\begin{proof}
Let $\cE_{\mathbf z}$ denote the standard-basis tensor associated
with coalition $\mathbf z$. For a linear functional $L$, define
\[
    \cW[\mathbf z]
    =
    L(\cE_{\mathbf z}).
\]
Linearity gives
\[
    L(\cX)=\langle\cX,\cW\rangle_F
\]
for every tensor $\cX$, while evaluation on the standard basis
establishes uniqueness.

For TT representations of ranks bounded by $r$ and $\rho$,
sequential contraction of the two networks costs
$\cO(dr^2\rho^2)$.
A product weight tensor has TT rank one, giving
$\cO(dr^2)$.
The final bound follows from the Cauchy--Schwarz inequality.
\end{proof}

\subsection{Marginal-Contribution Curves and Semivalues}

For $i\in N$, define the product weight tensor $\cW_i(t)$ by placing
$\boldsymbol\delta$ on mode $i$ and $\mathbf b(t)$ on every other
mode. The resulting marginal-contribution curve is
\begin{equation}
    \widehat g_i(t)
    =
    \langle\widehat{\cT},\cW_i(t)\rangle_F
    =
    \sum_{S\subseteq N\setminus\{i\}}
    t^{|S|}
    (1-t)^{d-1-|S|}
    \left[
        \widehat v(S\cup\{i\})
        -
        \widehat v(S)
    \right].
    \label{eq:marginal-curve}
\end{equation}
The same definition without hats applies to the measured set
function $v$.

\begin{proposition}[Marginal-contribution curve]
\label{prop:marginal}
For every $i\in N$,
\[
    \widehat{\mathrm{LOO}}_i
    =
    \widehat g_i(1),
    \qquad
    \widehat\phi_i^{\,\mathrm{Bz}}
    =
    \widehat g_i\!\left(\tfrac12\right),
\]
and
\[
    \widehat\phi_i^{\,\mathrm{Sh}}
    =
    \int_0^1
    \widehat g_i(t)\,dt .
\]
Each evaluation of $\widehat g_i(t)$ costs
$\cO(dr^2)$ for a rank-$r$ TT.
The polynomial $\widehat g_i$ has degree at most $d-1$, so
Gauss--Legendre quadrature with
$\lceil d/2\rceil$ nodes evaluates the Shapley integral exactly
in exact arithmetic, at total cost $\cO(d^2r^2)$.
\end{proposition}

\begin{proof}
Expanding the product weight tensor $\cW_i(t)$ gives
Eq.~\eqref{eq:marginal-curve}.
Setting $t=1$ retains only
$S=N\setminus\{i\}$, giving leave-one-out.
Setting $t=\tfrac12$ assigns equal probability
$2^{-(d-1)}$ to every context, giving Banzhaf.

For Shapley values,
\[
    \int_0^1
    t^{|S|}
    (1-t)^{d-1-|S|}
    dt
    =
    \frac{|S|!\,(d-|S|-1)!}{d!},
\]
which gives the Shapley weights.
The computational costs follow from
Thm.~\ref{thm:query}.
\end{proof}

Figure~\ref{fig:marginal-schematic} summarizes the three readouts:
leave-one-out evaluates the marginal curve at $t=1$, Banzhaf at
$t=\tfrac12$, and Shapley integrates the curve over $[0,1]$.

\begin{figure}[t]
\centering
\begin{tikzpicture}[x=1cm,y=1cm,font=\small]
\foreach \k in {1,...,7}{
   \node[core,minimum size=5.4mm,font=\scriptsize] (g\k) at ({0.55+0.68*(\k-1)},0) {$\cG$};}
\foreach \k in {1,...,6}{\draw[bond] (g\k)--(g\the\numexpr\k+1\relax);}
\node[tndim,anchor=north] at (0.89,-0.14) {$r$};
\foreach \k in {1,2,3,5,6,7}{
   \node[qv] (b\k) at ($(g\k)+(0,0.70)$) {$\mathbf b(t)$};
   \draw[bond] (g\k)--(b\k);}
\node[qv,fill=tnorange!38,line width=.8pt] (di) at ($(g4)+(0,0.70)$) {$\boldsymbol\delta$};
\draw[bond] (g4)--(di);
\node[tndim,anchor=south,text=tnorange!65!black] at ($(di)+(0,0.26)$) {site $i$};
\draw[decorate,decoration={brace,amplitude=3.5pt,mirror,raise=2pt},
      draw=tnink!45,line width=.5pt] (0.34,-0.34)--(4.70,-0.34);
\node[tncap,anchor=north,text width=4.9cm,align=center] at (2.52,-0.62)
     {one difference leg at site $i$; every other context averaged at retention rate $t$};
\node[tncap,anchor=south] at (2.52,1.30)
     {$\widehat g_i(t)=\inner{\widehat{\cT}}{\delta_i\otimes\bigotimes_{j\neq i}b(t)_j}$};
 
\begin{scope}[xshift=6.90cm,yshift=-0.46cm]
  \def\W{3.60} \def\Hh{1.30}
  \fill[tnorange!12]
     plot[domain=0:1,samples=60] ({\W*\x},{0.30+1.05*\x*\x*(1.15-0.35*\x)})
     -- (\W,0) -- (0,0) -- cycle;
  \draw[bond,draw=tnink!55,->] (0,0) -- (\W+0.34,0) node[tndim,anchor=west,text=tnink] {$t$};
  \draw[bond,draw=tnink!55,->] (0,0) -- (0,\Hh+0.24) node[tndim,anchor=south,text=tnink] {$\widehat g_i(t)$};
  \draw[draw=tnorange!85!black,line width=.9pt]
     plot[domain=0:1,samples=60] ({\W*\x},{0.30+1.05*\x*\x*(1.15-0.35*\x)});
  \draw[draw=tnink!40,dashed,line width=.4pt] (\W/2,0) -- (\W/2,{0.30+1.05*0.25*(1.15-0.175)});
  \draw[draw=tnink!40,dashed,line width=.4pt] (\W,0) -- (\W,{0.30+1.05*0.80});
  \node[tndim,anchor=north,text=tnink] at (\W/2,-0.05) {$t=\tfrac12$};
  \node[tndim,anchor=north,text=tnink] at (\W,-0.05) {$t=1$};
  \node[tndim,anchor=north east,text=tnink] at (-0.02,-0.02) {$0$};
  \fill[tnorange!85!black] (\W/2,{0.30+1.05*0.25*(1.15-0.175)}) circle (1.5pt);
  \fill[tnorange!85!black] (\W,{0.30+1.05*0.80}) circle (1.5pt);
  \node[tncap,anchor=south east,text=tnink] at ({\W/2-0.08},{0.30+1.05*0.25*(1.15-0.175)+0.07}) {Banzhaf};
  \node[tncap,anchor=south east,text=tnink] at ({\W-0.06},{0.30+1.05*0.80+0.10}) {LOO};
  \node[tncap,anchor=south,text=tnorange!55!black] at (\W*0.52,0.010) {Shapley $=\int_0^1 \widehat g_i(t)\,dt$};
\end{scope}
\end{tikzpicture}
\caption{LOO, Banzhaf, and Shapley are three readouts of the same marginal
curve. Left: for a fixed $t$, the rank-one weight tensor $\cW_i(t)$ places
$\boldsymbol\delta=\mathbf e_1-\mathbf e_0$ on site $i$ and
$\mathbf b(t)=(t,1-t)^\top$ on every other site. Its contraction with the
fitted TNDM costs $\cO(dr^2)$. Right: LOO evaluates the curve at $t=1$,
Banzhaf at $t=\tfrac12$, and Shapley integrates the curve over $[0,1]$. Once a TNDM is fitted, several familiar attribution scores can be obtained as different summaries of the same learned subset function.}
\label{fig:marginal-schematic}
\end{figure}
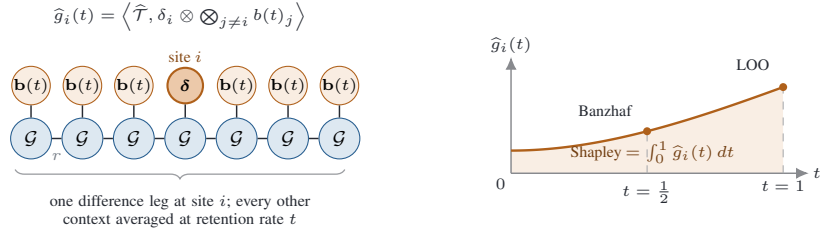

\begin{proposition}[Grouped marginal-curve coefficients]
\label{prop:marginal-coefficients}
Let $m(A)$ denote the M\"obius coefficient associated with
$A\subseteq N$. For every $i\in N$,
\begin{equation}
    g_i(t)
    =
    \sum_{k=1}^{d}
    h_{i,k}\,t^{k-1},
    \qquad
    h_{i,k}
    =
    \sum_{\substack{A\ni i\\|A|=k}}
    m(A).
    \label{eq:grouped-marginal}
\end{equation}
Consequently, $g_i$ is constant if and only if
$h_{i,k}=0$ for every $k\ge2$.

Vanishing grouped coefficients do not require every higher-order
M\"obius coefficient to vanish, because coefficients of equal order
may cancel. In particular, additive set functions satisfy
$m(A)=0$ for every $|A|\ge2$ and therefore have constant marginal
curves.
\end{proposition}

\begin{proof}
The marginal contribution satisfies
\[
    v(S\cup\{i\})-v(S)
    =
    \sum_{\substack{A\ni i\\A\setminus\{i\}\subseteq S}}
    m(A).
\]
Substitution into Eq.~\eqref{eq:marginal-curve} and summation
over all contexts containing $A\setminus\{i\}$ gives
\[
    g_i(t)
    =
    \sum_{A\ni i}
    m(A)t^{|A|-1}.
\]
Grouping terms by $|A|$ yields
Eq.~\eqref{eq:grouped-marginal}.
A polynomial is constant exactly when every coefficient of a
positive power vanishes.
\end{proof}

\begin{proposition}[Semivalue weight tensor]
\label{prop:semivalue}
Let
$\mathbf p=(p_0,\ldots,p_{d-1})$
be a distribution over context sizes and define
\[
    \cW_i^{(\mathbf p)}[\mathbf z]
    =
    (2z_i-1)
    \frac{
        p_{\|\mathbf z_{-i}\|_1}
    }{
        \binom{d-1}{\|\mathbf z_{-i}\|_1}
    }.
\]
Then
\[
    \widehat\phi_i^{(\mathbf p)}
    =
    \left\langle
        \widehat{\cT},
        \cW_i^{(\mathbf p)}
    \right\rangle_F .
\]
The weight tensor admits a counting-network representation with TT
ranks at most $d$. Exploiting the sparse counter transitions,
contraction with a rank-$r$ TNDM costs
$\cO(d^2r^2)$.

Special cases are
\[
    p_s=\frac1d
    \quad\text{for Shapley},
\]
\[
    p_s
    =
    \binom{d-1}{s}2^{-(d-1)}
    \quad\text{for Banzhaf},
\]
and
\[
    p_{d-1}=1
    \quad\text{for leave-one-out}.
\]
\end{proposition}

\begin{proof}
Expanding the contraction and grouping coalition indicators by
$\|\mathbf z_{-i}\|_1$ gives the semivalue definition.
The factor $2z_i-1$ implements the difference between retaining
and removing training unit $i$.

A counting network propagates the partial Hamming weight from left
to right. At site $j$, the sparse transition either preserves or
increments the count. The contraction therefore costs
$\cO(jr^2)$ at site $j$, and summation over all $d$ sites gives
$\cO(d^2r^2)$.
\end{proof}

Figure~\ref{fig:size} shows the counting-network construction used
to aggregate predictions over coalitions of a fixed cardinality.

\begin{figure}[t]
\centering
\begin{tikzpicture}[x=1cm,y=1cm,font=\small]
\foreach \k in {1,...,6}{
   \node[core,minimum size=5.4mm,font=\scriptsize] (g\k) at ({0.70+0.74*(\k-1)},0) {$\cG$};
   \node[qb,minimum width=5.4mm,minimum height=5.4mm] (k\k) at ({0.70+0.74*(\k-1)},1.10) {$K$};
   \draw[bond] (g\k)--(k\k);}
\foreach \k in {1,...,5}{
   \draw[bond] (g\k)--(g\the\numexpr\k+1\relax);
   \draw[qbond] (k\k)--(k\the\numexpr\k+1\relax);}
\node[qv,minimum size=3.8mm] (l0) at (-0.16,1.10) {$0$};
\node[qv,minimum size=3.8mm] (lS) at (5.02,1.10) {$s$};
\draw[qbond] (l0)--(k1);  \draw[qbond] (k6)--(lS);
\node[tndim,anchor=north] at (2.20,-0.36) {blue bond: model state $r$};
\node[tndim,anchor=south,text=tnorange!60!black] at (2.20,1.44) {orange bond: running count $a_j$};
\draw[decorate,decoration={brace,amplitude=3.5pt,mirror,raise=2pt},
      draw=tnink!45,line width=.5pt] (-0.36,-0.66)--(5.16,-0.66);
\node[tncap,anchor=north] at (2.20,-0.94) {closed network $=\sum_{|S|=s}\widehat v(S)$};
 
\begin{scope}[xshift=6.90cm,yshift=0.00cm]
  \node[tncap,anchor=south,text=tnink] at (1.70,1.86)
       {\textbf{the counting core $K^{(j)}$ is a counter}};
  \foreach \c/\lab in {0/0,1/1,2/2}{
     \node[tndim,anchor=east,text=tnink] at (-0.16,{0.58*\c}) {$a_{j-1}{=}\lab$};}
  \foreach \c in {0,1,2}{
     \foreach \x in {0,1,2}{
        \fill[tnorange!75!black] ({1.02*\x},{0.58*\c}) circle (1.4pt);}}
  \foreach \c in {0,1,2}{
     \foreach \x in {0,1}{
        \draw[qbond,->,draw=tnorange!45] ({1.02*\x+0.10},{0.58*\c}) -- ({1.02*\x+0.92},{0.58*\c});}}
  \foreach \c in {0,1}{
     \foreach \x in {0,1}{
        \draw[qbond,->,draw=tnorange!85!black] ({1.02*\x+0.09},{0.58*\c+0.06}) -- ({1.02*\x+0.90},{0.58*\c+0.52});}}
  \node[tndim,anchor=north,text=tnink] at (1.02,-0.20) {sites $j=1,\dots,d$};
  \draw[qbond,->,draw=tnorange!45] (2.55,1.34) -- (3.05,1.34);
  \node[tncap,anchor=west,text=tnink] at (3.10,1.34) {$z_j{=}0$: keep};
  \draw[qbond,->,draw=tnorange!85!black] (2.55,0.86) -- (3.05,1.06);
  \node[tncap,anchor=west,text=tnink] at (3.10,0.96) {$z_j{=}1$: advance};
\end{scope}
\end{tikzpicture}
\caption{Summing predictions at a fixed coalition size.
Contracting the fitted datamodel (blue) with a counting
weight tensor (orange) computes
$\sum_{|S|=s}\widehat v(S)$.
The counter starts at zero, increments when a group is
included, and terminates at $s$; it has at most $d+1$ states.
For the symmetric semivalue of group $i$, the corresponding
construction places the difference vector
$\boldsymbol\delta$ at $i$ and counts only the remaining
$d-1$ groups. A terminal count $s$ receives weight
$p_s/\binom{d-1}{s}$, where $p_s$ is the probability assigned
to context size $s$
(Proposition~\ref{prop:semivalue}). The tensor network representation lets us aggregate predictions over all subsets of a given size without enumerating those subsets individually.}
\label{fig:size}
\end{figure}
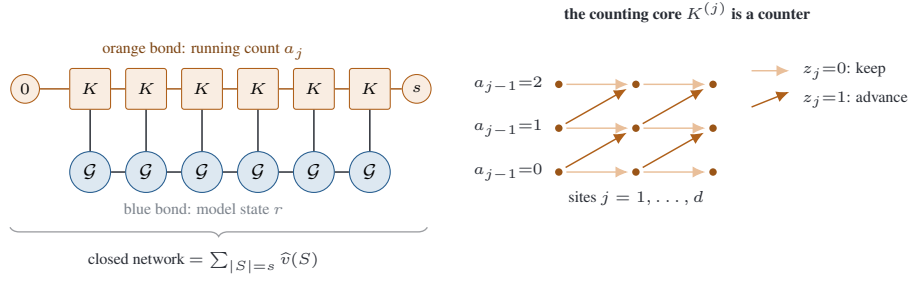

\subsection{Interaction Coefficients}

Under the indexing convention of Sec.~\ref{sec:tndm}, the
value-to-coefficient transformation on each physical mode is
\[
    M
    =
    A^\top
    =
    \begin{bmatrix}
        1 & -1\\
        0 & 1
    \end{bmatrix},
    \qquad
    M^{-1}
    =
    \begin{bmatrix}
        1 & 1\\
        0 & 1
    \end{bmatrix},
\]
where $A$ is defined in
Appendix~\ref{app:set-function-foundations}.
The first transformed coordinate selects the interaction component,
while the second selects the constant component.

\begin{proposition}[Rank-preserving interaction transform]
\label{prop:mobius}
The M\"obius transform acts locally on every physical mode and
therefore preserves the tensor network topology and every internal
rank of $\widehat{\cT}$.
For a rank-$r$ TT, the local transformation costs
$\cO(dr^2)$.

Fixing the transformed physical leg of training unit $i$ to the
interaction coordinate leaves an order-$(d-1)$ tensor whose entries
are precisely the M\"obius coefficients indexed by sets containing
$i$. Contracting the remaining modes with the counting network gives
\[
    \widehat h_{i,k}
    =
    \sum_{\substack{A\ni i\\|A|=k}}
    \widehat m(A)
\]
for all interaction orders $k$, at cost
$\cO(d^2r^2)$.
\end{proposition}

\begin{proof}
The Boolean M\"obius transformation applies the matrix $M$
independently to every physical mode.
Multiplication of a physical leg by an invertible local matrix
modifies only the corresponding tensor core and leaves every
internal bond dimension unchanged.

Fixing the transformed leg of unit $i$ to the interaction coordinate
selects coefficients indexed by sets containing $i$.
Contracting the remaining modes with the counting network groups
those coefficients by cardinality, giving the stated quantities and
the $\cO(d^2r^2)$ complexity.
\end{proof}

Figure~\ref{fig:mobius} illustrates the local M\"obius change of basis,
the interaction slice associated with one training unit, and the
subsequent aggregation by interaction order.

\begin{figure}[t]
\centering
\begin{tikzpicture}[x=1cm,y=1cm,font=\small]
\node[tnttl,anchor=south] at (2.20,2.24) {(a) open the interaction slice $\cI^{(i)}$};
\foreach \k in {1,...,7}{
   \node[core,minimum size=5.0mm] (g\k) at ({0.40+0.60*(\k-1)},0) {$\cG$};
   \node[loc] (m\k) at ({0.40+0.60*(\k-1)},0.80) {$M$};
   \draw[bond] (g\k)--(m\k);}
\foreach \k in {1,...,6}{\draw[bond] (g\k)--(g\the\numexpr\k+1\relax);}
\foreach \k in {1,2,3,5,6,7}{
   \draw[oleg] (m\k)--++(0,0.44);
   \node[tndim,anchor=south,text=tngreen!55!black] at ($(m\k)+(0,0.42)$) {$a_{\k}$};}
\node[qv] (one) at ($(m4)+(0,0.62)$) {$e_1$};
\draw[bond] (m4)--(one);
\node[tndim,anchor=south,text=tnorange!65!black] at ($(one)+(0,0.24)$) {fix $a_i{=}1$};
\node[tncap,anchor=north,text width=4.4cm,align=center] at (2.20,-0.36)
     {$d-1$ legs stay \emph{open}: the output is a tensor, not a scalar};
 
\begin{scope}[xshift=5.55cm]
  \node[tnttl,anchor=south] at (1.75,2.24) {(b) summarize it by order};
  \node[tn,minimum width=30mm,minimum height=8mm,fill=tngreen!10,draw=tngreen!80!black]
       (I) at (1.75,0.20) {$\cI^{(i)}$};
  \foreach \x in {-1.24,-0.62,0,0.62,1.24}{
     \draw[oleg] ({1.75+\x},0.60)--({1.75+\x},0.86);
     \node[qb,minimum width=4.6mm,minimum height=4.6mm] at ({1.75+\x},1.10) {$K$};}
  \foreach \x/\y in {-1.24/-0.62,-0.62/0,0/0.62,0.62/1.24}{
     \draw[qbond] ({1.75+\x+0.24},1.10)--({1.75+\y-0.24},1.10);}
  \node[qv,minimum size=3.8mm] (z0) at (-0.05,1.10) {$0$};
  \node[qv,minimum size=5.4mm] (zk) at (3.58,1.10) {$k{-}1$};
  \draw[qbond] (z0)--(0.28,1.10);
  \draw[qbond] (3.22,1.10)--(zk);
  \node[tncap,anchor=north,text width=4.6cm,align=center] at (1.75,-0.36)
       {$H_{i,k}=\inner{\cI^{(i)}}{C^{(-i)}_{k-1}}$\\[2pt]
        $=\displaystyle
  \sum_{\substack{A\subseteq N\\ i\in A,\ |A|=k}}
  \widehat m(A)$};
\end{scope}
 
\begin{scope}[xshift=10.10cm,yshift=-0.10cm]
  \draw[panelbox] (0.00,-0.30) rectangle (3.80,2.34);
  \node[tncap,anchor=north,text width=3.5cm,align=center,text=tnink] at (1.90,2.26)
    {\textbf{green node} $=$ local change of basis};
  \node[anchor=north,inner sep=2pt] at (1.90,1.62)
    {\scriptsize$M=\left[\begin{smallmatrix}1&-1\\0&1\end{smallmatrix}\right]$,\;\;
 $M^{-1}=\left[\begin{smallmatrix}1&1\\0&1\end{smallmatrix}\right]$};
  \node[tncap,anchor=north,text width=3.5cm,align=center] at (1.90,0.86)
    {values $\leftrightarrow$ M\"obius basis. The graph topology and every bond
     dimension are unchanged (Prop.~\ref{prop:mobius}).};
\end{scope}
\end{tikzpicture}
\caption{Interaction discovery for one training unit (example $i=4$). (a)~A
green $2\times2$ matrix on every physical leg rewrites the network in the
M\"obius basis; fixing the transformed leg of unit $i$ to $1$ leaves an
order-$(d-1)$ tensor whose entries are exactly the interactions containing $i$.
Unlike a scalar attribution, this tensor-valued transformation deliberately keeps its
legs open. (b)~Contracting those legs with the counting weight tensor of
Figure~\ref{fig:size} collapses the slice into the degree profile $H_{i,k}$. The fitted TNDM retains information about interactions, allowing us to inspect how strongly each training unit participates in interactions of different orders.}
\label{fig:mobius}
\end{figure}
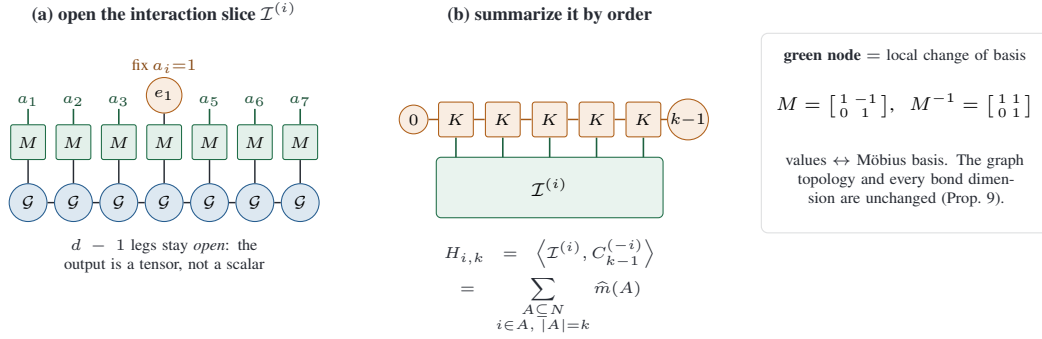

\section{Fixed-Score Attribution and Subset Prediction}
\label{app:if-trak}

Influence functions, TracIn, and TRAK compute attribution scores
using information from fitted learners or training trajectories
\citep{koh2017influence,pruthi2020tracin,park2023trak}.
Suppose a score vector
$\mathbf a=(a_1,\ldots,a_d)\in\mathbb R^d$
has been computed once and then reused to predict outcomes for
different training subsets. The resulting subset predictor has the
form
\[
    \widehat v_{\mathbf a}(S)
    =
    b+\sum_{i\in S}a_i .
\]

\begin{proposition}[A fixed-score subset predictor is additive]
\label{prop:score-induced}
For any fixed vector $\mathbf a\in\mathbb R^d$, the predictor
\[
    \widehat v_{\mathbf a}(S)
    =
    b+\sum_{i\in S}a_i
\]
has interaction degree at most one and TT ranks at most two in every
ordering. The marginal contribution of training unit $i$ equals
$a_i$ in every context. Consequently, every normalized symmetric
semivalue of training unit $i$ equals $a_i$.
\end{proposition}

\begin{proof}
Using coalition indicators,
\[
    \widehat v_{\mathbf a}(\mathbf z)
    =
    b+\sum_{i=1}^{d}a_i z_i .
\]
The multilinear expansion contains no coefficients of degree greater
than one, so
\[
    \deg(\widehat v_{\mathbf a})\le 1.
\]

For $d\ge2$, a TT representation with ranks at most two is given by
\[
    G_1[2-z_1]
    =
    \begin{bmatrix}
        1 & b+a_1z_1
    \end{bmatrix},
\]
\[
    G_j[2-z_j]
    =
    \begin{bmatrix}
        1 & a_jz_j\\
        0 & 1
    \end{bmatrix},
    \qquad
    1<j<d,
\]
and
\[
    G_d[2-z_d]
    =
    \begin{bmatrix}
        a_dz_d\\
        1
    \end{bmatrix}.
\]
Multiplying the cores gives
\[
    G_1[2-z_1]\cdots G_d[2-z_d]
    =
    b+\sum_{i=1}^{d}a_i z_i.
\]
The case $d=1$ is immediate.

For every context
$S\subseteq N\setminus\{i\}$,
\[
    \widehat v_{\mathbf a}(S\cup\{i\})
    -
    \widehat v_{\mathbf a}(S)
    =
    a_i.
\]
Hence the marginal-contribution curve of
Eq.~\eqref{eq:marginal-curve} satisfies
\[
    \widehat g_i(t)=a_i
    \qquad
    \text{for all }t\in[0,1].
\]
Every normalized symmetric semivalue therefore returns $a_i$.
\end{proof}

Proposition~\ref{prop:score-induced} concerns the reuse of a
\emph{fixed} score vector across training subsets. Influence
functions, TracIn, and TRAK may use gradients, curvature information,
checkpoints, projections, or ensembles when producing the score
vector. Recomputing the attribution scores after changing the
training subset can produce context-dependent quantities outside the
fixed-score model above.

The coalition-prediction experiments in the paper compare functions
fitted directly to observed subset--outcome pairs. Consequently, the
reported prediction comparisons concern the evaluated datamodel
families. A comparison with attribution algorithms that use
additional information from the trained learner or training
trajectory requires a separate experimental protocol.
\section{Surrogate classes}
\label{app:surrogates}
\label{app:surrogate-definitions}

Let $\mathbf z\in\{0,1\}^d$ denote the indicator vector of
a training coalition. We compare datamodels that impose
different restrictions on the set function predicted from
these indicators.

\paragraph{Linear and polynomial datamodels.}
The linear datamodel is the affine set function defined in
Sec.~\ref{sec:tndm}. A degree-$p$ polynomial datamodel
has the form
\[
    \widehat v(\mathbf z)
    =
    \sum_{\substack{A\subseteq N\\|A|\leq p}}
    \theta_A\prod_{i\in A}z_i.
\]
Degree-2 and degree-3 datamodels correspond to quadratic
and cubic models, respectively; the degree specifies
the maximum interaction order included.
Dense polynomial models fit all monomials through the
specified degree using ridge regression. Sparse degree-3
models select monomials using orthogonal matching pursuit.
Within a degree-bounded model, factorizing the coefficient
tensors imposes an additional rank constraint.

\paragraph{Kernel datamodels.}
RBF kernel ridge uses the kernel
\[
    K_\gamma(\mathbf z,\mathbf z')
    =
    \exp\!\left(
        -\gamma\|\mathbf z-\mathbf z'\|_2^2
    \right),
\]
with bandwidth and regularization selected using
validation coalitions.

\paragraph{Tensor network datamodels.}
Direct CP, TT, and hierarchical Tucker (HT) datamodels
approximate the value tensor using the corresponding
factorization, without imposing a polynomial-degree cutoff.
The TT degree ladder first fits degree-bounded predictors
and then represents them as tensor-trains. We distinguish
this procedure from fitting a TT directly to the observed
coalition outcomes.

Where cardinality adjustment is used, the method label
identifies the surrogate component fitted after adjustment.
Its degree need not equal the degree of the complete
predictor. Preprocessing and fitting configurations are
specified separately for each experiment.

\section{Experimental details}
\label{app:protocol}

\subsection{Relation-composition task}
\label{app:khop-details}
Our construction draws on the controlled relational-composition
setting of \citet{wang2024grokked}, adapted to training data
attribution through removable groups of relation examples.
We use three-step chains with explicitly generated intermediate
entities, whereas their composition experiments study two-hop
inference without emitting the intermediate entity.
\paragraph{Entities and relations.}
Entities are elements of $\mathbb Z_{31}$, and each relation
specifies an additive shift modulo 31. For relation sequence
$a$ and initial entity $x$, the evaluation query $q_{a,x}$
asks for the terminal entity in
\[
    s_0=x,
    \qquad
    s_h=s_{h-1}+b_{a,h}\pmod{31},
    \quad h=1,2,3.
\]
The model generates the intermediate entities before
predicting the terminal entity. Each relation sequence
is evaluated at all 31 initial entities.

\paragraph{Training groups.}
Let $G_{a,h}$ denote the groups supporting the $h$-th
coalition-dependent relation of sequence $a$. These sets
are disjoint within each sequence, although a group may
support relations in several sequences. Table~\ref{tab:task-design}
summarizes the three configurations. All have three reasoning
steps. At $d=24$, the third relation is generic and is taught
through background data available to every coalition.

\begin{table}[t]
    \centering
    \small
    \begin{tabular}{rrrr}
        \toprule
        Training groups &
        Relation sequences &
        Required relations &
        Groups per relation \\
        \midrule
        8  & 8  & 3 & 2 \\
        12 & 10 & 3 & 2 \\
        24 & 10 & 2 & 3 \\
        \bottomrule
    \end{tabular}
    \caption{Reasoning-task configurations. Required relations
    are those whose training examples depend on the selected
    coalition. Each relation sequence defines 31 evaluation
    queries, one per initial entity. The three reasoning settings vary the number of training groups while keeping the same basic relation-composition structure.}
    \label{tab:task-design}
\end{table}

\paragraph{Relation examples.}
Each required relation contributes 372 training examples:
31 initial entities, three positions for its token, and
four generic contexts. Each example contains one
coalition-dependent relation token; the other two positions
contain generic relation tokens.

Groups supporting the same relation supply identical
examples of that relation. The coalition training pool
contains these examples once whenever at least one
supporting group is selected. Evaluation combines the
required relation tokens within a single chain, whereas
training presents them separately. All 31 initial entities
occur in training, so evaluation tests composition of
separately learned relations.

\paragraph{Absent relations.}
When no selected group supports a required relation, its
token embedding is initialized to the identity-relation
embedding and tied to it after every optimizer step.
The shifts are chosen so that no nonempty subset of the
coalition-dependent shifts sums to zero modulo 31.
This prevents omitted shifts from cancelling each other.

\paragraph{Background data and sampling.}
Background examples are three-step chains with randomly
sampled initial entities, shifts, and generic relation
aliases. They contain no coalition-dependent relation
tokens. Approximately one quarter of each background batch
consists of examples with two identity relations.

For a nonempty coalition, each training batch contains
410 background examples and 102 examples sampled with
replacement from the selected relation pool. The empty
coalition uses 512 background examples. Training steps
and batch size are fixed across coalitions. Consequently,
the expected exposure per required relation decreases
as the selected relation pool grows.

\subsection{Transformer learner and training}
\label{app:learner}

\paragraph{Architecture.}
The learner is an eight-layer decoder-only transformer
with width 256, eight attention heads, and MLP expansion
ratio four. It uses parameter-free pre-RMSNorm, rotary
position embeddings, QK RMS normalization, ReLU-squared
MLP activations, value-residual mixing, an embedding
shortcut, and an untied output layer. Logits are capped
using a scaled $\tanh$ with cap 30. There are no biases
or dropout.

The context length is nine tokens. The vocabulary contains
31 entity tokens, 93 generic relation tokens, one token
for each coalition-dependent relation of each sequence,
and an equality token. Vocabulary sizes are 149, 155,
and 145 for $d=8$, $d=12$, and $d=24$, respectively.

\paragraph{Training objective.}
Training is teacher-forced on sequences
$
    [x,r_1,=,s_1,r_2,=,s_2,r_3,=].
$
Cross-entropy is applied to the three next-entity
predictions, with the softmax restricted to the 31
entity logits. The intermediate entities supplied
during training are the correct ones; at evaluation,
they are generated greedily by the model.

\paragraph{Optimization.}
Each coalition is trained from scratch for 1,000 AdamW
steps with batch size 512, learning rate $10^{-3}$,
and weight decay $0.1$ applied to all parameters.
AdamW uses its default momentum and numerical-stability
settings. There is no learning-rate schedule, warmup,
gradient clipping, or mixed precision. We evaluate the
final model; intermediate logging is not used to select
a checkpoint.

Within each experimental configuration and learner seed,
coalitions share the same initialization. The learner
seed also controls data generation. The blind $d=12$
replication and all $d=24$ runs use Quadro RTX 8000 GPUs.

\subsection{Measured outcomes}
\label{app:outcomes}

For evaluation query $q_{a,x}$, the model first generates
intermediate entities $\widehat s_1$ and $\widehat s_2$
greedily. Let $\ell_{a,x}(S)$ denote the natural logarithm
of the probability assigned to the correct terminal entity,
conditional on the resulting context. The softmax is
restricted to the 31 entity logits.

The outcome is therefore a final-answer log probability,
rather than the log probability of the entire generated
sequence. Higher values indicate better predictions.
For each relation sequence, we also record the mean
\begin{equation}
    v_a(S)
    =
    \frac{1}{31}\sum_{x=0}^{30}\ell_{a,x}(S).
        \label{eq:khop-utility}
\end{equation}
We record both the individual-query
outcomes and these sequence means, together with the
corresponding final-answer accuracies.

The $d=12$ and $d=24$ datamodels are fitted to the
sequence means $v_a(S)$, separately for each learner
seed. Outcomes are not averaged across learner seeds
before fitting.

For a fixed learner seed, complete coalition enumeration
provides exact attribution references for the measured
set function. It does not give the expectation over
training randomness. Likewise, the evaluation queries
answered by one model share a single retraining run
and are not independent retraining replicates.
\subsection{Coalition sampling and evaluation splits}
\label{app:coalition-designs}

Within each experiment and observation budget, all
datamodels receive the same coalition--outcome pairs.
Validation coalitions are used for hyperparameter
selection. Test outcomes are excluded from fitting
and model selection.

\paragraph{Complete enumeration at $d=8$.}
We retrain the learner on all $2^8=256$ coalitions.
Datamodels receive only sampled coalition observations,
partitioned into fitting and validation sets. Predictions
are evaluated on the remaining coalitions. The prediction budget sweep
uses four fixed initial-entity-zero outputs, with a 75\%/25\%
fitting/validation split and three repeated coalition splits. Validation
selects hyperparameters within each family, output, and split; NMSE is
then summarized by the median across the four outputs and three splits.
Complete enumeration also supplies reference outcomes.

\paragraph{Fixed observation design at $d=12$.}
The observation design is constructed without retraining
outcomes, using pivoted QR on group-indicator monomials
of degree at most three, subject to cardinality quotas.
It contains 160 new coalitions: 52 of size two, 52 of
size four, and eight of each size from five through
eleven. Nested prefixes define budgets of
32, 48, 64, 96, 128, and 160 new observations.

All $\binom{12}{3}=220$ size-three coalitions are excluded
from surrogate fitting and used for evaluation.
Fourteen additional size-two coalitions complete the
pair observations needed to compute third-order
differences at the empty coalition. A separate collection
of gate observations checks the learner's behavior;
these observations are not included in the reported
count of new coalitions.
At 160 new observations, 177 distinct outcomes are used for fitting
and validation (141 fitting and 36 validation); additional learner checks are accounted for separately from these
fitting and validation outcomes.

The replication uses learner seeds 43 and 59 as blind
evaluation seeds. Seed 31 is reported separately because
results from the earlier experiment were available
during development. Datamodel fits and set nominations
are frozen before their evaluation outcomes are obtained.

\paragraph{Random observation design at $d=24$.}
Coalitions are sampled by including each training group
independently with probability $1/2$. Duplicate masks
are discarded, and the remaining masks are retained
in draw order. Nested prefixes of 256, 512, and 1,024
coalitions define the observation budgets. Each budget
includes a 20\% validation allocation stratified by
coalition cardinality.

A disjoint random set of 256 coalitions is used for
prediction evaluation. An additional 46 evaluation
coalitions comprise the full and empty sets, 24
leave-one-out sets, and 20 sets obtained by removing
all supporting groups for one required relation.
The full and empty coalitions are excluded from the
training masks.

The same observation and evaluation masks are used
across learner seeds 31, 43, and 59. Each datamodel
is fitted separately for each seed, so comparisons
across seeds hold the coalition design fixed.

\subsection{Surrogate fitting and model selection}
\label{app:surrogate-fitting}

Datamodels are fitted to the observed coalition indicators
and measured outcomes. Hyperparameters are selected using
validation coalitions. Test outcomes and attribution
references derived from test outcomes are excluded
from model selection.
\paragraph{Fitting objective.}
Let $\mathcal O_{\mathrm{fit}}=\{(S_b,y_b)\}_{b=1}^{n_{\mathrm{fit}}}$
be the fitting observations. The regularized fitting objective is
\begin{equation}
\label{eq:tndm-objective}
\min_\theta\;
\frac{1}{n_{\mathrm{fit}}}\sum_{b=1}^{n_{\mathrm{fit}}}
w_b\bigl(\widehat v_\theta(S_b)-y_b\bigr)^2
+\lambda\Omega(\theta),
\end{equation}
where $w_b$ are observation weights and $\Omega$ is the configured
regularizer. When cardinality adjustment is used, the tensor model
is fitted to the residual outcomes defined below. Ranks and
regularization settings are selected using validation coalitions.
\paragraph{Linear--tensor hybrids.}
\label{app:linear-tensor-hybrids}
For Linear + TNDM-TT, the residual component above takes
the form
$
r_a(S)=\ell_a(S)+t_a(S),
$
where $\ell_a$ is a fitted linear predictor and $t_a$
is a TT residual correction. The complete predictor retains
the shared cardinality adjustment.
Attribution values are computed from the combined predictor.
Setting the TT correction to zero recovers the fitted
linear baseline.
The Linear + CP and Linear + HT variants use the
corresponding tensor parameterizations for the correction.
\paragraph{Cardinality adjustment at $d=12$ and $d=24$.}
For each relation sequence $a$, the fitted predictor
has the form
$
\widehat v_a(S)
=
p_a(u(S)) + o_a(|S|) + r_a(S),
\qquad
u(S)=\frac{|S|-d/2}{d/2}.
$
Here $p_a$ is a degree-3 polynomial fitted by weighted
ridge regression.
Each fitting observation receives weight inversely
proportional to the number of fitting coalitions
of the same size, with weights normalized to have
mean one.
For each observed size $s$, $o_a(s)$ is the mean
residual after subtracting $p_a$ from the fitting
outcomes of that size.
We set $o_a(s)=0$ when no fitting coalition has
size $s$.
The residual model $r_a$ is then fitted to the
remaining signal.

The cardinality baseline sets $r_a=0$.
The cardinality-plus-linear baseline uses a linear
residual model.
Polynomial, kernel, and tensor network variants
use their respective residual models with the
same cardinality adjustment.
Thus, the family names in these comparisons
identify the residual model; polynomial degree
and tensor rank describe that component.

Cardinality coefficients, size offsets, observation
weights, and residual scales are computed from
fitting outcomes only.
Validation outcomes select model configurations
and stopping points.
The selected predictors are not refitted on the
combined fitting and validation sets.

\paragraph{Tensor orders.}
Known relation-to-group assignments define candidate tensor orders.
At \(d=12\), TT considers the natural order and orders grouping the
supporting groups of selected relation sequences. At \(d=24\), it considers
a relation-based order and a fixed shuffled order. Validation selects
among these candidates. CP and HT use the configured relation-based
mode arrangement.

\paragraph{Polynomial and kernel configurations at $d=24$.}
Polynomial ridge considers 13 logarithmically spaced
regularization values between $10^{-5}$ and $10^2$.
Sparse degree-3 regression uses orthogonal matching pursuit
with the number of selected terms chosen from
$
    \{4,8,16,32,64,128\}.
$
For RBF kernel ridge, let $\delta_{\mathrm{med}}$ denote
the median squared distance between coalition indicators.
The kernel coefficient is selected from
$
    \gamma
    =
    \frac{1}{\eta\,\delta_{\mathrm{med}}},
    \qquad
    \eta\in\{0.25,0.5,1,2,4\},
$
using the same ridge-regularization grid.

\paragraph{Tensor configurations at $d=24$.}
Direct CP, TT, and HT models consider ranks
$\{2,4,8,16\}$ and weight-decay/function-penalty pairs
$
    (10^{-3},10^{-4})
    \quad\text{and}\quad
    (10^{-2},10^{-3}).
$
Optimization uses AdamW with learning rate $0.02$,
at most 201 steps, and gradient clipping at 10.
Validation is evaluated every ten steps, with early
stopping.
Configurations are selected by validation mean-squared
error.

The TT degree ladder uses polynomial predictors through
degree six and distills them into tensor-trains on
2,048 coalitions, with ranks selected from $\{2,4,8\}$.
This procedure is reported separately from direct TT
fitting.

\subsection{Evaluation metrics}
\label{app:metrics}

\paragraph{Coalition prediction.}
For a fixed learner seed, let $\mathcal E$ denote the
test coalitions and let $M$ be the number of relation
sequences. Define
$
    \overline v_a
    =
    \frac{1}{|\mathcal E|}
    \sum_{S\in\mathcal E}v_a(S).
$
In the $d=24$ experiment, normalized mean-squared error
is computed as
$
    \operatorname{NMSE}
    =
    \frac{1}{M}\sum_{a=1}^{M}
    \frac{
        \sum_{S\in\mathcal E}
        \bigl(\widehat v_a(S)-v_a(S)\bigr)^2
    }{
        \sum_{S\in\mathcal E}
        \bigl(v_a(S)-\overline v_a\bigr)^2
    }.
$
Here, $\mathcal E$ contains the 256 random test coalitions.
We report the mean across three learner seeds and their
range. These seeds share the task and coalition design.

\paragraph{Interaction discovery.}
For the $d=12$ experiment, each triple $T\subseteq N$
receives the score
$
    \left|\Delta_T\widehat v_a(\emptyset)\right|
    =
    \left|
        \sum_{A\subseteq T}
        (-1)^{3-|A|}\widehat v_a(A)
    \right|.
$
We rank all 220 triples by this score and compute average
precision, taking complete triples as positives.
A triple is complete when it contains one supporting
group for each required relation.

The reported average precision is averaged over ten
relation sequences and the two blind learner seeds.
This measures recovery of the structural completeness
labels, rather than numerical accuracy of all third-order
differences. Interaction discovery is a secondary endpoint;
comparisons across methods and budgets are not
multiplicity-corrected.

\paragraph{Training-set selection.}
For direct selection, candidate coalitions of a fixed
size are ranked by predicted utility. Let
$\mathcal R_a$ denote the retained recommendations
for relation sequence $a$. We evaluate
\[
    \operatorname{Completeness}(\mathcal R_a)
    =
    \frac{1}{|\mathcal R_a|}
    \sum_{S\in\mathcal R_a}c_a(S),
    \qquad
    \operatorname{Utility}(\mathcal R_a)
    =
    \frac{1}{|\mathcal R_a|}
    \sum_{S\in\mathcal R_a}v_a(S).
\]
Completeness is determined by the support sets, whereas
utility is measured by retraining.

The primary $d=12$ endpoint retains three size-four
coalitions per relation sequence, excluding fitted
coalitions from the candidate set. The $d=24$
structural comparison retains three size-three
coalitions per sequence. A separate seed-31 evaluation
retrains the highest-ranked nomination per sequence.
Attribution-based selection is reported separately
from direct utility ranking.
\subsection{Modular retraining: attribution from one fit}
\label{app:modular-attribution}
We construct a controlled retraining task with eight training data
groups. The groups are $A_{\mathrm{primary}}$, $B$, $C$,
$A_{\mathrm{replica}}$, and four distractor groups. Each contains
eleven training facts describing a relation on eleven entities;
the two $A$ groups provide the same relation. The seven distinct
relations are affine permutations modulo eleven, with
multiplier--offset pairs
$(1,1),(2,0),(1,3),(3,1),(4,2),(5,3),(7,4)$.
The learner consists of seven relation-specific $11\times11$
conditional-logit tables.
For each learner seed in $\{31,43,59\}$, every coalition starts
from the same seed-specific initialization, with entries sampled
from $\mathcal{N}(0,0.05^2)$.
Each nonempty coalition is trained for 200 full-batch Adam
updates at learning rate $0.1$; the empty coalition retains
the initial state.
The primary utility is the mean log probability of the correct
target after composing the learned transition matrices for
$A\rightarrow B\rightarrow C$.
We measure all $2^8=256$ coalition outcomes for each seed.
This conditional-logit experiment is distinct from the
transformer-retraining experiments: its players are removable
training groups, not input features.

For each learner seed and budget
$m\in\{8,16,32,64,128\}$, we draw 20 observation sets by
uniformly sampling $m$ distinct coalitions without replacement.
The empty and full coalitions are not forced.
All surrogate methods share the observation set and internal
fitting/validation split within each repetition.
Budgets are sampled independently rather than as nested prefixes.
The budget includes validation observations.

Each surrogate jointly predicts four utility outputs:
the three-relation composition, the $A\rightarrow B$ composition,
an additive atomic-relation utility, and a redundant-$A$ utility.
Hyperparameters are selected by prediction MSE averaged over
the validation coalitions and all four outputs.
Figure~\ref{fig:main-har-attribution} reports attribution errors
for the primary three-relation output.
Shapley values, pairwise interactions, third-order interactions,
and held-out coalition predictions all use the same fitted
surrogate. Exact reference attributions and held-out errors
are not used for model selection.

\paragraph{Surrogates and fitting.}
Linear and Degree-2 use ridge regression on binary inclusion
indicators and their corresponding monomials, with an
unpenalized intercept and 19 penalties logarithmically spaced
from $10^{-6}$ to $10^3$.
Degree-3 uses orthogonal matching pursuit, considering support
sizes $\{2,4,8,16,32,64,128\}$ subject to the fitting-sample
limit, with a shared support across outputs.
RBF uses
$\exp(-\gamma\,\mathrm{Hamming}(S,T))$,
with $\gamma\in\{0.05,0.1,0.25,0.5,1,2\}$ and the same
ridge-penalty grid.
CP and TT are fitted by gradient descent. CP considers ranks $\{1,2,3,4,6,8,12,16\}$ and TT considers
$\{1,2,3,4,6,8\}$.
Both use ridge coefficients $\{10^{-4},10^{-2},1\}$,
two restart candidates.
\paragraph{Attribution accuracy.}
For completely enumerated set functions, reference
attribution quantities are computed from the measured
coalition outcomes. The corresponding quantities are
then evaluated on each fitted datamodel.
Figure~\ref{fig:main-har-attribution} reports relative $\ell_2$ errors
for the modular retraining task.

Reference values from enumeration are exact for the
measured fixed-seed set function. They do not represent
an average over learner randomness.

\subsection{Hyperparameter search at \texorpdfstring{$d=12$}{d=12}}
\label{app:d12-hyperparameters}

The following configurations apply to the
residual models fitted after cardinality
adjustment.

\paragraph{Polynomial and kernel models.}
Polynomial ridge uses thirteen logarithmically
spaced regularization values from $10^{-5}$
to $10^2$.
Its design contains indicator monomials through
the specified degree, without an additional
intercept.
Fitting uses inverse-frequency cardinality
weights.
A single regularization value is selected
jointly across the ten relation sequences.

Degree-three OMP uses the same weighted
monomial design and considers sparsities
from one through sixty.
The sparsity is selected jointly across
sequences.
RBF kernel ridge considers
$
\gamma\in
\{0.05,0.1,0.2,0.35,0.5,0.75,1.0\}
$
with the same thirteen ridge parameters.
For fitting weights $w_i$, its weighted
kernel system uses the regularizer
$\alpha\,\operatorname{diag}(1/w_i)$.

\paragraph{Direct tensor models.}
CP, TT, and HT consider ranks
$
r\in\{2,4,8,16,32\}.
$
The weight-decay and functional-penalty
coefficients are paired as
$
(10^{-4},10^{-5}),\qquad
(10^{-3},10^{-4}),\qquad
(10^{-2},10^{-3}).
$
Each relation sequence has its own factors
or cores.
TT is evaluated at every observation budget;
the CP and HT variants are evaluated at
$m_{\mathrm{new}}=160$.

Before optimization, each sequence's residual
is divided by its fitting-set standard
deviation, with a minimum scale of $0.05$.
The objective combines cardinality-weighted
squared error with a penalty on mean squared
model output over 128 randomly generated
coalitions.
These additional coalitions require no
measured learner outcomes.

Optimization uses float64 parameters and
AdamW with learning rate $0.02$ and gradient
norm clipping at ten.
Training runs for at most 201 steps.
Validation is evaluated every ten steps,
with early stopping after six checks without
improvement.
The best validation checkpoint is retained.
Each configuration uses a deterministic
initialization; no additional restart search
is performed.

\paragraph{TT degree ladder.}
The degree ladder uses an interaction-degree decomposition
of the predictor, following the perspective of
\citet{convy2022interaction}:
\[
\widehat{v}(z)
=
\sum_{k=0}^{d}\widehat{v}^{(k)}(z),
\qquad
\widehat{v}^{(k)}(z)
=
\sum_{\substack{A\subseteq[d]\\|A|=k}}
c_A\prod_{i\in A}z_i.
\]
Each component collects interactions among exactly $k$
training groups. Restricting the sum to $k\leq K$ gives
a degree-truncated model. The ladder increases this cutoff
and selects the retained degree using validation outcomes,
with polynomial coefficients compressed into TT form.
At each polynomial degree, ridge candidates
are fitted over the same regularization grid.
The three candidates with lowest validation
error are converted to coefficient tensors
and compressed by TT-SVD at ranks
$
\{2,4,8,16,32,64\}.
$
The ladder always evaluates through degree
three.
Higher degrees are accepted only when the
relative validation-error improvement is at
least $1\%$; the search stops at the first
rejected degree.
This procedure compresses fitted polynomial
coefficients and is distinct from directly
optimizing a TT residual model.
\section{Additional experimental results}
\label{app:additional-results}
\label{app:full-results}

\subsection{Coalition prediction at \texorpdfstring{$d=24$}{d=24}}
\label{app:completion-results}

Table~\ref{tab:d24-prediction} reports prediction error
for all evaluated surrogate families. Direct TT and CP
achieve the lowest errors across the three observation
budgets. At 512 observations, both outperform degree-3
ridge, degree-3 OMP, and RBF at 1,024 observations.
This comparison holds separately for each learner seed.
The improvement is specific to the fitted representation:
HT and the TT degree ladder do not match the performance
of direct TT or CP in this experiment.

\begin{table}[t]
    \centering
    \footnotesize
    \setlength{\tabcolsep}{4pt}
    \renewcommand{\arraystretch}{1.15}
    \begin{tabular}{lccc}
        \toprule
        Method & $m=256$ & $m=512$ & $m=1{,}024$ \\
        \midrule
        Cardinality only
        & $0.972\;[0.971,0.972]$
        & $0.934\;[0.933,0.934]$
        & $0.918\;[0.917,0.919]$ \\
        Linear
        & $0.749\;[0.749,0.749]$
        & $0.687\;[0.686,0.687]$
        & $0.669\;[0.669,0.670]$ \\
        Degree-2 ridge
        & $0.626\;[0.625,0.627]$
        & $0.467\;[0.466,0.469]$
        & $0.348\;[0.347,0.348]$ \\
        Degree-3 ridge
        & $0.669\;[0.668,0.670]$
        & $0.510\;[0.508,0.511]$
        & $0.342\;[0.340,0.343]$ \\
        Degree-3 OMP
        & $0.859\;[0.857,0.860]$
        & $0.635\;[0.624,0.654]$
        & $0.271\;[0.247,0.284]$ \\
        RBF kernel ridge
        & $0.562\;[0.562,0.562]$
        & $0.386\;[0.385,0.387]$
        & $0.257\;[0.256,0.258]$ \\
        \midrule
        CP
        & $0.490\;[0.447,0.526]$
        & $\mathbf{0.148}\;[0.133,0.177]$
        & $0.087\;[0.084,0.089]$ \\
        TT
        & $\mathbf{0.472}\;[0.464,0.485]$
        & $0.188\;[0.184,0.191]$
        & $\mathbf{0.071}\;[0.066,0.078]$ \\
        HT
        & $0.951\;[0.924,0.966]$
        & $0.850\;[0.825,0.877]$
        & $0.766\;[0.716,0.828]$ \\
        TT degree ladder
        & $0.660\;[0.646,0.671]$
        & $0.432\;[0.429,0.437]$
        & $0.269\;[0.267,0.272]$ \\
        \bottomrule
    \end{tabular}
    \caption{Held-out NMSE on the $d=24$ reasoning task;
    lower is better. Entries are means over learner seeds
    31, 43, and 59, with seed ranges in brackets.
    Each seed's NMSE averages over ten relation sequences
    on 256 random test coalitions. The observation budget
    $m$ includes the validation allocation. Seeds share
    the task and coalition design. Bold denotes the lowest
    mean within each budget. Method labels follow the
    fitting procedure in Appendix~\ref{app:surrogate-fitting}. Direct CP and TT give the lowest prediction errors across all observation budgets, while other tensor parameterizations do not show the same advantage.}
    \label{tab:d24-prediction}
\end{table}

\paragraph{Fully enumerated eight-group task.}
The $d=8$ task contains eight optional training groups and eight
relation sequences. Each sequence requires three relations, with
two supporting groups per relation, giving a coverage indicator
of interaction degree six. Retraining on all $2^8=256$ coalitions
provides a complete reference for evaluating subset prediction.

The observation-budget sweep first averages test NMSE over the four
outcomes within each split, then reports the mean and sample standard
deviation over three splits. At 128 observations, TNDM-TT attains
mean NMSE 0.041, compared with 0.230 for degree-three ridge and
0.259 for RBF.

\subsection{Predicting joint-removal effects}
\label{app:joint-removal}

Using the fitted 24-group datamodels, we predict changes
in utility after removing either one training group or
every group supplying a required relation. Effects are
measured relative to the full training coalition.
Evaluation includes 24 single-group and 20
relation-removal masks, using the same learner seeds
and fitting designs as the prediction comparison.
No additional datamodel is fitted for these effects. Table~\ref{tab:d24-effects-full} reports errors across
observation budgets. At 1,024 observations, TNDM-TT
has the lowest joint-removal RMSE among all ten
families in each learner seed. This evaluates the
prediction of interacting deletions; selection of
retained training subsets is evaluated separately.
For single-group removal, HT has the lowest
mean error.

\begin{table}[t]
\centering
\footnotesize
\setlength{\tabcolsep}{4pt}
\begin{tabular}{lcccc}
\toprule
Method
& Joint, $m=256$
& Joint, $m=512$
& Joint, $m=1{,}024$
& LOO, $m=1{,}024$ \\
\midrule
Cardinality
& $3.472\pm0.019$
& $3.648\pm0.202$
& $3.698\pm0.022$
& $0.556\pm0.007$ \\

Linear
& $2.562\pm0.018$
& $2.663\pm0.257$
& $2.741\pm0.016$
& $1.184\pm0.004$ \\

Degree-2 ridge
& $2.672\pm0.012$
& $2.733\pm0.297$
& $2.125\pm0.007$
& $0.962\pm0.007$ \\

Degree-3 ridge
& $4.324\pm0.012$
& $5.053\pm0.230$
& $3.869\pm0.019$
& $1.326\pm0.005$ \\

Degree-3 OMP
& $3.780\pm0.349$
& $4.390\pm0.361$
& $6.292\pm0.057$
& $3.136\pm0.026$ \\

RBF
& $\mathbf{2.077\pm0.009}$
& $2.484\pm0.454$
& $2.100\pm0.035$
& $1.011\pm0.027$ \\

HT
& $3.395\pm0.101$
& $3.454\pm0.373$
& $3.392\pm0.139$
& $\mathbf{0.555\pm0.005}$ \\

TNDM-TT, degree ladder
& $5.016\pm0.426$
& $5.957\pm0.173$
& $4.303\pm0.104$
& $1.661\pm0.181$ \\

TNDM-CP
& $2.123\pm0.211$
& $\mathbf{1.875\pm0.436}$
& $1.939\pm0.258$
& $0.780\pm0.193$ \\

TNDM-TT
& $2.614\pm0.053$
& $2.174\pm0.355$
& $\mathbf{1.395\pm0.043}$
& $0.559\pm0.016$ \\
\bottomrule
\end{tabular}
\caption{RMSE of predicted changes in retrained utility,
in nats; lower is better. Joint removal deletes every
group supplying one required relation; leave-one-out
(LOO) removal deletes one group. Entries show means
$\pm$ sample SD over learner seeds 31, 43, and 59,
which share the task and coalition design.
Bold marks the lowest mean in each column. TNDM-TT is most accurate for joint removals, while single-group removal is easier and is best matched by HT in this experiment.}
\label{tab:d24-effects-full}
\end{table}

\subsection{Interaction discovery at \texorpdfstring{$d=12$}{d=12}}
\label{app:interaction-results}

Table~\ref{tab:d12-interactions} reports average precision
for identifying complete triples from predicted
third-order differences. TT improves on degree-3 ridge
at the larger observation budgets. At 160 new observations,
its AP is $0.779$ and $0.746$ on learner seeds 43 and 59,
respectively, compared with $0.356$ and $0.281$ for
degree-3 ridge.

The ordering is budget-dependent. At 48 and 64
observations, degree-3 ridge has higher mean AP than TT.
At 128 observations, the linear-plus-TT variant has
the highest mean AP among the reported methods.
Direct TT has the highest mean at 160 observations.

\begin{table}[t]
    \centering
    \small
    \setlength{\tabcolsep}{4pt}
    \renewcommand{\arraystretch}{1.1}
    \begin{tabular}{lrrrrrr}
        \toprule
        & \multicolumn{6}{c}{New coalition observations} \\
        \cmidrule(lr){2-7}
        Method & 32 & 48 & 64 & 96 & 128 & 160 \\
        \midrule
        Linear
        & 0.053 & 0.054 & 0.083 & 0.050 & 0.052 & 0.055 \\
        Degree-2 ridge
        & 0.057 & 0.045 & 0.049 & 0.044 & 0.042 & 0.052 \\
        Degree-3 ridge
        & 0.081 & 0.143 & 0.200 & 0.260 & 0.289 & 0.318 \\
        Degree-3 OMP
        & 0.071 & 0.094 & 0.074 & 0.083 & 0.101 & 0.093 \\
        RBF kernel ridge
        & 0.045 & 0.043 & 0.071 & 0.071 & 0.146 & 0.089 \\
        \midrule
        CP
        & -- & -- & -- & -- & -- & 0.313 \\
        TT
        & 0.081 & 0.120 & 0.173 & 0.305 & 0.423 & 0.762 \\
        HT
        & -- & -- & -- & -- & -- & 0.214 \\
        TT degree ladder
        & 0.078 & 0.143 & 0.200 & 0.260 & 0.406 & 0.306 \\
        \midrule
        Linear + CP
        & -- & -- & -- & -- & -- & 0.328 \\
        Linear + TT
        & 0.056 & 0.122 & 0.096 & 0.271 & 0.540 & 0.668 \\
        Linear + HT
        & -- & -- & -- & -- & -- & 0.114 \\
        \bottomrule
    \end{tabular}
    \caption{Average precision for interaction discovery
    on the $d=12$ reasoning task; higher is better.
    Each entry averages over ten relation sequences
    and learner seeds 43 and 59. All 220 triples are
    ranked by the absolute predicted third difference
    at the empty coalition, with complete triples as
    positives. Budgets count new coalition observations
    and exclude the separate gate collection.
    CP and HT variants were evaluated only at 160 new
    coalition observations. This is a secondary endpoint;
    comparisons across methods and budgets are not
    multiplicity-corrected. TNDM-TT becomes substantially better at recovering the true complementary triples as more coalition outcomes are observed, although the strongest method depends on the observation budget.}
    \label{tab:d12-interactions}
\end{table}

\subsection{Training-set selection}
\label{app:boundary-results}

\paragraph{Primary comparison at $d=12$.}
Table~\ref{tab:d12-selection} reports selection of
size-four coalitions at a budget of 160 new observations.
Each datamodel ranks the candidate coalitions by
predicted utility and retains three recommendations
per relation sequence.

TT, degree-3 ridge, RBF, and the TT degree ladder
identify complete highest-ranked coalitions for every
relation sequence in both blind learner seeds.
Among the three retained recommendations, TT has
completeness $0.983$, compared with $1.000$ for
degree-3 ridge, RBF, and the TT degree ladder.
Thus, the interaction-discovery advantage of TT
does not imply better training-set selection.

\begin{table}[t]
    \centering
    \small
    \setlength{\tabcolsep}{5pt}
    \renewcommand{\arraystretch}{1.1}
    \begin{tabular}{lrrr}
        \toprule
        Method
        & \shortstack{Top-1\\completeness}
        & \shortstack{Top-3\\completeness}
        & \shortstack{Top-3 utility\\(nats)} \\
        \midrule
        Linear
        & 0.650 & 0.633 & $-4.187$ \\
        Degree-2 ridge
        & 0.900 & 0.950 & $-0.585$ \\
        Degree-3 ridge
        & 1.000 & 1.000 & $-0.001$ \\
        Degree-3 OMP
        & 0.500 & 0.500 & $-5.546$ \\
        RBF kernel ridge
        & 1.000 & 1.000 & $-0.003$ \\
        \midrule
        CP
        & 0.850 & 0.883 & $-1.336$ \\
        TT
        & 1.000 & 0.983 & $-0.173$ \\
        HT
        & 0.650 & 0.583 & $-4.717$ \\
        TT degree ladder
        & 1.000 & 1.000 & $-0.001$ \\
        \midrule
        Linear + CP
        & 0.800 & 0.850 & $-1.704$ \\
        Linear + TT
        & 0.950 & 0.917 & $-0.967$ \\
        Linear + HT
        & 0.700 & 0.683 & $-3.596$ \\
        \midrule
        Random
        & 0.050 & 0.117 & $-9.985$ \\
        Cardinality / exact LOO
        & 0.100 & 0.100 & $-10.030$ \\
        \bottomrule
    \end{tabular}
    \caption{Selection of size-four training coalitions
    at $d=12$ and 160 new observations. Entries average
    over ten relation sequences and learner seeds
    43 and 59. Top-1 completeness evaluates the
    highest-ranked recommendation; top-3 completeness
    is the fraction of complete sets among the three
    retained recommendations. Top-3 utility is their
    mean measured final-answer log probability.
    Higher values are better for all three metrics.
    Cardinality and exact leave-one-out selectors
    have identical reported results. Several nonlinear methods, including TNDM-TT, reliably identify complete subsets; stronger interaction recovery does not automatically translate into better subset selection.}
    \label{tab:d12-selection}
\end{table}

Exact leave-one-out scores vanish in this construction:
removing one group from the full coalition leaves
every required relation supported and the deduplicated
training pool unchanged. The resulting selection
comparison therefore does not establish an advantage
over influence-based attribution methods in general.

\paragraph{Retrained recommendations at $d=24$.}
For each relation sequence, we rank all size-three
coalitions by predicted utility and retrain the
highest-ranked recommendation using learner seed 31.
Table~\ref{tab:d24-selection} reports completeness
and measured utility at observation budgets 512
and 1,024.

At 512 observations, TT selects complete coalitions
for all ten relation sequences, matching degree-3
ridge, OMP, and RBF. At 1,024 observations, these
three baselines remain complete on all ten sequences,
whereas TT is complete on eight. CP improves from
nine complete recommendations to ten. The prediction
advantage of TT and CP therefore does not establish
a corresponding advantage in this selection task.

\begin{table}[t]
    \centering
    \small
    \setlength{\tabcolsep}{6pt}
    \renewcommand{\arraystretch}{1.1}
    \begin{tabular}{lrrrr}
        \toprule
        & \multicolumn{2}{c}{$m=512$}
        & \multicolumn{2}{c}{$m=1{,}024$} \\
        \cmidrule(lr){2-3}
        \cmidrule(lr){4-5}
        Method
        & Complete & Utility
        & Complete & Utility \\
        \midrule
        Linear
        & 8/10 & $-2.324$
        & 8/10 & $-2.326$ \\
        Degree-3 ridge
        & 10/10 & $-0.001$
        & 10/10 & $-0.001$ \\
        Degree-3 OMP
        & 10/10 & $-0.001$
        & 10/10 & $-0.001$ \\
        RBF kernel ridge
        & 10/10 & $-0.001$
        & 10/10 & $-0.001$ \\
        \midrule
        CP
        & 9/10 & $-1.142$
        & 10/10 & $-0.001$ \\
        TT
        & 10/10 & $-0.001$
        & 8/10 & $-2.285$ \\
        \bottomrule
    \end{tabular}
    \caption{Retrained evaluation of the highest-ranked
    size-three coalition for each of ten relation
    sequences at $d=24$. Completeness counts sequences
    whose recommended coalition supports every required
    relation. Utility is the mean measured final-answer
    log probability in nats; higher is better.
    All entries use learner seed 31.
    These results concern one recommendation per
    sequence, rather than the top-three completeness
    metric averaged across learner seeds. TNDMs can select high-utility subsets, but their prediction advantage does not consistently produce better subset selection than the polynomial or kernel baselines.}
    \label{tab:d24-selection}
\end{table}

\subsection{Subset selection from individual attribution scores}
\label{app:shapley-selection}

For each fitted 12-group datamodel, we enumerate its
4,096 predicted coalition values and compute individual
Shapley values. We retain the three groups with the
largest scores, breaking exact ties by ascending group
index. Each selected subset is evaluated using its
recorded retrained utility and whether it supplies all
three required relations. This selects one subset per
evaluation sequence, rather than ranking candidate
subsets directly by predicted utility.

At 160 new observations, Linear + TNDM-TT selects
complete subsets for five of ten sequences in learner
seed 43 and seven of ten in seed 59. Its mean
completeness is 60\%, compared with 45\% for quadratic
ridge and 20\% for degree-3 ridge and RBF. It also achieves
the highest mean retrained utility among the compared
methods at this budget
(Table~\ref{tab:shapley-selection-endpoint}).
These results use the two blind learner seeds and
exclude the seed-31 development run.

This is a secondary comparison at a retrospectively
selected budget. The ordering varies across budgets
and attribution rules: degree-3 ridge reaches 75\%
Shapley-based completeness at 64 observations, while
RBF reaches 65\% at 96. 

\begin{table}[t]
\centering
\footnotesize
\setlength{\tabcolsep}{5pt}
\begin{tabular}{lcc}
\toprule
Method & Complete (\%) & Retrained utility (nats) \\
\midrule
Linear + TNDM-TT
& $\mathbf{60.0\pm14.1}$
& $\mathbf{-4.568\pm1.582}$ \\

TNDM-CP
& $45.0\pm7.1$
& $-5.989\pm0.770$ \\

Degree-2 ridge
& $45.0\pm7.1$
& $-6.354\pm0.783$ \\

Linear + TNDM-CP
& $45.0\pm7.1$
& $-6.418\pm0.923$ \\

Linear + HT
& $40.0\pm0.0$
& $-6.891\pm0.051$ \\

Degree-3 OMP
& $40.0\pm14.1$
& $-6.510\pm1.528$ \\

TNDM-TT
& $40.0\pm14.1$
& $-6.714\pm1.504$ \\

TNDM-TT, degree ladder
& $35.0\pm7.1$
& $-7.445\pm0.685$ \\

Linear
& $35.0\pm7.1$
& $-7.450\pm0.842$ \\

HT
& $30.0\pm14.1$
& $-8.042\pm1.964$ \\

Degree-3 ridge
& $20.0\pm14.1$
& $-9.092\pm1.887$ \\

RBF
& $20.0\pm0.0$
& $-9.180\pm0.137$ \\

Cardinality
& $0.0\pm0.0$
& $-11.463\pm0.173$ \\

Fixed LOO
& $0.0\pm0.0$
& $-11.463\pm0.173$ \\

Random scores
& $0.0\pm0.0$
& $-11.006\pm0.161$ \\
\bottomrule
\end{tabular}
\caption{Selection by individual Shapley values at
160 new observations, with 17 shared observations
also used for fitting and validation.
Entries show means $\pm$ sample SD across two
learner-seed means, each averaged over ten evaluation
sequences. Higher is better in both columns.
Completeness measures relation coverage; utility is
measured after retraining. Bold marks the highest
mean in each column. Random scores denotes the
random-score baseline, not expected performance under
uniform sampling of triples. At this budget, combining a linear model with TNDM-TT gives the strongest Shapley-based subset selection among the compared methods, though the ordering varies across budgets.}
\label{tab:shapley-selection-endpoint}
\end{table}

\paragraph{Comparison across observation budgets.}
Tables~\ref{tab:shapley-selection-budgets}
and~\ref{tab:banzhaf-selection-budgets} report
completeness across all evaluated budgets.
For Banzhaf-based selection, we use the same fitted
datamodels, compute individual Banzhaf values, and
retain the three highest-scoring groups using the
same tie-breaking rule.
Each entry averages ten evaluation sequences within
each learner seed, then averages the two seed means.
The best-performing method depends on both the
observation budget and the attribution rule.

\begin{table}[t]
\centering
\footnotesize
\setlength{\tabcolsep}{4pt}
\begin{tabular}{lcccccc}
\toprule
Method & 32 & 48 & 64 & 96 & 128 & 160 \\
\midrule
Linear
& 20 & \textbf{20} & 15 & 35 & 20 & 35 \\
Degree-2 ridge
& \textbf{40} & 5 & 40 & 50 & 35 & 45 \\
Degree-3 ridge
& \textbf{40} & 5 & \textbf{75} & 40 & 30 & 20 \\
Degree-3 OMP
& 30 & \textbf{20} & 30 & 35 & 40 & 40 \\
RBF
& 35 & 10 & 55 & \textbf{65} & 15 & 20 \\
TNDM-TT
& 15 & 5 & 25 & 30 & 10 & 40 \\
Linear + TNDM-TT
& 30 & \textbf{20} & 35 & 45 & 25 & \textbf{60} \\
TNDM-TT, degree ladder
& \textbf{40} & 5 & \textbf{75} & 40 & \textbf{60} & 35 \\
Cardinality
& 0 & 0 & 0 & 0 & 0 & 0 \\
Fixed LOO
& 0 & 0 & 0 & 0 & 0 & 0 \\
Random scores
& 0 & 0 & 0 & 0 & 0 & 0 \\
\bottomrule
\end{tabular}
\caption{Shapley-based subset selection.
Columns give the number of new coalition observations.
Entries are mean percentages of complete selected
triples over blind learner seeds 43 and 59, each
evaluated on ten relation sequences.
Higher is better; bold marks all methods tied for
the highest mean within each budget.
Random scores reports the random-score baseline's completeness, not the expected
completeness of uniformly sampled triples. No method dominates across all budgets; attribution-based subset selection depends strongly on both the observation budget and the attribution rule.}
\label{tab:shapley-selection-budgets}
\end{table}

\begin{table}[t]
\centering
\footnotesize
\setlength{\tabcolsep}{4pt}
\begin{tabular}{lcccccc}
\toprule
Method & 32 & 48 & 64 & 96 & 128 & 160 \\
\midrule
Linear
& 20 & 20 & 15 & 35 & 20 & 35 \\
Degree-2 ridge
& \textbf{40} & 5 & 40 & 50 & 35 & 45 \\
Degree-3 ridge
& 35 & 0 & \textbf{65} & 40 & 25 & 30 \\
Degree-3 OMP
& 35 & 20 & 25 & 50 & 40 & 40 \\
RBF
& 35 & \textbf{25} & 40 & \textbf{70} & 40 & \textbf{50} \\
TNDM-TT
& 20 & 10 & 15 & 35 & 25 & 35 \\
Linear + TNDM-TT
& 25 & \textbf{25} & 20 & 50 & 30 & 40 \\
TNDM-TT, degree ladder
& 35 & 0 & \textbf{65} & 40 & \textbf{50} & 40 \\
Cardinality
& 0 & 0 & 0 & 0 & 0 & 0 \\
Fixed LOO
& 0 & 0 & 0 & 0 & 0 & 0 \\
Random scores
& 0 & 0 & 0 & 0 & 0 & 0 \\
\bottomrule
\end{tabular}
\caption{Banzhaf-based subset selection under the
same protocol as
Table~\ref{tab:shapley-selection-budgets}.
Entries are mean percentages of complete selected
triples; columns give new-observation budgets.
Higher is better; bold marks all methods tied for
the highest mean within each budget.
 }
\label{tab:banzhaf-selection-budgets}
\end{table}

\subsection{Pretrained-model fine-tuning evidence}
\label{app:finetuning-evidence}

\paragraph{Sport--Music.}
We adapt the task of \citet{huang2025generalization}.
Five sport--music associations define ten optional
training groups: one sport group and one music group
per association. Four subjects contribute documents
to each paired set of groups, giving four documents
per optional group. Every training run also includes
80 fixed background documents.

The study contains 96 full-parameter Qwen2-7B
fine-tunes: 88 distinct coalitions at one learner
seed and eight repeated runs covering four coalitions
at two additional seeds. Evaluation uses 80 probes
across five classes. Datamodel prediction is evaluated
over 15 shared coalition splits. These splits vary
the observed coalitions; they are not independent
replications of learner training.

Table~\ref{tab:sportmusic-readable} compares
prediction configurations. A single rank-five
TNDM-CP improves on the tuned linear, matched-pair,
and degree-2 baselines. Ensembles improve prediction
further using additional surrogate-fitting compute.
In the separate class-mode interaction analysis,
both TNDM-CP and class-specific degree-2 models
recover four of the five expected strongest pairs.

\begin{table}[t]
\centering
\small
\setlength{\tabcolsep}{5pt}
\begin{tabular}{lc}
\toprule
Datamodel configuration & Held-out $R^2$ $\uparrow$ \\
\midrule
Linear, tuned
& $0.579\pm0.048$ \\
Matched-pair, tuned
& $0.556\pm0.054$ \\
Degree-2, tuned
& $0.565\pm0.109$ \\
TNDM-CP, single fit, rank 5
& $0.668\pm0.084$ \\
TNDM-CP, best of 8, rank 5
& $0.616\pm0.133$ \\
TNDM-CP, ensemble of 8 initializations
& $0.719\pm0.070$ \\
TNDM-CP, ensemble over ranks 3--6
& $\mathbf{0.734\pm0.061}$ \\
TNDM-CP, ensemble over rank and regularization
& $0.706\pm0.059$ \\
\bottomrule
\end{tabular}
\caption{Sport--Music log-probability prediction:
mean $R^2$ $\pm$ sample SD over 15 shared coalition
splits. Ensembles use additional surrogate-fitting
compute. Bold marks the highest mean. TNDM-CP outperforms the tuned non-tensor baselines, and ensembling several low-rank fits improves prediction further.}
\label{tab:sportmusic-readable}
\end{table}

\paragraph{AG News RoBERTa--LoRA.}
Each coalition fine-tunes a fresh adapter and
classification head from the same pretrained
RoBERTa-base initialization. Eight optional training
groups vary document length, duplication, controlled
label corruption, and class enrichment. The outcome
is negative held-out probe log-loss.

The reference collection enumerates all 256
coalitions for each adapter rank in
$\{1,2,4,8,16,64\}$ and learner seed in
$\{31,43,59\}$, totaling 4,608 fine-tunes.
Reference attribution values use the complete
coalition outcomes for each rank and learner seed.
Estimators receive only the budgeted observations.

Table~\ref{tab:agnews-all-estimators} reports pooled
median relative $\ell_2$ errors at 128 observations,
using 20 observation-sampling repeats per rank/seed
combination. These combinations share the dataset
and do not represent 18 independent datasets.
TNDM-TT has the lowest median Banzhaf error among
the compared methods at this budget; KernelSHAP
has the lowest median Shapley error.
The separate comparison aggregating attribution
error across budgets favors degree-truncated and
RBF baselines over TNDMs at all six adapter ranks.
The displayed endpoint therefore supports competitive
attribution at this budget, rather than an advantage
across all fine-tuning comparisons.

\begin{table}[t]
\centering
\small
\setlength{\tabcolsep}{7pt}
\begin{tabular}{lcc}
\toprule
Estimator & Banzhaf $\downarrow$ & Shapley $\downarrow$ \\
\midrule
Linear
& 0.1701 & 0.3555 \\
Degree-2
& 0.2013 & 0.3721 \\
RBF
& 0.1724 & 0.3368 \\
TNDM-CP
& 0.1736 & 0.3358 \\
TNDM-TT
& \textbf{0.1634} & 0.3285 \\
\midrule
MSR Banzhaf
& 0.2258 & -- \\
KernelSHAP
& -- & \textbf{0.1233} \\
\bottomrule
\end{tabular}
\caption{Full AG News RoBERTa--LoRA attribution
comparison at 128 observed coalitions. Entries are
pooled median relative $\ell_2$ errors over six
adapter ranks, three learner seeds, and 20
observation-sampling repeats per combination.
MSR estimates Banzhaf values directly; KernelSHAP
estimates Shapley values directly.
Lower is better; bold marks the lowest median
for each target. TNDM-TT is competitive for attribution on AG News, while dedicated estimators can still be better for specific targets such as Shapley values.}
\label{tab:agnews-all-estimators}
\end{table}

\subsection{Prediction on broader experimental settings}
\label{app:broader-prediction}

We evaluate coalition prediction on AG News and UCI HAR
with $d=10$ training groups, and on a coin-flip task with
$d=8$ groups of in-context examples.
For AG News, each coalition defines a multinomial
naive Bayes classifier with 2048 hashed features and
add-one smoothing.
The outcome is the negative mean log-loss on 2000
test examples.
For UCI HAR, each coalition defines a multiclass ridge
classifier with regularization parameter $10$, applied
to a fixed 64-dimensional projection of the sensor
features.
The outcome is the negative mean softmax log-loss on
the official test split.
For the coin-flip task, the coalition selects examples
in the prompt of a frozen Qwen model, and the outcome
is the negative mean log-loss on the evaluation
examples.

\paragraph{Evaluation protocol.}
All $2^d$ coalition outcomes are available for each
setting.
For each of three random permutations of the
coalitions, we reveal a prefix of size $m$.
We use $75\%$ of this prefix for fitting and $25\%$
for validation; all remaining coalitions form the
test set.
The same split is used by every surrogate family.
Hyperparameters are selected separately within each
family by validation NMSE.
The polynomial baseline selects between degrees
$2$ and $3$; CP and TT select ranks from
$\{2,3,4\}$ and $\{2,3,4,6,8\}$, respectively.
Table~\ref{tab:broader-prediction} reports median test
NMSE across the three coalition splits at budgets
$m=2^d/8$, $2^d/4$, and $2^d/2$.
These repetitions vary the observed coalitions while
keeping the underlying coalition outcomes fixed.

\begin{table}[t]
\centering
\small
\setlength{\tabcolsep}{4pt}
\begin{tabular}{lrcccccc}
\toprule
Setting & $m$ & Linear & Polynomial & OMP & RBF & CP & TT \\
\midrule
AG News
& 128 & 0.522 & 0.208 & 0.778 & 0.131 & 0.062 & \textbf{0.041} \\
& 256 & 0.492 & 0.049 & 0.400 & 0.047 & 0.032 & \textbf{0.029} \\
& 512 & 0.479 & 0.044 & 0.294 & 0.034 & \textbf{0.014} & 0.041 \\
\midrule
UCI HAR
& 128 & 0.761 & \textbf{0.671} & 0.854 & 0.713 & 0.756 & 0.743 \\
& 256 & 0.777 & 0.700 & 0.719 & 0.656 & 0.774 & \textbf{0.613} \\
& 512 & 0.405 & 0.706 & 0.771 & 0.382 & 0.175 & \textbf{0.151} \\
\midrule
Coin flip
& 32  & 1.144 & 1.985 & 1.428 & 1.364 & \textbf{1.110} & 2.018 \\
& 64  & 1.904 & 0.872 & 3.443 & 1.622 & \textbf{0.868} & 1.159 \\
& 128 & 1.273 & 0.770 & 1.325 & 1.034 & 0.571 & \textbf{0.458} \\
\bottomrule
\end{tabular}
\caption{
Coalition prediction beyond relation composition.
Entries are median test NMSE over three coalition
splits; lower is better.
The budget $m$ includes both fitting and validation
coalitions.
AG News and UCI HAR use deterministic classifier
refits with $d=10$.
Coin flip uses frozen-model inference with $d=8$.
Bold marks the lowest median in each row.
}
\label{tab:broader-prediction}
\end{table}

\paragraph{Results.}
The relative performance depends on the setting and
the observation budget.
On AG News, TT has the lowest median NMSE at
$m=128$ and $256$, while CP performs best at $m=512$.
On UCI HAR, the polynomial baseline performs best at
$m=128$, and TT performs best at the two larger
budgets.
On coin flip, CP performs best at the two smaller
budgets and TT at $m=128$.
These results support tensor network surrogates
beyond relation composition, with gains that vary
across tasks and budgets.
\subsection{In-context prediction with \texorpdfstring{$d=500$}{d=500}}
\label{app:large-icl}

We evaluate coalition prediction for a frozen
Qwen2.5-1.5B-Instruct model on coin-flip reasoning
and GSM8K.
Each task has a pool of $d=500$ candidate
demonstrations, and each coalition selects four
demonstrations for the prompt.
The outcome is the negative answer-token log-loss,
averaged over 48 evaluation queries for coin flip
and 12 for GSM8K.
Model weights remain fixed throughout.

\paragraph{Evaluation protocol.}
For each task and each seed in $\{31,43\}$, we
evaluate 1300 coalitions.
We use nested fitting sets of size
$n_{\mathrm{fit}}\in\{400,800\}$, a separate set of
200 validation coalitions, and 300 test coalitions.
Thus, fitting and model selection together use
600 or 1000 observed outcomes.
Every surrogate receives the same coalitions and
outcomes.
Candidates are selected by validation RMSE within
each family and, for TT, within each ordering.
We report the random and semantic orderings
separately.

The CP baseline factorizes the degree-two and
degree-three coefficient tensors separately and
includes an additive component.
We denote it by coefficient CP to distinguish it
from the direct value-tensor CP model used in the
preceding comparison.
We also evaluate a degree-graded TT with maximum
degree three alongside the ordinary TT.

\begin{table}[t]
\centering
\footnotesize
\setlength{\tabcolsep}{4pt}
\begin{tabular}{lcccc}
\toprule
& \multicolumn{2}{c}{Coin flip}
& \multicolumn{2}{c}{GSM8K} \\
\cmidrule(lr){2-3}
\cmidrule(lr){4-5}
Surrogate
& $n_{\mathrm{fit}}=400$
& $n_{\mathrm{fit}}=800$
& $n_{\mathrm{fit}}=400$
& $n_{\mathrm{fit}}=800$ \\
\midrule
Linear
& $0.1779 \pm 0.0123$
& $0.1853 \pm 0.0002$
& $0.3262 \pm 0.0113$
& $0.2854 \pm 0.0105$ \\
RBF
& $0.1755 \pm 0.0131$
& $0.1679 \pm 0.0078$
& $0.2885 \pm 0.0057$
& $0.2645 \pm 0.0142$ \\
Coefficient CP
& $0.1754 \pm 0.0154$
& $0.1674 \pm 0.0091$
& $0.2890 \pm 0.0087$
& $0.2760 \pm 0.0123$ \\
TT, random
& $0.1732 \pm 0.0174$
& $0.1680 \pm 0.0077$
& $0.2874 \pm 0.0093$
& $0.2758 \pm 0.0121$ \\
TT, semantic
& $0.1733 \pm 0.0174$
& $0.1671 \pm 0.0088$
& $0.2887 \pm 0.0113$
& $0.2760 \pm 0.0119$ \\
Degree-graded TT, random
& $0.1750 \pm 0.0156$
& $0.1685 \pm 0.0124$
& $0.2879 \pm 0.0100$
& $0.2761 \pm 0.0116$ \\
Degree-graded TT, semantic
& $0.1750 \pm 0.0156$
& $0.1678 \pm 0.0086$
& $3.2994 \pm 0.3212$
& $1.1528 \pm 1.1618$ \\
\bottomrule
\end{tabular}
\caption{
Held-out coalition prediction for frozen-model
in-context learning with $d=500$ candidate
demonstrations and four demonstrations per coalition.
Entries are mean test RMSE $\pm$ sample standard
deviation across two seeds; lower is better.
Each fitting budget is supplemented by 200
validation coalitions. At larger candidate-pool dimension, CP and TT remain competitive with nonlinear baselines and improve on the linear model, while RBF is strongest on GSM8K at the larger budget.
}
\label{tab:large-icl}
\end{table}

\paragraph{Results.}
Coefficient CP and ordinary TT improve on the
linear baseline for both tasks, both fitting
budgets, and both seeds.
At $n_{\mathrm{fit}}=800$, their mean errors on
coin flip are close to those of RBF.
On GSM8K, RBF achieves the lowest mean error at
this budget.
Ordinary TT gives similar results under the two
orderings.
The degree-graded variant has substantially larger
test errors under the semantic ordering on GSM8K.
These experiments show that tensor network
surrogates remain competitive with nonlinear
baselines at a larger candidate-pool dimension.

\subsection{Attribution accuracy}
\label{app:attribution-results}

We compare attribution values computed from fitted
datamodels with reference values obtained by
enumerating all $2^{10}=1024$ measured coalition
outcomes.
We evaluate linear, RBF, and TT surrogates on
UCI HAR and a Qwen representation-probe task.

\paragraph{Settings.}
UCI HAR uses the classifier-refitting task described
above, with ten training data groups.
For the Qwen probe, each coalition selects training
groups for a logistic-regression classifier on fixed
representations from Qwen2.5-1.5B-Instruct.
The representations are standardized and projected
to 32 dimensions using a transformation fitted once
on the training representations.
Three of the ten groups have flipped training labels.
The classifier is refitted for each coalition with
$C=1$, while the representation model remains fixed.
The outcome is $\log 2$ minus the mean true-label
log-loss on 112 held-out examples.

\paragraph{Fitting and evaluation.}
For UCI HAR, we report 320 fitting coalitions and
32 validation coalitions over ten matched coalition
splits.
For the Qwen probe, we use 512 fitting coalitions
and 205 validation coalitions for each of the three
seeds $\{31,43,59\}$.
Within each setting and repetition, all surrogate
families receive the same fitting and validation
coalitions.
Candidates are selected using coalition-prediction
error on the validation set.

We measure attribution error by
$
\operatorname{RMSE}_{\mathcal{A}}
=
\left(
\frac{1}{|\mathcal{A}|}
\sum_{A\in\mathcal{A}}
\bigl(
T_A(\widehat v)-T_A(v)
\bigr)^2
\right)^{1/2},
$
where $T_A$ denotes the attribution quantity being
evaluated.
For Shapley values, Banzhaf values, and leave-one-out
effects, $\mathcal{A}$ contains the ten singleton
groups.
For pairwise Shapley interactions, it contains all
$\binom{10}{2}=45$ unordered pairs.
The reference values $T_A(v)$ use the complete
measured coalition function.

\begin{table}[t]
\centering
\small
\setlength{\tabcolsep}{5pt}
\begin{tabular}{lccc}
\toprule
Attribution & Linear & RBF & TT \\
\midrule
\multicolumn{4}{l}{
UCI HAR: 320 fitting, 32 validation; ten splits
} \\
Shapley
& $0.047190 \pm 0.000116$
& $0.044079 \pm 0.000344$
& $0.041886 \pm 0.000704$ \\
Banzhaf
& $0.001240 \pm 0.000105$
& $0.001132 \pm 0.000062$
& $0.000983 \pm 0.000072$ \\
Pairwise interaction
& $0.049988 \pm 0.000000$
& $0.047985 \pm 0.000215$
& $0.045343 \pm 0.000602$ \\
\midrule
\multicolumn{4}{l}{
Qwen probe: 512 fitting, 205 validation; three seeds
} \\
Shapley
& $0.067948 \pm 0.001012$
& $0.009971 \pm 0.002894$
& $0.055446 \pm 0.007836$ \\
Banzhaf
& $0.033136 \pm 0.002820$
& $0.005382 \pm 0.000884$
& $0.018738 \pm 0.006185$ \\
Leave-one-out
& $0.309582 \pm 0.004989$
& $0.009050 \pm 0.007205$
& $0.077350 \pm 0.024296$ \\
Pairwise interaction
& $0.229971 \pm 0.000743$
& $0.019093 \pm 0.004560$
& $0.062223 \pm 0.008455$ \\
\bottomrule
\end{tabular}
\caption{
Attribution RMSE against references computed from
all 1024 measured coalition outcomes.
Entries are mean $\pm$ sample standard deviation
across the repetitions specified in each panel;
lower is better. TNDM-TT gives the lowest attribution error on UCI HAR and improves over the linear baseline on the Qwen probe, while RBF is strongest on the probe task.
}
\label{tab:attribution-accuracy}
\end{table}

\paragraph{Results.}
TT has the lowest mean error among the three
surrogates for each reported attribution quantity
on UCI HAR.
On the Qwen probe, TT improves on the linear
baseline, while RBF achieves the lowest mean error
for all four quantities.
Thus, attribution accuracy depends on how well the
fitted datamodel captures the measured coalition
function.
Exact contraction computes attribution values for
the fitted TT; error relative to the measured
function remains a separate empirical quantity.

\subsection{What the linear datamodeling score measures}
\label{app:lds}

The linear datamodeling score (LDS) evaluates how
well sums of individual attribution scores rank
retraining outcomes \citep{park2023trak}.
For a fixed evaluation query $q$, attribution vector
$\tau(q)$, and evaluation subsets
$S_1,\ldots,S_M$ of a specified size, it is
\begin{equation}
\operatorname{LDS}_q(\tau)
=
\rho_{\mathrm{Sp}}\!\left(
\bigl(v_q(S_j)\bigr)_{j=1}^{M},
\left(
\sum_{i\in S_j}\tau_i(q)
\right)_{j=1}^{M}
\right),
\label{eq:lds-definition}
\end{equation}
where $\rho_{\mathrm{Sp}}$ denotes Spearman
correlation. We adapt this definition to training
groups. In our comparison, $\tau(q)$ is the exact
Shapley vector of each fitted datamodel, computed
from its predictions without using reference
outcomes. This is our evaluation convention;
LDS itself does not require Shapley values.

\paragraph{Individual scores do not retain all joint effects.}
For a nonlinear datamodel, LDS evaluates
$\sum_{i\in S}\tau_i(q)$ rather than the direct
prediction $\widehat v_q(S)$. Different subset
functions can have the same individual attribution
vector and therefore the same additive predictions.

Consider four training groups and the function
$
v(S)
=
\mathbf{1}\{S\cap\{1,2\}\ne\varnothing\}
\mathbf{1}\{S\cap\{3,4\}\ne\varnothing\}.
$
A subset succeeds only if it contains a group from
each pair. By symmetry, every group has Shapley
value $1/4$. Consequently, the sum of Shapley
values is $1/2$ for every two-element subset.
Nevertheless, four of these subsets have utility
one and the other two have utility zero.

LDS is therefore undefined because its additive
prediction is constant, even when the Shapley
values and the underlying datamodel are exact.
The direct prediction $\widehat v=v$ instead has
Spearman correlation one.
Moreover, the additive function $|S|/4$ has the
same Shapley vector despite representing a different
subset function. Evaluating only the attribution
vector cannot distinguish these representations.

\paragraph{Evidence from retraining experiments.}
Table~\ref{tab:lds-native} compares LDS with Spearman
correlation of the direct datamodel prediction
$\widehat v_q(S)$. Both use the same evaluation
subsets within each experiment. For LDS, the
prediction is the sum of the fitted datamodel's
individual Shapley values.

On the eight-group task, replacing direct nonlinear
predictions with summed individual scores reduces
correlation for every displayed nonlinear family.
For example, RBF achieves direct correlation 0.738
but LDS 0.386; TNDM-TT achieves 0.697 and 0.358,
respectively. Linear predictions and their score
sums have identical correlations.
On the twelve-group task, TNDM-TT's direct
correlation is 0.164, compared with LDS 0.031.
Thus the reduction to individual scores can discard
predictive information from polynomial, kernel,
and tensor datamodels alike.

\begin{table}[t]
\centering
\footnotesize
\setlength{\tabcolsep}{4pt}
\begin{tabular}{lcccc}
\toprule
& \multicolumn{2}{c}{$d=8$}
& \multicolumn{2}{c}{$d=12$} \\
\cmidrule(lr){2-3}
\cmidrule(lr){4-5}
Datamodel & LDS & Direct & LDS & Direct \\
\midrule
Linear
& $0.341\pm0.053$
& $0.341\pm0.053$
& $0.032\pm0.021$
& $0.032\pm0.021$ \\

Degree-2 ridge
& $0.377\pm0.055$
& $0.730\pm0.023$
& $0.026\pm0.018$
& $0.014\pm0.026$ \\

Degree-3 ridge
& $0.381\pm0.057$
& $0.735\pm0.023$
& $0.031\pm0.009$
& $0.037\pm0.020$ \\

RBF
& $0.386\pm0.054$
& $0.738\pm0.025$
& $0.022\pm0.003$
& $0.045\pm0.013$ \\

TNDM-CP
& $0.282\pm0.135$
& $0.571\pm0.108$
& $\mathbf{0.037\pm0.001}$
& $0.035\pm0.034$ \\

TNDM-TT
& $0.358\pm0.063$
& $0.697\pm0.046$
& $0.031\pm0.001$
& $\mathbf{0.164\pm0.008}$ \\

HT
& $\mathbf{0.403\pm0.047}$
& $\mathbf{0.755\pm0.033}$
& $0.020\pm0.024$
& $0.030\pm0.019$ \\

TNDM-TT, degree ladder
& $0.376\pm0.050$
& $0.734\pm0.021$
& $0.020\pm0.007$
& $0.024\pm0.023$ \\
\bottomrule
\end{tabular}
\caption{LDS and direct-prediction Spearman
correlation on the same evaluation subsets.
Entries are means $\pm$ sample SD; higher is better.
For $d=8$, the observation budget is 128 and
evaluation uses unobserved size-four subsets.
Statistics use 20 observation-sampling repeats
at one learner seed, after averaging eight
relation sequences within each repeat.
For $d=12$, the budget is 160 new observations
and evaluation uses all 220 size-three subsets.
Statistics use two blind learner seeds, after
averaging ten sequences within each seed.
Bold marks the highest displayed mean per column.
The two experiments use different protocols
and should not be pooled. Summing individual attribution scores can discard information captured by nonlinear datamodels; direct subset predictions better preserve their joint effects.}
\label{tab:lds-native}
\end{table}

\paragraph{Correlation conventions.}
Values tied within $10^{-10}$ receive average ranks.
A constant prediction or reference makes correlation
undefined; these cases are recorded as undefined,
not zero. The displayed methods have defined
correlations for all repetitions used in the table.
The twelve-group results represent two learner
seeds, rather than twenty independent training
replications.

\paragraph{Matching the metric to the prediction task.}
LDS assesses the predictive usefulness of additive
attribution scores. It does not, by itself,
evaluate a nonlinear datamodel's joint predictions.
Direct rank correlation retains those predictions,
but it is unchanged by strictly increasing
transformations and therefore does not measure
whether predicted utility changes have the correct
magnitude.

We consequently evaluate coalition prediction with
NMSE, removal effects with RMSE, interaction recovery
with average precision, and subset selection with
completeness and measured retrained utility.
These metrics assess different uses of the fitted
function. LDS can complement them, but is
insufficient as the sole measure of nonlinear
datamodel quality.

\subsection{Numerical verification of attribution computations}
\label{app:attribution-verification}

For the $d=10$ Qwen-probe experiment, we verify
the attribution implementations against exhaustive
enumeration of each fitted surrogate.
We evaluate the surrogate on all 1024 coalitions
and compute Shapley values, Banzhaf values,
leave-one-out effects, and pairwise Shapley
interactions directly from these predictions.
We compare these values with those returned
by the model-specific attribution routines.

Table~\ref{tab:attribution-verification} reports
the largest absolute discrepancy across seeds
31, 43, and 59.
The singleton column additionally takes the
maximum over the three singleton attribution
quantities and all ten groups.
The pairwise column covers all 45 unordered
pairs.

\begin{table}[t]
\centering
\small
\begin{tabular}{lcc}
\toprule
Surrogate & Singleton quantities & Pairwise interactions \\
\midrule
Linear
& $1.13\times10^{-14}$
& $8.23\times10^{-17}$ \\
Coefficient CP
& $1.33\times10^{-15}$
& $5.83\times10^{-16}$ \\
TT
& $4.77\times10^{-15}$
& $4.50\times10^{-15}$ \\
\bottomrule
\end{tabular}
\caption{
Maximum absolute discrepancy between
model-specific attribution computations and
exhaustive enumeration of the same fitted
surrogate. The tensor-based attribution computations match exhaustive enumeration to numerical precision.
}
\label{tab:attribution-verification}
\end{table}

The two computations agree to numerical
precision on these fitted models.
This check concerns evaluation of the surrogate.
Accuracy relative to the measured coalition
function is assessed separately by the
attribution errors reported above.
\subsection{Budget and replication conventions}
\label{app:reproducibility}

\paragraph{Observation budgets.}
We distinguish the number of fitting coalitions
from the total number of outcomes used for fitting
and model selection.
For the $d=24$ prediction experiment and the
small-dimensional broader-setting comparison,
$m$ includes both fitting and validation coalitions.
For the $d=500$ experiments, $n_{\mathrm{fit}}$
counts fitting coalitions, with 200 additional
validation coalitions.
The attribution comparisons use 320 fitting and
32 validation coalitions for UCI HAR, and 512
fitting and 205 validation coalitions for the
Qwen probe.
Complete enumeration supplies reference attribution
values without making all coalition outcomes
available for surrogate fitting.

\paragraph{Replication and aggregation.}
The repetitions differ across experiments.
The $d=24$ prediction comparison varies the learner
seed while sharing the coalition design.
The small-dimensional broader-setting comparison
varies the fitting and validation split of a fixed
coalition function.
The UCI HAR attribution comparison likewise varies
the coalition split with the underlying learner
fixed.
Accordingly, each table specifies its replication
unit and summary statistic.
Reported ranges and standard deviations describe
variation across the recorded repetitions;
they are not confidence intervals.

\paragraph{Reference quantities.}
For fully enumerated settings, reference
attributions are computed from measured coalition
outcomes.
Attributions computed from a fitted datamodel are
evaluated against these references.
For the relation-composition tasks, structural
completeness and interaction-support recovery
instead use the specified relation-to-group
assignments.
We report these structural criteria separately
from prediction error and retrained utility.

\subsection{Group construction and tensor ordering}
\label{app:group-construction}

\paragraph{AG News.}
We construct ten training groups by partitioning
each class into text-hash bins.
Classes zero and one each contribute three bins,
while classes two and three each contribute two.
Within a class, an example is assigned by its
text's SHA-256 hash modulo the number of bins.
Groups are ordered first by class and then by
bin index.

The classifier uses 2048 nonnegative hashed
features containing unigrams and bigrams.
The training and evaluation sets contain the
first 20,000 training examples and first 2000
test examples, respectively, in dataset order.
Class and feature smoothing parameters are
both one.

\paragraph{UCI HAR.}
We sort the official training-subject identifiers
and partition them into ten consecutive blocks
using approximately equal numbers of subjects.
Each block defines one training group.
The official test subjects remain held out.
Groups are ordered by their smallest subject
identifier.

The 561 sensor features are projected to
64 dimensions using a fixed random matrix
with seed 20260822.
Its entries take values $-1$, $0$, and $1$
with probabilities $1/6$, $2/3$, and $1/6$,
respectively, and are scaled by $\sqrt{3/64}$.
A bias term is added after projection and is
excluded from the classifier's ridge penalty.

\paragraph{Small-dimensional coin-flip task.}
The eight coalition units are the eight
in-context demonstrations.
Their tensor order follows their original
order in the prompt.
The empty coalition corresponds to evaluation
without demonstrations.
The outcome is the negative mean answer
log-loss, with the language-model parameters
held fixed.

\paragraph{Ordering in the prediction comparisons.}
For these three settings, the TT prediction
fits use the natural group orders specified
above.
The separate tensor-compression analysis
considers the natural order and twenty fixed
random permutations, using seeds
41001 through 41020.
Those compression-order comparisons are
distinct from the fitted prediction results.

\paragraph{Ordering at $d=500$.}
For coin flip, the semantic order sorts
demonstrations lexicographically by starting
face, sequence length, number of flips, and
the binary flip pattern.
For GSM8K, it sorts demonstrations by the
tuple of distinct arithmetic operators in
the annotated calculations, the number of
annotated calculation steps, and the character
length of the reference solution.
Ties are broken by the original demonstration
index.

Ordinary TT and degree-graded TT are fitted
under both the semantic order and a random
permutation generated with seed $s+909$,
where $s$ is the experiment seed.
Their results are reported separately.
The remaining fitted families are invariant
to permutations of the coalition coordinates.

\subsection{Fitting details for the \texorpdfstring{$d=500$}{d=500} experiments}
\label{app:large-icl-fitting}

\paragraph{Coalition partitions.}
Each task and seed has an ordered collection
of 1300 evaluated coalitions.
The first 800 form the fitting pool, the next
200 form the validation set, and the final
300 form the test set.
The smaller fitting budget uses only the first
400 coalitions of the fitting pool.
The validation and test sets remain fixed
across the two budgets.
All surrogate families use these same
partitions.

\paragraph{Target standardization.}
For each task, seed, and fitting budget, we
compute the outcome mean $\mu$ and population
standard deviation $\sigma$ using only the
fitting coalitions.
If the standard deviation is zero, we set
$\sigma=1$.
The surrogate is fitted to
$
\widetilde v(S)
=
\frac{v(S)-\mu}{\sigma}.
$
Predictions are transformed back to the
original outcome scale:
$
\widehat v(S)
=
\mu+\sigma\,\widehat f(S),
$
where $\widehat f$ predicts the standardized
outcome.
Validation and test outcomes do not contribute
to the standardization parameters.

\paragraph{Configuration selection.}
Candidates are selected separately for each
surrogate family and, for the two TT families,
for each ordering.
Selection minimizes validation RMSE, with ties
broken by candidate identifier.
The selected candidate is evaluated on the
held-out test coalitions.
The semantic and random orderings remain
separate reported configurations.

\paragraph{Reported error.}
Test RMSE is computed in the original
log-loss units.
For each task, fitting budget, family, and
ordering, we report the mean and sample
standard deviation across seeds 31 and 43.
These errors describe prediction of coalition
outcomes; they are distinct from the frozen
language model's answer accuracy.

\begin{figure}[t]
\centering
\includegraphics[width=\linewidth]{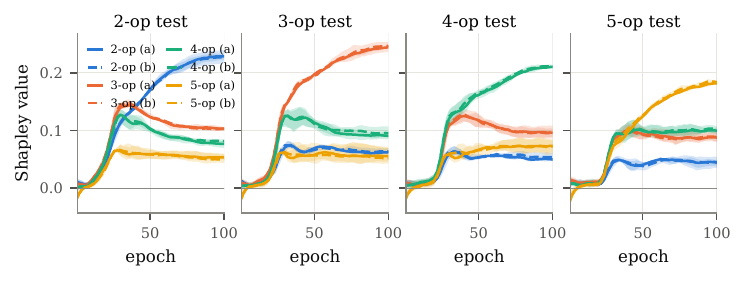}
\caption{\textbf{Data contributions depend on the target and training stage.}
Shapley values of eight training data groups for arithmetic targets
with two through five operators. Solid and dashed lines distinguish
the two random halves of each operator-count group.
Lines show three-seed means; bands span the seed minimum and maximum.
The group matching the evaluation target becomes dominant later in
training.}
\label{fig:dynamics-eight-targets}
\end{figure}

\subsection{Training dynamics of data attribution}
\label{app:arithmetic-dynamics}
To study how training data contributions evolve, the modular-arithmetic
chain-of-thought experiment is evaluated every two training epochs.
Training expressions containing two, three, four, or five operators are
grouped by operator count, and each group is divided into two random
halves, giving $d=8$ training data units. Every one of the $2^8=256$
coalitions is trained under three learner seeds.
At checkpoint $t$, $v_t^{(s)}(S)$ is the exact-answer accuracy for
coalition $S$ and seed $s$, evaluated on 500 held-out expressions per
target operator count. These are attributions to groups of training
examples. Checkpoint evaluation reuses the same training trajectories,
with additional evaluation cost but no additional retraining runs.

Shapley values and pairwise Banzhaf interaction indices are computed
from the complete measured coalition games, without fitting a surrogate.
For each pair $i,j$, the interaction index is the uniform average of
$v_t(S\cup\{i,j\})-v_t(S\cup\{i\})-v_t(S\cup\{j\})+v_t(S)$
over $S\subseteq[d]\setminus\{i,j\}$; positive and negative values
indicate average synergy and redundancy, respectively.
The non-additivity index, $1-R^2$, is the fraction of coalition-outcome
variance left unexplained by the best affine predictor under uniform
weighting of all coalitions. The Walsh spectrum decomposes this variance
by interaction degree. After removing the affine predictor, the
unfolding diagnostic measures the fraction of squared singular values
retained by rank $r$, reporting the minimum over cuts in the displayed
group ordering.
Related spectral analyses characterize the effective rank
and compression of hidden representations in language models
\citep{skean2025layer}. The present analysis instead tracks
the compressibility of training-subset outcomes across
training checkpoints.
For the five-operator target, contributions from three- and four-operator
training data rise early and subsequently plateau or decline, whereas
five-operator data becomes dominant later. Selected pairwise
interactions also change sign during training. In particular, the two
halves of the five-operator group are synergistic near the learning
transition but redundant at the final checkpoint, with a Banzhaf
interaction of approximately $-0.20$ (seed range $[-0.22,-0.17]$).
Non-additive terms account for up to approximately $66\%$ of the variance
of the seed-averaged game, and approximately $22\%$ at the final
checkpoint. Nevertheless, rank four retains at least approximately
$96\%$ of the residual unfolding energy at every cut over the reported
time window. This is another example of the relevance of  the theoretical distinction
between interaction degree and tensor rank. Here we see that substantial non-additive
dependence need not produce large unfolding ranks.
Across all four evaluation targets, the group matching the target
operator count becomes the largest contributor later in training
(Fig.~\ref{fig:dynamics-eight-targets}).

\end{document}